\documentclass[10pt,twoside,twocolumn,english,aps,manuscript,nofootinbib, nobibnotes, secnumarabic]{revtex4}
\usepackage[T1]{fontenc}
\usepackage[latin9]{inputenc}
\usepackage[letterpaper]{geometry}
\usepackage{xcolor}
\usepackage{amsthm}
\usepackage{amsmath}
\usepackage{amssymb}
\usepackage{esint}
\usepackage[all]{xy}
\PassOptionsToPackage{normalem}{ulem}
\usepackage{ulem}

\makeatletter

\providecommand{\tabularnewline}{\\}
\providecolor{lyxadded}{rgb}{0,0,1}
\providecolor{lyxdeleted}{rgb}{1,0,0}
\@ifundefined{textcolor}{}
{%
 \definecolor{BLACK}{gray}{0}
 \definecolor{WHITE}{gray}{1}
 \definecolor{RED}{rgb}{1,0,0}
 \definecolor{GREEN}{rgb}{0,1,0}
 \definecolor{BLUE}{rgb}{0,0,1}
 \definecolor{CYAN}{cmyk}{1,0,0,0}
 \definecolor{MAGENTA}{cmyk}{0,1,0,0}
 \definecolor{YELLOW}{cmyk}{0,0,1,0}
 }
\newenvironment{lyxlist}[1]
{\begin{list}{}
{\settowidth{\labelwidth}{#1}
 \setlength{\leftmargin}{\labelwidth}
 \addtolength{\leftmargin}{\labelsep}
 }}
{\end{list}}
\theoremstyle{plain}
\newtheorem{thm}{\protect\theoremname}[section]
  \theoremstyle{definition}
  \newtheorem{example}[thm]{\protect\examplename}
  \theoremstyle{definition}
  \newtheorem{defn}[thm]{\protect\definitionname}
  \theoremstyle{remark}
  \newtheorem{rem}[thm]{\protect\remarkname}
  \theoremstyle{plain}
  \newtheorem{criterion}[thm]{\protect\criterionname}
  \theoremstyle{plain}
  \newtheorem{lem}[thm]{\protect\lemmaname}
  \theoremstyle{plain}
  \newtheorem{cor}[thm]{\protect\corollaryname}

\usepackage[all]{xy}

\usepackage{multirow}
\usepackage{pslatex}
\usepackage{calligra}
\DeclareMathAlphabet{\mathcalligra}{T1}{calligra}{m}{n}
\DeclareFontShape{T1}{calligra}{m}{n}{<->s*[2.2]callig15}{}

\def\frontmatter@abstractheading{}

\makeatother

\usepackage{babel}
  \providecommand{\corollaryname}{Corollary}
  \providecommand{\criterionname}{Criterion}
  \providecommand{\definitionname}{Definition}
  \providecommand{\examplename}{Example}
  \providecommand{\lemmaname}{Lemma}
  \providecommand{\remarkname}{Remark}
\providecommand{\theoremname}{Theorem}

\begin{document}

\title{Selectivity in Probabilistic Causality: Drawing Arrows from Inputs
to Stochastic Outputs}

\author{Ehtibar N. Dzhafarov}

\thanks{Corresponding author: Ehtibar Dzhafarov, Purdue University, Department
of Psychological Sciences, 703 Third Street West Lafayette, IN 47907,
USA. email: ehtibar@purdue.edu. This research has been supported by
AFOSR grant FA9550-09-1-0252 to Purdue University and by the Academy
of Finland grant 121855 to University of Jyväskylä.}

\affiliation{Purdue University}

\author{Janne V. Kujala}

\affiliation{University of Jyväskylä}
\begin{abstract}
\mbox{}

Given a set of several inputs into a system (e.g., independent variables
characterizing stimuli) and a set of several stochastically non-independent
outputs (e.g., random variables describing different aspects of responses),
how can one determine, for each of the outputs, which of the inputs
it is influenced by? The problem has applications ranging from modeling
pairwise comparisons to reconstructing mental processing architectures
to conjoint testing. A necessary and sufficient condition for a given
pattern of selective influences is provided by the Joint Distribution
Criterion, according to which the problem of {}``what influences
what'' is equivalent to that of the existence of a joint distribution
for a certain set of random variables. For inputs and outputs with
finite sets of values this criterion translates into a test of consistency
of a certain system of linear equations and inequalities (Linear Feasibility
Test) which can be performed by means of linear programming. The Joint
Distribution Criterion also leads to a metatheoretical principle for
generating a broad class of necessary conditions (tests) for diagrams
of selective influences. Among them is the class of distance-type
tests based on the observation that certain functionals on jointly
distributed random variables satisfy triangle inequality. 

\mbox{}


\markboth{Dzhafarov and Kujala}{Selectivity in Probabilisitc Causality}
\end{abstract}
\maketitle

\section{Introduction\label{sec:Introduction}}

This paper presents a general methodology of dealing with \emph{diagrams
of selective influences}, like this one: \emph{
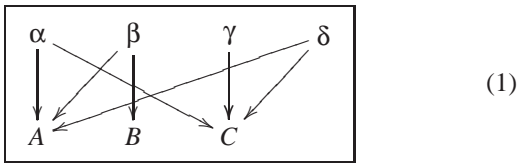
\begin{equation}
\boxed{\begin{array}{c}
\xymatrix{\alpha\ar[d]\ar[drr] & \beta\ar[d]\ar[dl] & \gamma\ar[d]\ar[d] & \delta\ar[dl]\ar[dlll]\\
A & B & C
}
\end{array}}\label{diag:1}
\end{equation}
}

\protect{\noindent}The Greek letters in this diagram represent \emph{inputs},
or \emph{external factors}, e.g., parameters of stimuli whose values
can be chosen at will, or randomly vary but can be observed. The capital
Roman letters stand for random outputs characterizing reactions of
the system (an observer, a group of observers, a technical device,
etc.). The arrows show which factor influences which random output.
The factors are treated as \emph{deterministic} entities: even if
$\alpha,\beta,\gamma,\delta$ in reality vary randomly (e.g., being
randomly generated by a computer program, or being concomitant parameters
of observations, such as age of respondents), for the purposes of
analyzing selective influences the random outputs $A,B,C$ are always
viewed as \emph{conditioned} upon various combinations of specific
values of $\alpha,\beta,\gamma,\delta$.

The first question to ask is: what is the meaning of the above diagram
if the random outputs $A,B,C$ in it are not necessarily stochastically
independent? (If they are, the answer is of course trivial.) And once
the meaning of the diagram of selective influences is established,
how can one determine that this diagram correctly characterizes the
dependence of the joint distributions of the random outputs $A,B,C$
on the external factors $\alpha,\beta,\gamma,\delta$? These questions
are important, because the assumption of stochastic independence of
the outputs more often than not is either demonstrably false or adopted
for expediency alone, with no other justification, while the assumption
of selectivity in causal relations between inputs and stochastic outputs
is ubiquitous in theoretical modeling, often being built in the very
language of the models.

\subsection{An illustration: Pairwise comparisons}

Consider Thurstone's most general model of pairwise comparisons (Thurstone,
1927).%
\footnote{This model is known as Thurstonian Cases 1 and 2. The only difference
between the two is that in Case 1 the responding system is an individual
observer to whom pairs of stimuli are presented repeatedly, while
in Case 2 the responding system is a group of people each responding
to every pair of stimuli once. One can, of course, think of all kinds
of mixed or intermediate situations.%
} This model is predicated on the diagram

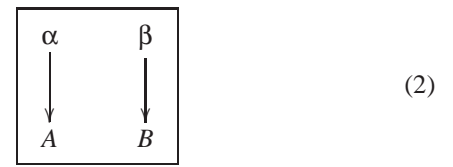
\begin{equation}
\boxed{\begin{array}{c}
\xymatrix{\alpha\ar[d] & \beta\ar[d]\\
A & B
}
\end{array}}\label{diag:Thurstone}
\end{equation}
where $\left(A,B\right)$ are bivariate normally distributed random
variables, and $\alpha,\beta$ are two stimuli being compared. The
stimuli are identified by their {}``observation areas'' (Dzhafarov,
2002): say, the label $\alpha$ may stand for {}``chronologically
first'' or {}``located to the left from fixation point,'' and the
label $\beta$ for, respectively, {}``chronologically second'' or
{}``located to the right from fixation point.'' For our present
purposes, $\alpha$ and $\beta$ are external factors with varying
values (e.g., light intensity in, respectively, first and second observation
areas). The random variables $A$ and $B$ are supposed to represent
some unidimensional property (say, brightness) of the images of, \emph{respectively},
the stimuli $\alpha$ and $\beta$ (the emphasized word {}``respectively''
indicating selectiveness). According to the model, the probability
with which $\alpha$ is judged to have less of the property in question
than $\beta$ equals $\Pr\left[A<B\right]$. The problem is: what
restrictions should be imposed in this theoretical scheme on the bivariate-normal
distribution of $A,B$ to ensure that $A$ is an image of the stimulus
$\alpha$ alone and $B$ is an image of the stimulus $\beta$ alone,
as opposed to both or either of them being an image of both the stimuli
$\alpha$ and $\beta$? In other words, how can one distinguish, within
the framework of Thurstone's general model, the diagram of selective
influences (\ref{diag:Thurstone}) from the diagrams
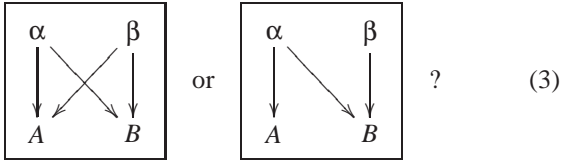
\begin{equation}
\boxed{\begin{array}{c}
\xymatrix{\alpha\ar[d]\ar[dr] & \beta\ar[d]\ar[dl]\\
A & B
}
\end{array}}\quad\textnormal{or}\quad\boxed{\begin{array}{c}
\xymatrix{\alpha\ar[d]\ar[dr] & \beta\ar[d]\\
A & B
}
\end{array}}\quad?\label{diag:Thurstone.alt}
\end{equation}
Denoting by $A\left(x,y\right),B\left(x,y\right)$ the two random
variables at the values $\left(x,y\right)$ of the factors $\left(\alpha,\beta\right)$,%
\footnote{It may seem unnecessary to use separate notation for factors and their
values (levels), but it is in fact more convenient in view of the
formal treatment presented below. The factors there are defined as
sets of {}``factor points,'' and the latter are defined as factor
values associated with particular factor names: e.g., $\left(x,`\alpha\textnormal{'}\right)$
is a factor point of factor $\alpha$.%
} intuition tells us that one should be able to write
\[
A\left(x,y\right)=A\left(x\right),B\left(x,y\right)=B\left(y\right)
\]
if the diagram (\ref{diag:Thurstone}) holds, but not in the case
of the diagrams \ref{diag:Thurstone.alt}. Clearly then, one should
require that 
\begin{equation}
\begin{array}{c}
\begin{array}{cc}
\mathrm{E}\left[A\left(x,y\right)\right]=\mu_{A}\left(x\right), & \mathrm{Var}\left[A\left(x,y\right)\right]=\sigma_{AA}\left(x\right),\\
\\
\end{array}\\
\begin{array}{cc}
\mathrm{E}\left[B\left(x,y\right)\right]=\mu_{B}\left(y\right), & \mathrm{Var}\left[B\left(x,y\right)\right]=\sigma_{BB}\left(y\right),\end{array}
\end{array}\label{eq:Thurstone1}
\end{equation}
with the obvious notation for the parameters of the two distributions.
These equations form an instance of what is called \emph{marginal
selectivity} (the notion introduced in Townsend \& Schweickert, 1989)
in the dependence of $\left(A,B\right)$ on $\left(\alpha,\beta\right)$:
separately taken, the distribution of $A$ (here, normal) does not
depend on $\beta$, nor the distribution of $B$ on $\alpha$. The
problem is, however, in dealing with the covariance $\mathrm{Cov}\left[A\left(x,y\right),B\left(x,y\right)\right]$.
If it is zero for all $x,y$ (i.e., $A$ and $B$ are always stochastically
independent), the marginal selectivity is all one needs to speak of
$\alpha$ selectively causing $A$ and $\beta$ selectively causing
$B$. In general, however, the covariance depends on both $x$ and
$y$,
\[
\mathrm{Cov}\left[A\left(x,y\right),B\left(x,y\right)\right]=\sigma_{AB}\left(x,y\right).
\]

It would be unsatisfactory to simply ignore stochastic interdependence
among random variables and focus on marginal selectivity alone. It
will be shown in Section \ref{sub:Three-basic-properties} that marginal
selectivity is too weak a concept to allow one to write $A\left(x,y\right)=A\left(x\right),B\left(x,y\right)=B\left(y\right)$,
because $A\left(x\right)$ generally does not preserve its identity
(is not the same random variable) under different $y$, and analogously
for $B\left(y\right)$ under different $x$. So one needs to answer
the conceptual question: under what forms of the dependence of $\sigma_{AB}$
on $\left(x,y\right)$ can one say that the diagram (\ref{diag:Thurstone})
is correct? Even in the seemingly simple special cases one cannot
reply on one's common sense alone. Thus, if $\sigma_{AB}\left(x,y\right)=\sigma_{AB}\left(x\right)$,
what does this tell us about the selectiveness? Even simpler: what
can one conclude if one finds out that $\sigma_{AB}\left(x,y\right)=const\not=0$
across all $x,y$? After all, if $\sigma_{AB}$ is a constant, other
measures of stochastic interdependence will be functions of both $x$
and $y$. For instance, the correlation coefficient then is 
\[
\mathrm{Cor}\left[A\left(x,y\right),B\left(x,y\right)\right]=\frac{const}{\sqrt{\sigma_{AA}\left(x\right)\sigma_{BB}\left(y\right)}}=\rho\left(x,y\right).
\]
One might be tempted to adopt a radical solution: to always attribute
each of $A$ and $B$ to both $\alpha$ and $\beta$ (i.e., deny any
selectiveness), unless $A$ and $B$ are stochastically independent
and exhibit marginal selectivity. But a simple example will show that
such an approach would be far too restrictive to be useful. 

Consider the model in which the observer can be in one of two states
of attention, or activation, called {}``attentive'' and {}``inattentive,''
with probabilities $p$ and $1-p$, respectively. When in the inattentive
state, the stimuli $\alpha,\beta$ (with respective values $x,y$)
cause independent normally distributed images $A\left(x\right),B\left(y\right)$,
with parameters
\[
\begin{array}{c}
\begin{array}{cc}
\mathrm{E}\left[A\left(x\right)\right]=0, & \mathrm{Var}\left[A\left(x\right)\right]=1,\\
\\
\end{array}\\
\begin{array}{cc}
\mathrm{E}\left[B\left(y\right)\right]=0, & \mathrm{Var}\left[B\left(y\right)\right]=1.\end{array}
\end{array}
\]
That is, in the inattentive state the distribution of the images does
not depend on the stimuli at all. When in the attentive state, $A\left(x\right),B\left(y\right)$
remain independent and normally distributed, but their parameters
change as
\[
\begin{array}{c}
\begin{array}{cc}
\mathrm{E}\left[A\left(x\right)\right]=\mu_{A}\left(x\right), & \mathrm{Var}\left[A\left(x\right)\right]=1,\\
\\
\end{array}\\
\begin{array}{cc}
\mathrm{E}\left[B\left(y\right)\right]=\mu_{B}\left(y\right), & \mathrm{Var}\left[B\left(y\right)\right]=1.\end{array}
\end{array}
\]
We note that, first, $A$ and $B$ are stochastically independent
in either state of attention; second, that $A$ does not depend on
$\beta$ and $B$ does not depend on $\alpha$ in either state of
attention; and third, that the switches from one attention state to
another do not depend on the stimuli at all. It is intuitively clear
then that the causality is selective here, in conformity with the
diagram \ref{diag:Thurstone}. But the overall distribution of $A,B$
in this example (a mixture of two bivariate normal distributions),
while obviously satisfying marginal selectivity, has 
\[
\begin{array}{c}
\mathrm{Cov}\left[A\left(x,y\right),B\left(x,y\right)\right]=p\left(1-p\right)\mu_{A}\left(x\right)\mu_{B}\left(y\right)\neq0.\end{array}
\]
In the theory of selectiveness presented later in this paper it is
easily proved that in this situation $A$ only depends on $\alpha$
and $B$ on $\beta$, in spite of their stochastic interdependence
(see Example \ref{exa:SI}).

It is instructive to see that if one ignores the issue of selectiveness
and formulates Thurstone's general model as Thurstone did it himself,
with no restrictions imposed on the covariance $\sigma_{AB}\left(x,y\right)$,
the model becomes redundant and unfalsifiable, not just with respect
to a finite matrix of data, but for any theoretical probability function
\begin{equation}
\begin{array}{c}
p\left(x,y\right)=\Pr\left[A\left(x,y\right)<B\left(x,y\right)\right]\\
\\
\qquad=\Phi\left(\frac{\mu_{B}\left(y\right)-\mu_{A}\left(x\right)}{\sqrt{\sigma_{AA}\left(x\right)+\sigma_{BB}\left(y\right)+2\sigma_{AB}\left(x,y\right)}}\right),
\end{array}\label{eq:Thurstone formula}
\end{equation}
where $\Phi$ is the standard normal integral. Denoting $z\left(x,y\right)=\Phi^{-1}\left(p\left(x,y\right)\right)$,
let $\mu_{A}\left(x\right)$ and $\mu_{B}\left(y\right)$ be any functions
such that
\[
\left|\frac{\mu_{A}\left(x\right)-\mu_{B}\left(y\right)}{z\left(x,y\right)}\right|<M,
\]
for some $M$. Then, putting $\sigma_{AA}\left(x\right)\equiv\sigma_{BB}\left(y\right)\equiv M^{2}/2$,
one can always find the covariance $\sigma_{AB}\left(x,y\right)$
to satisfy $\left(\ref{eq:Thurstone formula}\right)$. On a moment's
reflection, this is what one should expect: without the assumption
of selective influences Thurstone's general model is essentially the
same as the vacuous {}``model'' in which stimuli $\alpha$ and $\beta$
evoke a single normally distributed random variable $D\left(x,y\right)$
(interpretable as {}``subjective difference'' between the value
$x$ of $\alpha$ and the value $y$ of $\beta$), with the decision
rule {}``say that $\beta$ exceeds $\alpha$ (in a given respect)
if $D\left(x,y\right)<0$, otherwise say that $\alpha$ exceeds $\beta$.'' 

The importance of having a principled way of selectively attributing
stochastic images to stimuli they represent is even more obvious in
the context of the Thurstonian-type models applied to same-different
rather than greater-less judgments (Dzhafarov, 2002). When combined
with another constraint, called the {}``well-behavedness'' of the
random variables representing stimuli, the notion of selective influences
has been shown to impose highly non-obvious constraints on the minima
of discrimination functions and the relationship {}``$x$ of $\alpha$
is the best match for $y$ of $\beta$'' (for details, see Dzhafarov,
2003b-c, 2006; Kujala \& Dzhafarov, 2009)

\subsection{History and related notions}

Historically, the notion of selective probabilistic causality was
introduced in psychology by Sternberg (1969), in the context of the
reconstruction of {}``stages'' of mental processing. If $\alpha$
and $\beta$ are certain experimental manipulations (say, size of
memory lists and legibility of items, respectively), and if $A$ and
$B$ are durations of two hypothetical stages of processing (say,
memory search and perception, respectively), then one can hope to
test this hypothesis (that memory search and perception are indeed
two stages, processes occurring one after another) only if one assumes
that $A$ is selectively influenced by $\alpha$ and $B$ by $\beta$.
Sternberg allows for the possibility of $A$ and $B$ being stochastically
interdependent, but it seems that in this case he reduces the selectivity
of the influence of $\alpha,\beta$ upon $A,B$ to a condition that
is weaker than even marginal selectivity: the condition is that the
mean value of $A$ only depends on $\alpha$ and the mean value of
$B$ on $\beta$, while any other parameter of the distributions of
$A$ and $B$, say, variance, may very well depend on both $\alpha$
and $\beta$.

Townsend (1984), basing his analysis on Townsend and Ashby (1983,
Chapter 12), was the first to investigate the notion of selective
influences without assuming that the processes which may be selectively
influenced by factors are organized serially. He proposed to formalize
the notion of selectively influenced and stochastically interdependent
random variables by the concept of {}``indirect nonselectiveness'':
the conditional distribution of the variable $A$ given any value
$b$ of the variable $B$, depends on $\alpha$ only, and, by symmetry,
the conditional distribution of $B$ at any $A=a$ depends on $\beta$
only. Under the name of {}``conditionally selective influence''
this notion was mathematically characterized and generalized in Dzhafarov
(1999). Although interesting in its own right, this notion turns out
to be inadequate, however, for capturing even the most obvious desiderata
for the notion of selective influences. In particular, indirect nonselectiveness
does not imply marginal selectivity, in fact is not even compatible
with it in nontrivial cases. Consider Thurstone's general model again.
If both the indirect nonselectiveness and marginal selectivity are
satisfied, then
\[
\mathrm{E}\left[A|B=b\right]=\mu_{A}\left(x\right)+\frac{\sigma_{AB}\left(x,y\right)}{\sigma_{BB}\left(y\right)}\left(b-\mu_{B}\left(y\right)\right)=\mu_{A|b}\left(x\right),
\]
\[
\mathrm{Var}\left[A|B=b\right]=\left(1-\frac{\sigma_{AB}^{2}\left(x,y\right)}{\sigma_{AA}\left(x\right)\sigma_{BB}\left(y\right)}\right)\sigma_{AA}\left(x\right)=\sigma_{AA|b}\left(x\right),
\]
\[
\mathrm{E}\left[B|A=a\right]=\mu_{B}\left(y\right)+\frac{\sigma_{AB}\left(x,y\right)}{\sigma_{AA}\left(x\right)}\left(a-\mu_{A}\left(x\right)\right)=\mu_{B|a}\left(y\right),
\]
\[
\mathrm{Var}\left[B|A=a\right]=\left(1-\frac{\sigma_{AB}^{2}\left(x,y\right)}{\sigma_{AA}\left(x\right)\sigma_{BB}\left(y\right)}\right)\sigma_{BB}\left(y\right)=\sigma_{BB|a}\left(y\right).
\]

It is not difficult to show that these equations can be satisfied
if and only if either
\begin{lyxlist}{00.00.0000}
\item [{(i)}] $\sigma_{AB}\left(x,y\right)\equiv0$, in which case the
notions of indirect nonselectiveness and of marginal selectivity simply
coincide; or
\item [{(ii)}] the joint distribution of $\left(A,B\right)$ does not depend
on either $\alpha$ or $\beta$ (i.e., $\mu_{A},\mu_{B},\sigma_{AA},\sigma_{BB}$,
and $\sigma_{AB}$ are all constants).
\end{lyxlist}
Neither of these cases, of course, calls for indirect nonselectiveness
as a separate notion.

The difficulty of developing a rigorous and useful definition of selective
influences has nothing to do with the fact that in the above examples
the random outputs in the diagrams of selective influences are unobservable.
They may very well be entirely observable, at least on a sample level.
An example would be two performance tests, with outcomes $A$ and
$B$, conducted on a group of people divided into four subgroups according
as they were trained or not trained for the $A$-test and for the
$B$-test. It may be reasonable to hypothesize (at least for some
pairs of tests) that the random test score $A$ is selectively influenced
by the factor $\alpha$ with the values `not trained for the $A$-test'
and `trained for the $A$-test', while the random test score $B$
is selectively influenced by the factor $\beta$ with the values `not
trained for the $B$-test' and `trained for the $B$-test'. It is
highly likely, however, that the values of $A$ and $B$ will be stochastically
interdependent within each of the four subgroups.

A definition of selective influences we adopt in this paper was proposed
in Dzhafarov (2003a), and further developed in Dzhafarov and Gluhovsky
(2006), Kujala and Dzhafarov (2008), and Dzhafarov and Kujala (2010).
Its rigorous formulation is given in Section \ref{sec:Basic-Notions},
but the gist of it, when applied to a diagram like (\ref{diag:Thurstone}),
is as follows: there is a random entity $R$ whose distribution does
not depend on either of the factors $\alpha,\beta$, such that $A$
can be presented as a transformation of $R$ determined by the value
$x$ of $\alpha$, and $B$ can be presented as a transformation of
$R$ determined by the value $y$ of $\beta$, so that for every allowable
pair $x,y$, the joint distribution of $A,B$ at these $x,y$ is the
same as the joint distributions of the two corresponding transformations
of $R$. In the case of the diagram (\ref{diag:1}), the transformations
are 
\[
f_{1}\left(R,x,y,u\right),f_{2}\left(R,y\right),f_{3}\left(R,x,z,u\right),
\]
where $x,y,z,u$ are values of $\alpha,\beta,\gamma,\delta$, respectively. 

With some additional assumptions this definition has been applied
to Thurstonian-type modeling for same-different comparisons (Dzhafarov,
2003b-c; Kujala \& Dzhafarov, 2009), as well as to the hypothetical
networks of processes underlying response times (Dzhafarov, Schweickert,
Sung, 2004; Schweickert, Fisher, \& Goldstein, 2010). Unexplicated,
intuitive uses of this notion's special versions can even be found
in much earlier publications, such as Bloxom (1972), Schweickert (1982),
and Dzhafarov (1992, 1997). In the latter two publications, for instance,
response time is considered the sum of a signal-dependent and a signal-independent
components, whose durations may very well be stochastically interdependent
(even perfectly positively correlated). 

Any combination of regression-analytic and factor-analytic models
can be viewed as a special version of our definition of selective
influences. When applied to the diagram (\ref{diag:1}), such a model
would have the form

\[
\begin{array}{l}
f_{1}\left(R,x,y,u\right)=h_{1}(C,x,y,u)+g_{1}(x,y,u)S_{1},\\
\\
f_{2}\left(R,y\right)=h_{2}(C,y)+g_{2}(y)S_{2},\\
\\
f_{3}\left(R,y,z,u\right)=h_{3}(C,y,z,u)+g_{3}(y,z,u)S_{3},
\end{array}
\]
where $C$ is a vector of random variables ({}``common sources of
variation''), $S_{1},S_{2},S_{3}$ are {}``specific sources of variation,''
all sources of variation being stochastically independent. To recognize
in this model our definition one should put $R=\left(C,S_{1},S_{2},S_{3}\right)$.
With some distributional assumptions, this model, for every possible
quadruple $\left(x,y,z,u\right)$, has the structure of the nonlinear
factor analysis (McDonald, 1967, 1982); the more familiar linear structure
is obtained by making $h_{1},h_{2},h_{3}$ linear in the components
of $C$.%
\footnote{To avoid confusion, our use of the term {}``factor'' is reserved
for observable external inputs (corresponding to the use of the term
in MANOVA); the unobservable {}``factors'' of the factor analysis
can be referred to in the present context as {}``sources of variation,''
or {}``sources of randomness.'' %
} 

More details on the early history of the notion of selective influences
can be found in Dzhafarov (2003a). The relation of this notion to
that of {}``probabilistic explanation'' in the sense of Suppes and
Zanotti (1982) and to that of {}``probabilistic dimensionality''
in psychometrics (Levine, 2003) are discussed in Dzhafarov and Gluhovsky
(2006). The probabilistic foundations of the issues involved are elaborated
in Dzhafarov and Gluhovsky (2006) and, especially, Dzhafarov and Kujala
(2010).

\subsection*{Plan of the paper}

In this paper we are primarily concerned with necessary (and, under
additional constraints, necessary and sufficient) conditions for diagrams
of selective influences, like (\ref{diag:1}) or (\ref{diag:Thurstone}).
We call these conditions {}``\emph{tests},'' in the same way in
mathematics we speak of the tests for convergence or for divisibility.
That is, the meaning of the term is non-statistical. We assume that
random outputs are known on the population level. The principles of
constructing statistical tests based on our population level tests
are discussed in Section \ref{sub:Sample-level-tests}, but specific
statistical issues are outside the scope of this paper. 

Unlike in Dzhafarov and Kujala (2010), we do not pursue the goal of
maximal generality of formulations, focusing instead on the conceptual
set-up that would apply to commonly encountered experimental designs.
This means a finite number of factors, each having a finite number
of values, with some (not necessarily all) combinations of the values
of the factors serving as allowable treatments. It also means that
the random outcomes influenced by these factors are \emph{random variables}:
their values are vectors of real numbers or elements of countable
sets, rather than more complex structures, such as functions or sets.
To keep the paper self-contained, however, we have added an appendix
in which we formulate the main definitions and statements of the theory
on a much higher level of generality: for arbitrary sets of factors,
arbitrary sets of factors values, and arbitrarily complex random outcomes.

In Section \ref{sec:Basic-Notions} we introduce the notion of several
random variables influenced by several factors and formulate a definition
of selective influences. In Section \ref{sec:Joint-Distribution-Criterion}
we present the Joint Distribution Criterion, a necessary and sufficient
condition for selective influences (or, if one prefers, an alternative
definition thereof), and we list three basic properties of selective
influences. In the same section we formulate the principle by which
one can construct tests for selective influences, on population and
sample levels. In Section \ref{sec:Linear-Feasibility-Test} we describe
the main and universally applicable test for selective influences,
Linear Feasibility Test. The test is universally applicable because
every random outcome and every set of factors can be discretized into
a finite number of categories. The Linear Feasibility Test is both
necessary and sufficient condition for selective influences within
the framework of the chosen discretization of inputs and outputs.
In Section \ref{sec:Distance-type-tests} we study tests based on
{}``pseudo-quasi-metrics'' defined on spaces of jointly distributed
random variables, and we introduce many examples of such tests. Finally,
in Section \ref{sec:Non-distance-tests} we discuss, with less elaboration,
two examples of non-distance-type tests.

\section{Basic Notions\label{sec:Basic-Notions}}

\subsection{Factors, factor points, treatments}

A \emph{factor} $\alpha$, formally, is a set of \emph{factor points},
each of which has the format {}``value (or level) $x$ of factor
$\alpha$.'' In symbols, this can be presented as $\left(x,`\alpha\textnormal{'}\right)$,
where $`\alpha\textnormal{'}$ is the unique name of the set $\alpha$
rather than the set itself. It is convenient to write $x^{\alpha}$
in place of $\left(x,`\alpha\textnormal{'}\right)$. Thus, if a factor
with the name $`intensity\textnormal{'}$ has three levels, $`low,\textnormal{'}$
$`medium,\textnormal{'}$ and $`high,\textnormal{'}$ then this factor
is taken to be the set 
\[
intensity=\left\{ low^{intensity},medium^{intensity},high^{intensity}\right\} .
\]
There is no circularity here, for, say, the factor point $low^{intensity}$
stands for $\left(value=low,name=`intensity\textnormal{'}\right)$
rather than $\left(value=low,set=intensity\right)$.

In the main text we will deal with finite sets of factors $\Phi=\left\{ \alpha_{1},\ldots,\alpha_{m}\right\} $,
with each factor $\alpha\in\Phi$ consisting of a finite number of
factor points, 
\[
\alpha=\left\{ v_{1}^{\alpha},\ldots,v_{k_{\alpha}}^{\alpha}\right\} .
\]
Clearly, $\alpha\cap\beta=\textrm{Ø}$ for any distinct $\alpha,\beta\in\Phi$. 

A \emph{treatment}, as usual, is defined as the set of factor points
containing one factor point from each factor,%
\footnote{We present treatments as sets $\left\{ x_{1}^{\alpha_{1}},\ldots,x_{m}^{\alpha_{m}}\right\} $
rather than vectors $\left(x_{1}^{\alpha_{1}},\ldots,x_{m}^{\alpha_{m}}\right)$,
which would be a correct representation of elements of $\alpha_{1}\times\ldots\times\alpha_{m}$,
because the superscripting we use makes the ordering of the points
$x_{i}^{\alpha_{i}}$ irrelevant.%
} 
\[
\phi=\left\{ x_{1}^{\alpha_{1}},\ldots,x_{m}^{\alpha_{m}}\right\} \in\alpha_{1}\times\ldots\times\alpha_{m}.
\]

The \emph{set of treatments} (used in an experiment or considered
in a theory) is denoted by $T\subset\alpha_{1}\times\ldots\times\alpha_{m}$
and assumed to be nonempty. Note that $T$ need not include all possible
combinations of factor points. This is an important consideration
in view of the {}``canonical rearrangement'' described below. Also,
incompletely crossed designs occur broadly --- in an experiment because
the entire set $\alpha_{1}\times\ldots\times\alpha_{m}$ may be too
large, or in a theory because certain combinations of factor points
may be physically or logically impossible (e.g., contrast and shape
cannot be completely crossed if zero is one of the values for contrast).
\begin{example}
\label{exa:factors}In the diagram (\ref{diag:1}), let $\alpha,\beta,\gamma$,
and $\delta$ have respectively 3, 2, 1, and 2 values. Then these
factors can be presented as 
\[
\Phi=\left\{ \begin{array}{l}
\alpha=\left\{ 1^{\alpha},2^{\alpha},3^{\alpha}\right\} ,\\
\\
\beta=\left\{ 1^{\beta},2^{\beta}\right\} ,\\
\\
\gamma=\left\{ 1^{\gamma}\right\} ,\\
\\
\delta=\left\{ 1^{\delta},2^{\delta}\right\} 
\end{array}\right\} .
\]
The only constraint on one's choice of the labels for the values (here,
$1,2,3$) is that within a factor they should be pairwise distinct.
Due to the unique superscripting, no two factors can share a factor
point. The maximum number of possible treatments in this example is
12, in which case
\[
T=\left\{ \begin{array}{c}
\left\{ 1^{\alpha},1^{\beta},1^{\delta}\right\} ,\left\{ 1^{\alpha},1^{\beta},2^{\delta}\right\} ,\left\{ 1^{\alpha},2^{\beta},1^{\delta}\right\} ,\left\{ 1^{\alpha},2^{\beta},2^{\delta}\right\} ,\\
\\
\left\{ 2^{\alpha},1^{\beta},1^{\delta}\right\} ,\left\{ 2^{\alpha},1^{\beta},2^{\delta}\right\} ,\left\{ 2^{\alpha},2^{\beta},1^{\delta}\right\} ,\left\{ 2^{\alpha},2^{\beta},2^{\delta}\right\} ,\\
\\
\left\{ 3^{\alpha},1^{\beta},1^{\delta}\right\} ,\left\{ 3^{\alpha},1^{\beta},2^{\delta}\right\} ,\left\{ 3^{\alpha},2^{\beta},1^{\delta}\right\} ,\left\{ 3^{\alpha},2^{\beta},2^{\delta}\right\} 
\end{array}\right\} .
\]
We have deleted $1^{\gamma}$ from all treatments because a factor
with a single factor point can always be removed from a diagram (or
added to a diagram, if convenient; see $\textrm{Ø}^{\alpha}$ notation
in Section \ref{sub:Canonical-Rearrangement}).\qed
\end{example}

\subsection{Random variables}

A rigorous definition of a \emph{random variable} (as a special case
of a random entity) is given in the appendix. For simplicity of notation,
any \emph{random variable} $A$ considered in the main text may be
assumed to be a vector of {}``more elementary'' \emph{discrete}
and \emph{continuous} random variables: for a discrete variable, the
set of its possible values is countable (finite or infinite), and
each value possesses a \emph{probability mass}; in the continuous
case, the set of possible values is $\mathbb{R}^{N}$ (vectors with
$N$ real-valued components), and each $a\in\mathcal{A}$ possesses
a conventional \emph{probability density}. So a random variable $A$
consists of several jointly distributed components, $(A_{1},\ldots,A_{k})$,
some (or all) of which are continuous and some (or all) of which are
discrete. Note that random vectors in this terminology are random
variables. The set of possible values of $A$ is denoted $\mathcal{A}$
and each $a\in\mathcal{A}$ has a mass/density value $p\left(a\right)$
associated with it.%
\footnote{Probability mass/density is generally the Radon-Nikodym derivative
with respect to the product of a counting measure and the Lebesgue
measure on $\mathbb{R}^{N}$.%
}

Every vector of jointly distributed random variables $A=(A_{1},\ldots,A_{n})$
is a random variable, and every value $a=\left(a_{1},\ldots,a_{n}\right)\in\mathcal{A}_{1}\times\ldots\times\mathcal{A}_{n}$
of this random variable possesses a \emph{joint mass/density} $p\left(a\right)=p(a_{1},\ldots,a_{n})$;
then for any subvector $(a_{i_{1}},\ldots,a_{i_{k}})$ of $(a_{1},\ldots,a_{n})$
the mass/density $p_{i_{1}\ldots,i_{k}}(a_{i_{1}},\ldots,a_{i_{k}})$
is obtained by summing and/or integrating $p(a_{1},\ldots,a_{n})$
across all possible values of $(a_{1},\ldots,a_{n})-(a_{i_{1}},\ldots,a_{i_{k}})$.
Note, however, that a vector of random variables $A=(A_{1},\ldots,A_{n})$
need not be a random variable, because $(A_{1},\ldots,A_{n})$ need
not possess a joint distribution.

We use the relational symbol $\sim$ in the meaning of {}``is distributed
as.'' $A\sim B$ is well defined irrespective of whether $A$ and
$B$ are jointly distributed. 

Let, for each treatment $\phi\in T$, there be a vector of jointly
distributed random variables with the set of possible values $\mathcal{A}=\mathcal{A}_{1}\times\ldots\times\mathcal{A}_{n}$
(that does not depend on $\phi$) and probability mass/density $p_{\phi}\left(a_{1},\ldots,a_{n}\right)$
that depends on $\phi$.%
\footnote{\label{fn:Invariance}The invariance of $\mathcal{A}$ with respect
to $\phi$ (more generally, the invariance of the observation space
for $A$ with respect to $\phi$) is convenient to assume, but it
is not essential for the theory. Its two justifications are that (a)
this requirement makes it natural to speak of {}``one and the same''
$A$ whose distribution changes with $\phi$ rather than to speak
(more correctly) of different random variables $A\left(\phi\right)$
for different $\phi$; and (b) in the context of selective influences
one can always redefine the observation spaces for different treatments
$\phi$ to make them coincide (see Remark \ref{rem:Invariance} in
the appendix). %
} Then we say that we have a \emph{vector of jointly distributed random
variables that depends on treatment} $\phi$, and write 
\[
A(\phi)=(A_{1},\ldots,A_{n})(\phi),\quad\phi\in T.
\]
A correct way of thinking of $A(\phi)$ is that it represents a \emph{set
of vectors of jointly distributed random }variables, each of these
vectors being labeled (indexed) by a particular treatment. Any subvector
of $A\left(\phi\right)$ should also be written with the argument
$\phi$, say, $(A_{1},A_{2},A_{3})\left(\phi\right)$. If $\phi$
is explicated as $\phi=\left\{ x_{1}^{\alpha_{1}},\ldots,x_{m}^{\alpha_{m}}\right\} $
or, say, $\phi=\left\{ 3^{\alpha},1^{\beta},1^{\delta}\right\} $,
we will write $A(x_{1}^{\alpha_{1}},\ldots,x_{m}^{\alpha_{m}})$ or
$(A,B,C)\left(3^{\alpha},1^{\beta},1^{\delta}\right)$ instead of
more correct $A(\{x_{1}^{\alpha_{1}},\ldots,x_{m}^{\alpha_{m}}\})$
or $(A,B,C)(\{3^{\alpha},1^{\beta},1^{\delta}\})$.

It is important to note that for distinct treatments $\phi_{1}$ and
$\phi_{2}$ the corresponding $A(\phi_{1})$ and $A(\phi_{2})$ \emph{do
not possess a joint distribution}, they are \emph{stochastically unrelated}.
This is easy to understand: since $\phi_{1}$ and $\phi_{2}$ are
mutually exclusive conditions for observing values of $A$, there
is no non-arbitrary way of choosing which value $a=(a_{1},\ldots,a_{n})$
observed at $\phi_{1}$ should be paired with which value $a'=(a'_{1},\ldots,a'_{n})$
observed at $\phi_{2}$. To consider $A(\phi_{1})$ and $A(\phi_{2})$
stochastically independent and to pair every possible value of $A(\phi_{1})$
with every possible value $A(\phi_{2})$ is as arbitrary as, say,
to consider them positively correlated and to pair every quantile
of $A(\phi_{1})$ with the corresponding quantile of $A(\phi_{2})$. 
\begin{example}
\label{exa:rvs}In diagram (\ref{diag:1}), let $\Phi$ and $T$ be
as in Example \ref{exa:factors}, and let $A,B,C$ be binary, 0/1,
variables. Then $(A,B,C)(\phi)$ is defined, for each $\phi=\{x^{\alpha},y^{\beta},z^{\delta}\}$,
by a table of the following form:

\

\begin{center}%
\begin{tabular}{ccc|ccc|c|}
\hline 
\multicolumn{1}{|c}{$\alpha$} & $\beta$ & $\delta$ & $A$ & $B$ & $C$ & $\Pr$\tabularnewline
\hline 
\hline 
\multicolumn{1}{|c}{$x$} & $y$ & $z$ & 0 & 0 & 0 & $p_{000}$\tabularnewline
\cline{1-3} 
 &  &  & 0 & 0 & 1 & $p_{001}$\tabularnewline
 &  &  & 0 & 1 & 0 & $p_{010}$\tabularnewline
 &  &  & 0 & 1 & 1 & $p_{011}$\tabularnewline
 &  &  & 1 & 0 & 0 & $p_{100}$\tabularnewline
 &  &  & 1 & 0 & 1 & $p_{101}$\tabularnewline
 &  &  & 1 & 1 & 0 & $p_{110}$\tabularnewline
 &  &  & 1 & 1 & 1 & $p_{111}$\tabularnewline
\cline{4-7} 
\end{tabular}\end{center}

\

\protect{\noindent}separately for each of the 12 treatments.\qed
\end{example}

\subsection{Selective influences}

Given a set of factors $\Phi=\left\{ \alpha_{1},\ldots,\alpha_{m}\right\} $
and a vector $A(\phi)=(A_{1},\ldots,A_{n})(\phi)$ of random variables
depending on treatment, a \emph{diagram of selective influences} is
a mapping
\begin{equation}
M:\left\{ 1,\ldots,n\right\} \rightarrow2^{\Phi}\label{eq:DSI}
\end{equation}
($2^{\Phi}$ being the set of subsets of $\Phi$), with the interpretation
that 
\[
\Phi_{i}=M\left(i\right)
\]
is the subset of factors (which may be empty) \emph{selectively influencing}
$A_{i}$ ($i=1,\ldots,n$). The definition of selective influences
is yet to be given (Definition \ref{def:(Selective-Influence)}),
but for the moment think simply of arrows drawn from factors to random
variables (or vice versa). The subset of factors $\Phi_{i}$ influencing
$A_{i}$ determines, for any treatment $\phi\in T$, the subtreatments
$\phi_{\Phi_{i}}$ defined as
\[
\phi_{\Phi_{i}}=\left\{ x^{\alpha}\in\phi:\alpha\in\Phi_{i}\right\} ,\quad i=1,\ldots,n.
\]
Subtreatments $\phi_{\Phi_{i}}$ across all $\phi\in T$ can be viewed
as \emph{admissible} \emph{values} of the subset of factors $\Phi_{i}$
($i=1,\ldots,n$). Note that $\phi_{\Phi_{i}}$ is empty whenever
$\Phi_{i}$ is empty. 
\begin{example}
\label{exa:subtreatments}In the diagram \ref{diag:1}, having enumerated
$A,B,C$ by 1,2,3, respectively, $\Phi_{1}=\left\{ \alpha,\beta,\delta\right\} $,
$\Phi_{2}=\left\{ \beta\right\} $, $\Phi_{3}=\left\{ \alpha,\gamma,\delta\right\} $.
If the factor points are as in Examples \ref{exa:factors} and \ref{exa:rvs},
then, choosing $\phi=\left\{ 3^{\alpha},1^{\beta},1^{\gamma},2^{\delta}\right\} $,
we have $\phi_{\Phi_{1}}=\left\{ 3^{\alpha},1^{\beta},2^{\delta}\right\} $,
$\phi_{\Phi_{2}}=\left\{ 1^{\beta}\right\} $, and $\phi_{\Phi_{3}}=\left\{ 3^{\alpha},1^{\gamma},2^{\delta}\right\} $
(where $\gamma$ and its only point $1^{\gamma}$ can be omitted everywhere,
making, in particular, the treatments $\phi_{\Phi_{1}}$ and $\phi$
coincide).\qed
\end{example}
The definition below is a special case of the definition of selective
influences given in the appendix. This definition will be easier to
justify in terms of the Joint Distribution Criterion formulated in
the next section. 
\begin{defn}[\emph{Selective influences}]
\label{def:(Selective-Influence)} A vector of random variables $A\left(\phi\right)=(A_{1},\ldots,A_{n})\left(\phi\right)$
is said to satisfy a diagram of selective influences (\ref{eq:DSI})
if there is a random variable%
\footnote{\label{fn:entity}Even though $A\left(\phi\right)$ is a random variable,
and $\Phi$ is a finite set of factors containing a finite set of
factor points each, the requirement in the definition that $R$ be
a random variable is unnecessarily restrictive: it is sufficient to
require the existence of a \emph{random entity $R$ }distributed on
some probability space $\left(\mathcal{R},\Sigma_{\mathcal{R},},\mu\right)$
(see the appendix). It is shown in the appendix, however, based on
the Joint Distribution Criterion, that if the definition is satisfied
with an arbitrary $R$, then the latter can always be chosen to be
a random variable --- discrete, continuous, or mixed according as
the variable $A\left(\phi\right)$ is discrete, continuous, or mixed.
(Recall that in our terminology every vector of random variables is
a random variable.) Moreover, $R$ can always be chosen to be distributed
unit-uniformly, or according to any distribution function strictly
increasing on any interval of reals constituting $\mathcal{R}$.%
} $R$ taking values on some set $\mathcal{R}$, and functions $f_{i}:\Phi_{i}\times R\rightarrow\mathcal{A}_{i}$
($i=1,\ldots,n$), such that, for any treatment $\phi\in T$, 
\begin{equation}
(A_{1},\ldots,A_{n})\left(\phi\right)\sim(f_{1}(\phi_{\Phi_{1}},R),\ldots,f_{n}(\phi_{\Phi_{n}},R)).\label{eq:SI def}
\end{equation}
We write then, schematically, $(A_{1},\ldots,A_{n})\looparrowleft(\Phi_{1},\ldots,\Phi_{n})$.

\

The qualifier {}``schematically'' in reference to $(A_{1},\ldots,A_{n})\looparrowleft(\Phi_{1},\ldots,\Phi_{n})$
is due to the fact that $(A_{1},\ldots,A_{n})$ is not well-defined
without mentioning a treatment $\phi$ at which these variables are
taken. This notation, therefore, is merely a compact way of referring
to the diagram (\ref{eq:DSI}). \end{defn}
\begin{example}
\label{exa:SI}Consider the Thurstonian {}``mixture'' model described
in the introduction:
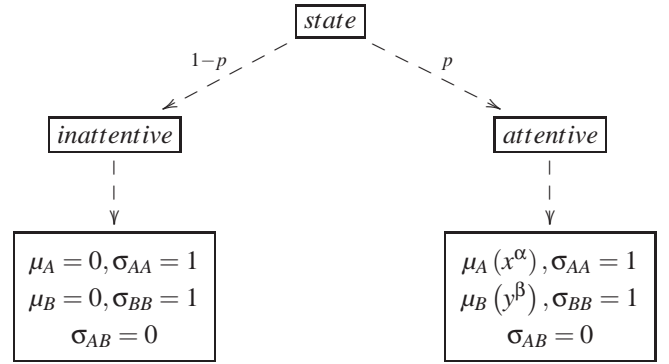
\[
\xymatrix{ & \boxed{state\ar@{-->}[dl]_{1-p}\ar@{-->}[dr]^{p}}\\
\boxed{inattentive\ar@{-->}[d]} &  & \boxed{attentive\ar@{-->}[d]}\\
\boxed{\begin{array}{c}
\mu_{A}=0,\sigma_{AA}=1\\
\mu_{B}=0,\sigma_{BB}=1\\
\sigma_{AB}=0
\end{array}} &  & \boxed{\begin{array}{c}
\mu_{A}\left(x^{\alpha}\right),\sigma_{AA}=1\\
\mu_{B}\left(y^{\beta}\right),\sigma_{BB}=1\\
\sigma_{AB}=0
\end{array}}
}
\]

\protect{\noindent}The selectivity $(A,B)\looparrowleft(\alpha,\beta)$
here is shown by
\begin{enumerate}
\item putting $R=\left(S,N_{1},N_{2}\right)$, where $S$ is a Bernoulli
(0/1) variable with $\Pr\left[S=1\right]=p$, $N_{1},N_{2}$ are standard
normal variables, and the three variables are independent;
\item defining 
\[
\begin{array}{r}
\left(f_{1}\left(x^{\alpha},\left(S,N_{1},N_{2}\right)\right),f_{2}\left(y^{\beta},\left(S,N_{1},N_{2}\right)\right)\right)\\
\\
=\left(\mu_{A}\left(x^{\alpha}\right)S+N_{1},\mu_{B}\left(y^{\beta}\right)S+N_{2}\right);
\end{array}
\]
 
\item and observing that 
\[
\left(\mu_{A}\left(x^{\alpha}\right)S+N_{1},\mu_{B}\left(y^{\beta}\right)S+N_{2}\right)\sim\left(A,B\right)\left(x^{\alpha},y^{\beta}\right)
\]
 for all treatments $\left\{ x^{\alpha},y^{\beta}\right\} .$\qed 
\end{enumerate}
\end{example}
\begin{rem}
Note that the components of $(f_{1}(\phi_{\Phi_{1}},R),\ldots,f_{n}(\phi_{\Phi_{n}},R))$
are jointly distributed for any given $\phi$ because they are functions
of one and the same random variable. The components of $(A_{1},\ldots,A_{n})(\phi)$
are jointly distributed for any given $\phi$ by definition. There
is, however, no joint distribution of these two vectors, $(f_{1}(\phi_{\Phi_{1}},R),\ldots,f_{n}(\phi_{\Phi_{n}},R))$
and $(A_{1},\ldots,A_{n})(\phi)$, for any $\phi$; and, as emphasized
earlier, no joint distribution for $(A_{1},\ldots,A_{n})(\phi_{1})$
and $(A_{1},\ldots,A_{n})(\phi_{2})$, for distinct $\phi_{1}$ and
$\phi_{2}$. 
\end{rem}

\section{\label{sec:Joint-Distribution-Criterion}Joint Distribution Criterion}

\subsection{Canonical Rearrangement\label{sub:Canonical-Rearrangement}}

The simplest diagram of selective influences is \emph{bijective},
\begin{equation}
\boxed{\xymatrix{\alpha_{1}\ar[d] & \ldots & \alpha_{n}\ar[d]\\
A_{1} & \ldots & A_{n}
}
}\label{diag:bijective DSI}
\end{equation}
In this case we write $(A_{1},\ldots,A_{n})\looparrowleft(\alpha_{1},\ldots,\alpha_{n})$
instead of $(A_{1},\ldots,A_{n})\looparrowleft(\Phi_{1}=\{\alpha_{1}\},\ldots,\Phi_{n}=\{\alpha_{n}\})$. 

We can simplify the subsequent discussion without sacrificing generality
by agreeing to reduce each diagram of selective influences to a bijective
form, by appropriately redefining factors and treatments. It is almost
obvious how this should be done. Given the subsets of factors $\Phi_{1}\ldots,\Phi_{n}$
determined by a diagram of selective influences (\ref{eq:DSI}), each
$\Phi_{i}$ can be viewed as a factor identified with the set of factor
points 
\[
\alpha_{i}^{*}=\left\{ (\phi_{\Phi_{i}})^{\alpha_{i}^{*}}:\phi\in T\right\} ,
\]
in accordance with the notation we have adopted for factor points:
$(\phi_{\Phi_{i}})^{\alpha_{i}^{*}}=(\phi_{\Phi_{i}},`\alpha^{*}\textnormal{'})$.
If $\Phi_{i}$ is empty, then $\phi_{\Phi_{i}}$ is empty too, and
we should designate a certain value, say $\textrm{Ø}^{\alpha_{i}^{*}}$,
as a dummy factor point (the only element of factor $\alpha_{i}^{*}$).
The set of treatments $T$ for the original factors $\{\alpha_{1},\ldots,\alpha_{m}\}$
should then be redefined for the vector of new factors $(\alpha_{1}^{*},\ldots,\alpha_{n}^{*})$
as
\[
T^{*}=\left\{ \left\{ (\phi_{\Phi_{1}})^{\alpha_{1}^{*}},\ldots,(\phi_{\Phi_{n}})^{\alpha_{n}^{*}}\right\} :\phi\in T\right\} \subset\alpha_{1}^{*}\times\ldots\times\alpha_{n}^{*}.
\]
We call this redefinition of factor points, factors, and treatments
the \emph{canonical rearrangement.} 
\begin{example}
Diagram (\ref{diag:1}), with the factors defined as in Example \ref{exa:factors}
(with $\gamma$ omitted), is reduced to a bijective form as follows:
\[
\begin{array}{l}
\alpha^{*}=\left\{ \left\{ x_{1}^{\alpha},x_{2}^{\beta},x_{3}^{\delta}\right\} ^{\alpha^{*}}:\left\{ x_{1}^{\alpha},x_{2}^{\beta},x_{3}^{\delta}\right\} \in\alpha\times\beta\times\delta\right\} ,\\
\\
\beta_{2}^{*}=\left\{ \left\{ y^{\beta}\right\} ^{\beta^{*}}:y^{\beta}\in\beta\right\} ,\\
\\
\gamma_{3}^{*}=\left\{ \left\{ z_{1}^{\alpha},z_{3}^{\delta}\right\} ^{\gamma^{*}}:\left\{ z_{1}^{\alpha},z_{3}^{\delta}\right\} \in\alpha\times\delta\right\} ,
\end{array}
\]
with, respectively, 12, 2, and 6 factor points, and
\[
T^{*}=\left\{ \begin{array}{r}
\left\{ \left\{ x_{1}^{\alpha},x_{2}^{\beta},x_{3}^{\delta}\right\} ^{\alpha^{*}},\left\{ y^{\beta}\right\} ^{\beta^{*}},\left\{ z_{1}^{\alpha},z_{3}^{\delta}\right\} ^{\gamma^{*}}\right\} \in\alpha_{1}^{*}\times\beta_{2}^{*}\times\gamma_{3}^{*}\\
\\
:x_{1}^{\alpha}=z_{1}^{\alpha},x_{2}^{\beta}=y^{\beta},x_{3}^{\delta}=z_{3}^{\delta}
\end{array}\right\} ,
\]
the number of treatments, obviously remaining the same, 12, as for
the original factors.\qed
\end{example}
The purpose of canonical rearrangement is to achieve a bijective correspondence
between factors and the random variables selectively influenced by
these factors. Equivalently, we may say that the random variables
following canonical rearrangement can be indexed by the factors (assumed
to be) selectively influencing them. Thus, if we test the hypothesis
that $(A_{1},\ldots,A_{n})\looparrowleft(\alpha_{1},\ldots,\alpha_{n})$,
we can, when convenient, write $A_{\left\{ \alpha_{1}\right\} }$
in place of $A_{1}$, $A_{\left\{ \alpha_{2}\right\} }$ in place
of $A_{2},$ etc.

\subsection{The criterion}

From now on let us assume that we deal with bijective diagrams of
selective influences, (\ref{diag:bijective DSI}). The notation $\phi_{\Phi_{i}}=\phi_{\left\{ \alpha_{i}\right\} }$
then indicates the singleton set $\left\{ x^{\alpha_{i}}\right\} \subset\phi$.
As usual, we write $x^{\alpha_{i}}$ in place of $\left\{ x^{\alpha_{i}}\right\} $:
\[
\phi_{\left\{ \alpha_{i}\right\} }=\left\{ x_{1}^{\alpha_{1}},\ldots,x_{n}^{\alpha_{n}}\right\} _{\left\{ \alpha_{i}\right\} }=x_{i}^{\alpha_{i}}.
\]
The definition of selective influences (Definition \ref{def:(Selective-Influence)})
then acquires the following form: 
\begin{defn}[\emph{Selective influences, bijective form}]
\label{def:(SI,-bijective)} A vector of random variables $A(\phi)=(A_{1},\ldots,A_{n})(\phi)$
is said to satisfy a diagram of selective influences (\ref{diag:bijective DSI}),
and we write $(A_{1},\ldots,A_{n})\looparrowleft(\alpha_{1},\ldots,\alpha_{n})$,
if, for some random variable%
\footnote{See footnote \ref{fn:entity}.%
} $R$ and for any treatment $\phi\in T$,
\begin{equation}
(A_{1},\ldots,A_{n})(\phi)\sim\left(f_{1}(\phi_{\{\alpha_{1}\}},R),\ldots,f_{n}(\phi_{\{\alpha_{n}\}},R)\right),
\end{equation}
where $f_{i}:\alpha_{i}\times\mathcal{R}\rightarrow\mathcal{A}_{i}$
($i=1,\ldots,n$) are some functions, with $\mathcal{R}$ denoting
the set of possible values of $R$. 

\

This definition is difficult to put to work, as it refers to an existence
of a random variable $R$ without showing how one can find it or prove
that it cannot be found. In Dzhafarov and Kujala (2010), however,
we have formulated a necessary and sufficient condition for $(A_{1},\ldots,A_{n})\looparrowleft(\alpha_{1},\ldots,\alpha_{n})$
which circumvents this problem. \end{defn}
\begin{criterion}[\emph{Joint Distribution Criterion}, JDC]
A vector of random variables $A(\phi)=(A_{1},\ldots,A_{n})(\phi)$
satisfies a diagram of selective influences (\ref{diag:bijective DSI})
if and only if there is a vector of jointly distributed random variables
\[
H=\left(\overset{\textnormal{for }\alpha_{1}}{\overbrace{H_{x_{1}^{\alpha_{1}}},\ldots,H_{x_{k_{1}}^{\alpha_{i}}}}},\ldots,\overset{\textnormal{for }\alpha_{n}}{\overbrace{H_{x_{1}^{\alpha_{n}}},\ldots,H_{x_{k_{n}}^{\alpha_{n}}}}}\right),
\]
one random variable for each factor point of each factor, such that
\begin{equation}
\left(H_{\phi_{\{\alpha_{1}\}}},\dots,H_{\phi_{\{\alpha_{n}\}}}\right)\sim A(\phi)
\end{equation}
for every treatment $\phi\in T$\textup{\emph{.}}
\end{criterion}
Due to its central role, the simple proof of this criterion (for the
general case of arbitrary factors and sets of random entities) is
reproduced in the appendix. The vector $H$ in the formulation of
the JDC is referred to as the \emph{JDC-vector for} $A(\phi)$, or
the \emph{hypothetical JDC-vector for} $A(\phi)$, if the existence
of such a vector of jointly distributed variables is in question.
\begin{example}
\label{exa:JDC}For the diagram of selective influences
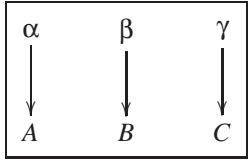
\[
\boxed{\xymatrix{\alpha\ar[d] & \beta\ar[d] & \gamma\ar[d]\\
A & B & C
}
}
\]
with $\alpha=\left\{ 1^{\alpha},2^{\alpha}\right\} $, $\beta=\left\{ 1^{\beta},2^{\beta},3^{\beta}\right\} $,
$\gamma=\left\{ 1^{\gamma},2^{\gamma},3^{\gamma},4^{\gamma}\right\} $,
and the set of allowable treatments
\[
T=\left\{ \begin{array}{c}
\left\{ 1^{\alpha},2^{\beta},1^{\gamma}\right\} ,\left\{ 1^{\alpha},2^{\beta},3^{\gamma}\right\} ,\left\{ 2^{\alpha},1^{\beta},4^{\gamma}\right\} ,\\
\\
\left\{ 1^{\alpha},3^{\beta},1^{\gamma}\right\} ,\left\{ 2^{\alpha},3^{\beta},2^{\gamma}\right\} 
\end{array}\right\} ,
\]
the hypothetical JDC-vector is
\[
\left(H_{1^{\alpha}},H_{2^{\alpha}},H_{1^{\beta}},H_{2^{\beta}},H_{3^{\beta}},H_{1^{\gamma}},H_{2^{\gamma}},H_{3^{\gamma}},H_{4^{\gamma}}\right),
\]
the hypothesis being that
\[
\begin{array}{l}
(H_{1^{\alpha}},H_{2^{\beta}},H_{1^{\gamma}})\sim(A,B,C)\left(1^{\alpha},2^{\beta},1^{\gamma}\right),\\
\\
(H_{1^{\alpha}},H_{2^{\beta}},H_{3^{\gamma}})\sim(A,B,C)\left(1^{\alpha},2^{\beta},3^{\gamma}\right),\\
\\
(H_{2^{\alpha}},H_{1^{\beta}},H_{4^{\gamma}})\sim(A,B,C)\left(2^{\alpha},1^{\beta},4^{\gamma}\right),\\
\\
(H_{1^{\alpha}},H_{3^{\beta}},H_{1^{\gamma}})\sim(A,B,C)\left(1^{\alpha},3^{\beta},1^{\gamma}\right),\\
\\
(H_{2^{\alpha}},H_{3^{\beta}},H_{2^{\gamma}})\sim(A,B,C)\left(2^{\alpha},3^{\beta},2^{\gamma}\right).
\end{array}
\]
This means, in particular, that $H_{1^{\alpha}}$ and $H_{2^{\alpha}}$
have the same set of values as $A$ (which, by our convention, does
not depend on treatment), the set of values for $H_{1^{\beta}}$,
$H_{2^{\beta}}$, and $H_{3^{\beta}}$ is the same as that of $B$,
and the set of values for $H_{1^{\gamma}}$, $H_{2^{\gamma}}$, $H_{3^{\gamma}}$,
and $H_{4^{\gamma}}$ is the same as that of $C$.\qed
\end{example}
The JDC prompts a simple justification for our definition of selective
influences. Let $(A,B,C)\looparrowleft(\alpha,\beta,\gamma)$, as
in the previous example, with each factors containing two factor points.
Consider all treatments $\phi$ in which the factor point of $\alpha$
is fixed, say, at $1^{\alpha}$. If $(A,B,C)\looparrowleft(\alpha,\beta,\gamma)$,
then in the vectors of random variables
\[
(A,B,C)\left(1^{\alpha},2^{\beta},1^{\gamma}\right),(A,B,C)\left(1^{\alpha},2^{\beta},3^{\gamma}\right),(A,B,C)\left(1^{\alpha},3^{\beta},1^{\gamma}\right),
\]
the marginal distribution of the variable $A$ is one and the same,
\[
A\left(1^{\alpha},2^{\beta},1^{\gamma}\right)\sim A\left(1^{\alpha},2^{\beta},3^{\gamma}\right)\sim A\left(1^{\alpha},3^{\beta},1^{\gamma}\right).
\]
But the intuition of selective influences requires more: that we can
denote this variable $A\left(1^{\alpha}\right)$ because it \emph{preserves
its identity} (and not just its distribution) no matter what other
variables it is paired with, $(B,C)\left(2^{\beta},1^{\gamma}\right)$,
$(B,C)\left(2^{\beta},3^{\gamma}\right)$, or $(B,C)\left(3^{\beta},1^{\gamma}\right)$.
Analogous statements hold for $A\left(2^{\alpha}\right)$, $B\left(2^{\beta}\right)$,
$B\left(3^{\beta}\right)$, $C\left(1^{\gamma}\right)$. The JDC formalizes
the intuitive notion of variables {}``preserving their identity''
when entering in various combinations with each other: there are jointly
distributed random variables 
\[
H_{1^{\alpha}},H_{2^{\alpha}},H_{1^{\beta}},H_{2^{\beta}},H_{3^{\beta}},H_{1^{\gamma}},H_{2^{\gamma}},H_{3^{\gamma}},H_{4^{\gamma}}
\]
whose identity is defined by this joint distribution; when $H_{1^{\alpha}}$
is combined with random variables $H_{2^{\beta}}$ and $H_{3^{\gamma}}$,
it forms the triad $(H_{1^{\alpha}},H_{2^{\beta}},H_{1^{\gamma}})$
whose distribution is the same as that of $(A,B,C)\left(1^{\alpha},2^{\beta},1^{\gamma}\right)$;
when the same random variable $H_{1^{\alpha}}$ is combined with random
variables $H_{2^{\beta}}$ and $H_{3^{\gamma}}$, the triad $(H_{1^{\alpha}},H_{2^{\beta}},H_{3^{\gamma}})$
is distributed as $(A,B,C)\left(1^{\alpha},2^{\beta},3^{\gamma}\right)$;
and so on --- the key concept being that it is \emph{one and the same}
$H_{1^{\alpha}}$ which is being paired with other variables, as opposed
to different random variables $A\left(1^{\alpha},2^{\beta},1^{\gamma}\right),A\left(1^{\alpha},2^{\beta},3^{\gamma}\right),A\left(1^{\alpha},3^{\beta},1^{\gamma}\right)$
which are identically distributed (cf. Example \ref{exa:Identity}
below, which shows that the identity is not generally preserved if
all we know is marginal selectivity).

\subsection{Three basic properties of selective influences\label{sub:Three-basic-properties}}

The three properties in question are immediate consequences of JDC.

\subsubsection{Property 1: Nestedness.}

For any subset $\left\{ i_{1},\ldots,i_{k}\right\} $ of $\left\{ 1,\ldots,n\right\} $,
if $(A_{1},\ldots,A_{n})\looparrowleft(\alpha_{1},\ldots,\alpha_{n})$
then $(A_{i_{1}},\ldots,A_{i_{k}})\looparrowleft(\alpha_{i_{1}},\ldots,\alpha_{i_{k}})$.
\begin{example}
In Example \ref{exa:JDC}, if $(A,B,C)\looparrowleft(\alpha,\beta,\gamma)$,
then $(A,C)\looparrowleft(\alpha,\gamma)$, because the JDC criterion
for $(A,B,C)\looparrowleft(\alpha,\beta,\gamma)$ implies that $(H_{1^{\alpha}},H_{2^{\alpha}},H_{1^{\gamma}},H_{2^{\gamma}},H_{3^{\gamma}},H_{4^{\gamma}})$
are jointly distributed, and that
\[
\begin{array}{l}
(H_{1^{\alpha}},H_{1^{\gamma}})\sim(A,C)\left(1^{\alpha},1^{\gamma}\right),\\
(H_{1},H_{3^{\gamma}})\sim(A,C)\left(1^{\alpha},3^{\gamma}\right),\\
(H_{2^{\alpha}},H_{2^{\gamma}})\sim(A,C)\left(2^{\alpha},2^{\gamma}\right),\\
(H_{2^{\alpha}},H_{4^{\gamma}})\sim(A,C)\left(2^{\alpha},4^{\gamma}\right).
\end{array}
\]
Analogously, $(A,B)\looparrowleft(\alpha,\beta)$ and $(B,C)\looparrowleft(\beta,\gamma)$.
Statements with $\looparrowleft$ involving a single variable merely
indicate the dependence of its distribution on the corresponding factor:
thus, $A\looparrowleft\alpha$ simply mean that the distribution of
$A\left(x^{\alpha},y^{\beta},z^{\gamma}\right)$ does not depend on
$y^{\beta},z^{\gamma}$.\qed 
\end{example}

\subsubsection{Property 2: Complete Marginal Selectivity}

For any subset $\left\{ i_{1},\ldots,i_{k}\right\} $ of $\left\{ 1,\ldots,n\right\} $,
if $(A_{1},\ldots,A_{n})\looparrowleft(\alpha_{1},\ldots,\alpha_{n})$
then the $k$-marginal distribution%
\footnote{$k$-marginal distribution is the distribution of a subset of $k$
random variables ($k\geq1$) in a set of $n\geq k$ variables. In
Townsend and Schweickert (1989) the property was formulated for 1-marginals
of a pair of random variables. The adjective {}``complete'' we use
with {}``marginal selectivity'' is to emphasize that we deal with
all possible marginals rather than with just 1-marginals. %
} of $(A_{i_{1}},\ldots,A_{i_{k}})(\phi)$ does not depend on points
of the factors outside $(\alpha_{i_{1}},\ldots,\alpha_{i_{k}})$.
In particular, the distribution of $A_{i}$ only depends on points
of $\alpha_{i}$, $i=1,\ldots,n$. 

This is, of course, a trivial consequence of the nestedness property,
but its importance lies in that it provides the easiest to check necessary
condition for selective influences. 
\begin{example}
Let the factors, factor points, and the set of treatments be as in
Example \ref{exa:JDC}. Let the distributions of $\left(A,B,C\right)$
at the five different treatments be as shown:

\

\begin{center}%
\begin{tabular}{ccc|ccc|c|}
\hline 
\multicolumn{1}{|c}{$\alpha$} & $\beta$ & $\gamma$ & $A$ & $B$ & $C$ & $\Pr$\tabularnewline
\hline 
\hline 
\multicolumn{1}{|c}{1} & 2 & 1 & 0 & 0 & 0 & $.2$\tabularnewline
\cline{1-3} 
 &  &  & 0 & 0 & 1 & $.1$\tabularnewline
 &  &  & 0 & 1 & 0 & $.1$\tabularnewline
 &  &  & 0 & 1 & 1 & $.1$\tabularnewline
 &  &  & 1 & 0 & 0 & $.1$\tabularnewline
 &  &  & 1 & 0 & 1 & $.1$\tabularnewline
 &  &  & 1 & 1 & 0 & $.1$\tabularnewline
 &  &  & 1 & 1 & 1 & $.2$\tabularnewline
\cline{4-7} 
\end{tabular}$\quad$%
\begin{tabular}{ccc|ccc|c|}
\hline 
\multicolumn{1}{|c}{$\alpha$} & $\beta$ & $\gamma$ & $A$ & $B$ & $C$ & $\Pr$\tabularnewline
\hline 
\hline 
\multicolumn{1}{|c}{1} & 2 & 3 & 0 & 0 & 0 & $0$\tabularnewline
\cline{1-3} 
 &  &  & 0 & 0 & 1 & $.3$\tabularnewline
 &  &  & 0 & 1 & 0 & $.2$\tabularnewline
 &  &  & 0 & 1 & 1 & $0$\tabularnewline
 &  &  & 1 & 0 & 0 & $.1$\tabularnewline
 &  &  & 1 & 0 & 1 & $.1$\tabularnewline
 &  &  & 1 & 1 & 0 & $.1$\tabularnewline
 &  &  & 1 & 1 & 1 & $.2$\tabularnewline
\cline{4-7} 
\end{tabular}$\quad$%
\begin{tabular}{ccc|ccc|c|}
\hline 
\multicolumn{1}{|c}{$\alpha$} & $\beta$ & $\gamma$ & $A$ & $B$ & $C$ & $\Pr$\tabularnewline
\hline 
\hline 
\multicolumn{1}{|c}{2} & 1 & 4 & 0 & 0 & 0 & $.3$\tabularnewline
\cline{1-3} 
 &  &  & 0 & 0 & 1 & $0$\tabularnewline
 &  &  & 0 & 1 & 0 & $.3$\tabularnewline
 &  &  & 0 & 1 & 1 & $0$\tabularnewline
 &  &  & 1 & 0 & 0 & $.3$\tabularnewline
 &  &  & 1 & 0 & 1 & $0$\tabularnewline
 &  &  & 1 & 1 & 0 & $0$\tabularnewline
 &  &  & 1 & 1 & 1 & $.1$\tabularnewline
\cline{4-7} 
\end{tabular}

\

\begin{tabular}{ccc|ccc|c|}
\hline 
\multicolumn{1}{|c}{$\alpha$} & $\beta$ & $\gamma$ & $A$ & $B$ & $C$ & $\Pr$\tabularnewline
\hline 
\hline 
\multicolumn{1}{|c}{1} & 3 & 1 & 0 & 0 & 0 & $.4$\tabularnewline
\cline{1-3} 
 &  &  & 0 & 0 & 1 & $.1$\tabularnewline
 &  &  & 0 & 1 & 0 & $0$\tabularnewline
 &  &  & 0 & 1 & 1 & $0$\tabularnewline
 &  &  & 1 & 0 & 0 & $0$\tabularnewline
 &  &  & 1 & 0 & 1 & $.2$\tabularnewline
 &  &  & 1 & 1 & 0 & $.1$\tabularnewline
 &  &  & 1 & 1 & 1 & $.2$\tabularnewline
\cline{4-7} 
\end{tabular}$\quad$%
\begin{tabular}{ccc|ccc|c|}
\hline 
\multicolumn{1}{|c}{$\alpha$} & $\beta$ & $\gamma$ & $A$ & $B$ & $C$ & $\Pr$\tabularnewline
\hline 
\hline 
\multicolumn{1}{|c}{2} & 3 & 2 & 0 & 0 & 0 & $.2$\tabularnewline
\cline{1-3} 
 &  &  & 0 & 0 & 1 & $.1$\tabularnewline
 &  &  & 0 & 1 & 0 & $.2$\tabularnewline
 &  &  & 0 & 1 & 1 & $.1$\tabularnewline
 &  &  & 1 & 0 & 0 & $.3$\tabularnewline
 &  &  & 1 & 0 & 1 & $.1$\tabularnewline
 &  &  & 1 & 1 & 0 & $0$\tabularnewline
 &  &  & 1 & 1 & 1 & $0$\tabularnewline
\cline{4-7} 
\end{tabular}\end{center}

\

\protect{\noindent}One can check that marginal selectivity holds
for all 1-marginals: thus, irrespective of other factor points,

\

\begin{center}%
\begin{tabular}{|c|c|c|}
\hline 
$\alpha$ & $A$ & $\Pr$\tabularnewline
\hline 
\hline 
1 & 0 & $.5$\tabularnewline
\cline{1-1} 
\multicolumn{1}{c|}{} & 1 & $.5$\tabularnewline
\cline{2-3} 
\end{tabular}$\quad$%
\begin{tabular}{|c|c|c|}
\hline 
$\alpha$ & $A$ & $\Pr$\tabularnewline
\hline 
\hline 
2 & 0 & $.6$\tabularnewline
\cline{1-1} 
\multicolumn{1}{c|}{} & 1 & $.4$\tabularnewline
\cline{2-3} 
\end{tabular}

\

\begin{tabular}{|c|c|c|}
\hline 
$\beta$ & $B$ & $\Pr$\tabularnewline
\hline 
\hline 
1 & 0 & $.6$\tabularnewline
\cline{1-1} 
\multicolumn{1}{c|}{} & 1 & $.4$\tabularnewline
\cline{2-3} 
\end{tabular}$\quad$%
\begin{tabular}{|c|c|c|}
\hline 
$\beta$ & $B$ & $\Pr$\tabularnewline
\hline 
\hline 
2 & 0 & $.5$\tabularnewline
\cline{1-1} 
\multicolumn{1}{c|}{} & 1 & $.5$\tabularnewline
\cline{2-3} 
\end{tabular}$\quad$%
\begin{tabular}{|c|c|c|}
\hline 
$\beta$ & $B$ & $\Pr$\tabularnewline
\hline 
\hline 
3 & 0 & $.7$\tabularnewline
\cline{1-1} 
\multicolumn{1}{c|}{} & 1 & $.3$\tabularnewline
\cline{2-3} 
\end{tabular}

\

\begin{tabular}{|c|c|c|}
\hline 
$\gamma$ & $A$ & $\Pr$\tabularnewline
\hline 
\hline 
1 & 0 & $.5$\tabularnewline
\cline{1-1} 
\multicolumn{1}{c|}{} & 1 & $.5$\tabularnewline
\cline{2-3} 
\end{tabular}$\quad$%
\begin{tabular}{|c|c|c|}
\hline 
$\gamma$ & $A$ & $\Pr$\tabularnewline
\hline 
\hline 
2 & 0 & $.7$\tabularnewline
\cline{1-1} 
\multicolumn{1}{c|}{} & 1 & $.3$\tabularnewline
\cline{2-3} 
\end{tabular}$\quad$%
\begin{tabular}{|c|c|c|}
\hline 
$\gamma$ & $A$ & $\Pr$\tabularnewline
\hline 
\hline 
3 & 0 & $.4$\tabularnewline
\cline{1-1} 
\multicolumn{1}{c|}{} & 1 & $.6$\tabularnewline
\cline{2-3} 
\end{tabular}$\quad$%
\begin{tabular}{|c|c|c|}
\hline 
$\gamma$ & $A$ & $\Pr$\tabularnewline
\hline 
\hline 
4 & 0 & $.9$\tabularnewline
\cline{1-1} 
\multicolumn{1}{c|}{} & 1 & $.1$\tabularnewline
\cline{2-3} 
\end{tabular}\end{center}

\

\protect{\noindent}One can also check that irrespective of the factor
point of $\gamma$, the 2-marginal $\left(A,B\right)$ only depends
on $\alpha$ and $\beta$:

\

\begin{center}%
\begin{tabular}{cc|cc|c|}
\hline 
\multicolumn{1}{|c}{$\alpha$} & $\beta$ & $A$ & $B$ & $\Pr$\tabularnewline
\hline 
\hline 
\multicolumn{1}{|c}{1} & 2 & 0 & 0 & $.3$\tabularnewline
\cline{1-2} 
 &  & 0 & 1 & $.2$\tabularnewline
 &  & 1 & 0 & $.2$\tabularnewline
 &  & 1 & 1 & $.3$\tabularnewline
\cline{3-5} 
\end{tabular}$\quad$%
\begin{tabular}{cc|cc|c|}
\hline 
\multicolumn{1}{|c}{$\alpha$} & $\beta$ & $A$ & $B$ & $\Pr$\tabularnewline
\hline 
\hline 
\multicolumn{1}{|c}{2} & 1 & 0 & 0 & $.3$\tabularnewline
\cline{1-2} 
 &  & 0 & 1 & $.3$\tabularnewline
 &  & 1 & 0 & $.3$\tabularnewline
 &  & 1 & 1 & $.1$\tabularnewline
\cline{3-5} 
\end{tabular}$\quad$%
\begin{tabular}{cc|cc|c|}
\hline 
\multicolumn{1}{|c}{$\alpha$} & $\beta$ & $A$ & $B$ & $\Pr$\tabularnewline
\hline 
\hline 
\multicolumn{1}{|c}{1} & 3 & 0 & 0 & $.5$\tabularnewline
\cline{1-2} 
 &  & 0 & 1 & $0$\tabularnewline
 &  & 1 & 0 & $.2$\tabularnewline
 &  & 1 & 1 & $.3$\tabularnewline
\cline{3-5} 
\end{tabular}$\quad$%
\begin{tabular}{cc|cc|c|}
\hline 
\multicolumn{1}{|c}{$\alpha$} & $\beta$ & $A$ & $B$ & $\Pr$\tabularnewline
\hline 
\hline 
\multicolumn{1}{|c}{2} & 3 & 0 & 0 & $.3$\tabularnewline
\cline{1-2} 
 &  & 0 & 1 & $.3$\tabularnewline
 &  & 1 & 0 & $.4$\tabularnewline
 &  & 1 & 1 & $0$\tabularnewline
\cline{3-5} 
\end{tabular}\end{center}

\

\protect{\noindent}Marginal selectivity, however, is violated for
the 2-marginal $\left(A,C\right)$: if the factor point of $\beta$
is $2^{\beta}$,

\

\begin{center}%
\begin{tabular}{cc|cc|c|}
\hline 
\multicolumn{1}{|c}{$\alpha$} & $\gamma$ & $A$ & $C$ & $\Pr$\tabularnewline
\hline 
\hline 
\multicolumn{1}{|c}{1} & 1 & 0 & 0 & $.3$\tabularnewline
\cline{1-2} 
 &  & 1 & 0 & $.2$\tabularnewline
 &  & 0 & 1 & $.2$\tabularnewline
 &  & 1 & 1 & $.3$\tabularnewline
\cline{3-5} 
\end{tabular}\end{center}

\

\protect{\noindent}but at $3^{\beta}$,

\

\begin{center}%
\begin{tabular}{cc|cc|c|}
\hline 
\multicolumn{1}{|c}{$\alpha$} & $\gamma$ & $A$ & $C$ & $\Pr$\tabularnewline
\hline 
\hline 
\multicolumn{1}{|c}{1} & 1 & 0 & 0 & $.4$\tabularnewline
\cline{1-2} 
 &  & 1 & 0 & $.1$\tabularnewline
 &  & 0 & 1 & $.1$\tabularnewline
 &  & 1 & 1 & $.4$\tabularnewline
\cline{3-5} 
\end{tabular}\end{center}

\

\protect{\noindent}This means that the diagram of selective influences
$(A,B,C)\looparrowleft(\alpha,\beta,\gamma)$ is ruled out.\qed
\end{example}
As pointed out in Section \ref{sec:Introduction}, the marginal selectivity
property alone is too weak to define selective influences. The example
below demonstrates that the property of marginal selectivity does
not allow one to treat each of the random variables as preserving
its identity in different combinations of {}``its'' factor with
other factors.
\begin{example}
\label{exa:Identity}Let $\alpha=\left\{ 1^{\alpha},2^{\alpha}\right\} $,
$\beta=\left\{ 1^{\beta},2^{\beta}\right\} $, and the set of allowable
treatments $T$ consist of all four possible combinations of the factor
points. Let $A$ and $B$ be be Bernoulli variables distributed as
shown:

\

\begin{center}%
\begin{tabular}{cc|cc|c|}
\hline 
\multicolumn{1}{|c}{$\alpha$} & $\beta$ & $A$ & $B$ & $\Pr$\tabularnewline
\hline 
\hline 
\multicolumn{1}{|c}{1} & 1 & 0 & 0 & $.1$\tabularnewline
\cline{1-2} 
 &  & 0 & 1 & $0$\tabularnewline
 &  & 1 & 0 & $0$\tabularnewline
 &  & 1 & 1 & $.9$\tabularnewline
\cline{3-5} 
\end{tabular}$\quad$%
\begin{tabular}{cc|cc|c|}
\hline 
\multicolumn{1}{|c}{$\alpha$} & $\beta$ & $A$ & $B$ & $\Pr$\tabularnewline
\hline 
\hline 
\multicolumn{1}{|c}{1} & 2 & 0 & 0 & $.09$\tabularnewline
\cline{1-2} 
 &  & 0 & 1 & $.01$\tabularnewline
 &  & 1 & 0 & $.81$\tabularnewline
 &  & 1 & 1 & $.09$\tabularnewline
\cline{3-5} 
\end{tabular}

\

\begin{tabular}{cc|cc|c|}
\hline 
\multicolumn{1}{|c}{$\alpha$} & $\beta$ & $A$ & $B$ & $\Pr$\tabularnewline
\hline 
\hline 
\multicolumn{1}{|c}{2} & 1 & 0 & 0 & $0$\tabularnewline
\cline{1-2} 
 &  & 0 & 1 & $.9$\tabularnewline
 &  & 1 & 0 & $.1$\tabularnewline
 &  & 1 & 1 & $0$\tabularnewline
\cline{3-5} 
\end{tabular}$\quad$%
\begin{tabular}{cc|cc|c|}
\hline 
\multicolumn{1}{|c}{$\alpha$} & $\beta$ & $A$ & $B$ & $\Pr$\tabularnewline
\hline 
\hline 
\multicolumn{1}{|c}{2} & 2 & 0 & 0 & $0$\tabularnewline
\cline{1-2} 
 &  & 0 & 1 & $.9$\tabularnewline
 &  & 1 & 0 & $.1$\tabularnewline
 &  & 1 & 1 & $0$\tabularnewline
\cline{3-5} 
\end{tabular}\end{center}

\

\protect{\noindent}Marginal selectivity is satisfied: $\Pr\left[A\left(1^{\alpha},\cdot\right)=0\right]=0.1$
and $\Pr\left[A\left(2^{\alpha},\cdot\right)=0\right]=0.9$ irrespective
of whether the placeholder is replaced with $1^{\beta}$ or $2^{\beta}$;
and analogously for $B$. If we assume, however, that this allows
us to write $A\left(1^{\alpha}\right),$ $A\left(2^{\alpha}\right)$,
$B\left(1^{\beta}\right)$, $B\left(2^{\beta}\right)$ instead of
$A\left(1^{\alpha},1^{\beta}\right)$, $A\left(1^{\alpha},2^{\beta}\right)$,
etc., we will run into a contradiction. From the tables for $\phi=\left\{ 1^{\alpha},1^{\beta}\right\} $,
$\left\{ 2^{\alpha},1^{\beta}\right\} $, and $\left\{ 2^{\alpha},2^{\beta}\right\} $,
we can successively conclude $A\left(1^{\alpha}\right)=B\left(1^{\beta}\right)$,
$A\left(2^{\alpha}\right)=1-B\left(1^{\beta}\right)$, and $A\left(2^{\alpha}\right)=1-B\left(2^{\beta}\right)$.
But then $A\left(1^{\alpha}\right)=B\left(2^{\beta}\right)$, which
contradicts the table for $\phi=\left\{ 1^{\alpha},2^{\beta}\right\} $,
where $A\left(1^{\alpha}\right)$ and $B\left(2^{\beta}\right)$ are
stochastically independent and nonsingular. This contradiction proves
that the diagram of selective influences $(A,B)\looparrowleft(\alpha,\beta)$
cannot be inferred from the compliance with marginal selectivity.\qed
\end{example}

\subsubsection{\label{sub:Invariance}Invariance under factor-point-specific transformations}

Let $(A_{1},\ldots,A_{n})\looparrowleft(\alpha_{1},\ldots,\alpha_{n})$
and 
\[
H=\left(H_{x_{1}^{\alpha_{1}}},\ldots,H_{x_{k_{1}}^{\alpha_{i}}},\ldots,H_{x_{1}^{\alpha_{n}}},\ldots,H_{x_{k_{n}}^{\alpha_{n}}}\right)
\]
be the JDC-vector for $(A_{1},\ldots,A_{n})(\phi)$. Let $F\left(H\right)$
be any function that applies to $H$ componentwise and produces a
corresponding vector of random variables
\[
F\left(H\right)=\left(\begin{array}{c}
F\left(x_{1}^{\alpha_{1}},H_{x_{1}^{\alpha_{1}}}\right),\ldots,F\left(x_{k_{1}}^{\alpha_{i}},H_{x_{k_{1}}^{\alpha_{i}}}\right),\\
\ldots,\\
F\left(x_{1}^{\alpha_{n}},H_{x_{1}^{\alpha_{n}}}\right),\ldots,F\left(x_{k_{n}}^{\alpha_{n}},H_{x_{k_{n}}^{\alpha_{n}}}\right)
\end{array}\right),
\]
where we denote by $F\left(x^{\alpha},\cdot\right)$ the application
of $F$ to the component labeled by $x^{\alpha}$. Clearly, $F\left(H\right)$
possesses a joint distribution and contains one component for each
factor point. If we now define a vector of random variables $B\left(\phi\right)$
for every treatment $\phi\in T$ as
\[
(B_{1},\ldots,B_{n})\left(\phi\right)=\left(F\left(\phi_{\left\{ \alpha_{1}\right\} },A_{1}\right),\ldots,F\left(\phi_{\left\{ \alpha_{n}\right\} },A_{n}\right)\right)\left(\phi\right),
\]
then
\[
(B_{1},\ldots,B_{n})\left(\phi\right)\sim\left(F\left(\phi_{\left\{ \alpha_{1}\right\} },A_{1}\right),\ldots,F\left(\phi_{\left\{ \alpha_{n}\right\} },A_{n}\right)\right)\left(\phi\right),
\]
and it follows from JDC that $(B_{1},\ldots,B_{n})\looparrowleft(\alpha_{1},\ldots,\alpha_{n})$.%
\footnote{Since it is possible that $F\left(x^{\alpha},H_{x^{\alpha}}\right)$
and $F\left(y^{\alpha},H_{y^{\alpha}}\right)$, with $x^{\alpha}\not=y^{\alpha}$,
have different sets of possible values, strictly speaking, one may
need to redefine the functions to ensure that the sets of possible
values for $B\left(\phi\right)$ is the same for different $\phi$.
This is, however, not essential (see footnote \ref{fn:Invariance}).%
} A function $F\left(x^{\alpha_{i}},\cdot\right)$ can be referred
to as a \emph{factor-point-specific transformation} of the random
variable $A_{i}$, because the random variable is generally transformed
differently for different points of the factor assumed to selectively
influence it. We can formulate the property in question by saying
that a diagram of selective influences is invariant under all factor-point-specific
transformations of the random variables. Note that this includes as
a special case transformations which are not factor-point-specific,
with 
\[
F\left(x_{1}^{\alpha_{i}},\cdot\right)\equiv\ldots\equiv F\left(x_{k_{i}}^{\alpha_{i}},\cdot\right)\equiv F\left(\alpha_{i},\cdot\right).
\]

\begin{example}
\label{exa:renaming}Let the set-up be the same as in Example \ref{exa:Identity},
except for the distributions of $(A,B)$ at the four treatments: we
now assume that these distributions are such that $(A,B)\looparrowleft(\alpha,\beta)$.
The tables below show all factor-point-specific transformations $A\rightarrow A^{*}$
and $B\rightarrow B^{*}$ at the four treatments, provided that the
sets of possible values of $A^{*}$ and $B^{*}$ are respectively,
$\left\{ \star,\bullet\right\} $ and $\left\{ \triangleright,\circ\right\} $,
and that at the treatment $\left\{ 1^{\alpha},1^{\beta}\right\} $
the value 0 of $A$ is mapped into $\star$ and the value 0 of $B$
is mapped into $\triangleright$. 

\

\begin{center}%
\begin{tabular}{|c|c|c|c|}
\hline 
$\alpha$ & $\beta$ & $A\rightarrow A^{*}$ & $B\rightarrow B^{*}$\tabularnewline
\hline 
\hline 
1 & 1 & $\begin{array}{c}
0\rightarrow\star\\
1\rightarrow\bullet
\end{array}$ & $\begin{array}{c}
0\rightarrow\triangleright\\
1\rightarrow\circ
\end{array}$\tabularnewline
\hline 
1 & 2 & $\begin{array}{c}
0\rightarrow\star\\
1\rightarrow\bullet
\end{array}$ & $\begin{array}{c}
0\rightarrow\circ\\
1\rightarrow\triangleright
\end{array}$\tabularnewline
\hline 
2 & 1 & $\begin{array}{c}
0\rightarrow\bullet\\
1\rightarrow\star
\end{array}$ & $\begin{array}{c}
0\rightarrow\triangleright\\
1\rightarrow\circ
\end{array}$\tabularnewline
\hline 
2 & 2 & $\begin{array}{c}
0\rightarrow\bullet\\
1\rightarrow\star
\end{array}$ & $\begin{array}{c}
0\rightarrow\circ\\
1\rightarrow\triangleright
\end{array}$\tabularnewline
\hline 
\end{tabular}$\quad$%
\begin{tabular}{|c|c|c|c|}
\hline 
$\alpha$ & $\beta$ & $A\rightarrow A^{*}$ & $B\rightarrow B^{*}$\tabularnewline
\hline 
\hline 
1 & 1 & $\begin{array}{c}
0\rightarrow\star\\
1\rightarrow\bullet
\end{array}$ & $\begin{array}{c}
0\rightarrow\triangleright\\
1\rightarrow\circ
\end{array}$\tabularnewline
\hline 
1 & 2 & $\begin{array}{c}
0\rightarrow\star\\
1\rightarrow\bullet
\end{array}$ & $\begin{array}{c}
0\rightarrow\triangleright\\
1\rightarrow\circ
\end{array}$\tabularnewline
\hline 
2 & 1 & $\begin{array}{c}
0\rightarrow\bullet\\
1\rightarrow\star
\end{array}$ & $\begin{array}{c}
0\rightarrow\triangleright\\
1\rightarrow\circ
\end{array}$\tabularnewline
\hline 
2 & 2 & $\begin{array}{c}
0\rightarrow\bullet\\
1\rightarrow\star
\end{array}$ & $\begin{array}{c}
0\rightarrow\triangleright\\
1\rightarrow\circ
\end{array}$\tabularnewline
\hline 
\end{tabular}\bigskip{}

\begin{tabular}{|c|c|c|c|}
\hline 
$\alpha$ & $\beta$ & $A\rightarrow A^{*}$ & $B\rightarrow B^{*}$\tabularnewline
\hline 
\hline 
1 & 1 & $\begin{array}{c}
0\rightarrow\star\\
1\rightarrow\bullet
\end{array}$ & $\begin{array}{c}
0\rightarrow\triangleright\\
1\rightarrow\circ
\end{array}$\tabularnewline
\hline 
1 & 2 & $\begin{array}{c}
0\rightarrow\star\\
1\rightarrow\bullet
\end{array}$ & $\begin{array}{c}
0\rightarrow\circ\\
1\rightarrow\triangleright
\end{array}$\tabularnewline
\hline 
2 & 1 & $\begin{array}{c}
0\rightarrow\star\\
1\rightarrow\bullet
\end{array}$ & $\begin{array}{c}
0\rightarrow\triangleright\\
1\rightarrow\circ
\end{array}$\tabularnewline
\hline 
2 & 2 & $\begin{array}{c}
0\rightarrow\star\\
1\rightarrow\bullet
\end{array}$ & $\begin{array}{c}
0\rightarrow\circ\\
1\rightarrow\triangleright
\end{array}$\tabularnewline
\hline 
\end{tabular}$\quad$%
\begin{tabular}{|c|c|c|c|}
\hline 
$\alpha$ & $\beta$ & $A\rightarrow A^{*}$ & $B\rightarrow B^{*}$\tabularnewline
\hline 
\hline 
1 & 1 & $\begin{array}{c}
0\rightarrow\star\\
1\rightarrow\bullet
\end{array}$ & $\begin{array}{c}
0\rightarrow\triangleright\\
1\rightarrow\circ
\end{array}$\tabularnewline
\hline 
1 & 2 & $\begin{array}{c}
0\rightarrow\star\\
1\rightarrow\bullet
\end{array}$ & $\begin{array}{c}
0\rightarrow\triangleright\\
1\rightarrow\circ
\end{array}$\tabularnewline
\hline 
2 & 1 & $\begin{array}{c}
0\rightarrow\star\\
1\rightarrow\bullet
\end{array}$ & $\begin{array}{c}
0\rightarrow\triangleright\\
1\rightarrow\circ
\end{array}$\tabularnewline
\hline 
2 & 2 & $\begin{array}{c}
0\rightarrow\star\\
1\rightarrow\bullet
\end{array}$ & $\begin{array}{c}
0\rightarrow\triangleright\\
1\rightarrow\circ
\end{array}$\tabularnewline
\hline 
\end{tabular}\end{center}

\

\protect{\noindent}The possible transformations are restricted to
these four because we adhere to our convention that $A$ has the same
set of values at all treatments, and the same is true for $B$. This
convention, however, is not essential, and nothing else in the theory
prevents one from thinking of $A$ at different treatments as arbitrarily
different random variables. With this {}``relaxed'' approach, the
following table gives an example of a factor-point-specific transformation: 

\

\begin{center}%
\begin{tabular}{|c|c|c|c|}
\hline 
$\alpha$ & $\beta$ & $A\rightarrow A^{*}$ & $B\rightarrow B^{*}$\tabularnewline
\hline 
\hline 
1 & 1 & $\begin{array}{c}
0\rightarrow0\\
1\rightarrow1
\end{array}$ & $\begin{array}{c}
0\rightarrow0\\
1\rightarrow1
\end{array}$\tabularnewline
\hline 
1 & 2 & $\begin{array}{c}
0\rightarrow0\\
1\rightarrow1
\end{array}$ & $\begin{array}{c}
0\rightarrow-2\\
1\rightarrow3
\end{array}$\tabularnewline
\hline 
2 & 1 & $\begin{array}{c}
0\rightarrow10\\
1\rightarrow-20
\end{array}$ & $\begin{array}{c}
0\rightarrow0\\
1\rightarrow1
\end{array}$\tabularnewline
\hline 
2 & 2 & $\begin{array}{c}
0\rightarrow10\\
1\rightarrow-20
\end{array}$ & $\begin{array}{c}
0\rightarrow-2\\
1\rightarrow3
\end{array}$\tabularnewline
\hline 
\end{tabular}\end{center}

If this is considered undesirable, the variables $\left(A^{*},B^{*}\right)$
can be redefined to have $\left\{ -20,0,1,10\right\} $ and $\left\{ -2,0,1,3\right\} $
and the respective sets of their possible values, assigning zero probabilities
to the values that cannot be attained at a given factor point.\qed

This property is of critical importance for construction and use of
tests for selective influences, as defined in the next section. A
test, generally, lacks the invariance property just formulated: e.g.,
if the transformation consists in grouping of the original values
of random variables, different groupings may result in different outcomes
of certain tests, fail or pass. Such a test then can be profitably
applied to various factor-point-specific transformations of an original
set of random variables, creating thereby in place of a single test
a multitude of tests with potentially different outcomes (a single
negative outcome ruling out the hypothesis of selective influences). 
\end{example}

\subsection{\label{sub:General-principles}General principles for constructing
tests for selective influences}

\subsubsection{\label{sub:Population-level-tests}Population level tests}

\label{sub:General logic}Given a set of factors $\left\{ \alpha_{1},\ldots,\alpha_{n}\right\} $,
a vector of random variables depending on treatments, $(A_{1},\ldots,A_{n})(\phi)$,
and the hypothesis $(A_{1},\ldots,A_{n})\looparrowleft(\alpha_{1},\ldots,\alpha_{n}$),
a \emph{test} for this hypothesis is a statement $\mathfrak{S}$ relating
to each other $(A_{1},\ldots,A_{n})(\phi)$ for different treatments
$\phi\in T$ which (a) holds true if $(A_{1},\ldots,A_{n})\looparrowleft(\alpha_{1},\ldots,\alpha_{n})$,
and (b) does not always hold true if this hypothesis is false. A test
for a diagram of selective influences therefore is a necessary condition:
if the variables $\left\{ (A_{1},\ldots,A_{n})(\phi):\phi\in T\right\} $
fail it (i.e., if $\mathfrak{S}$ is false for this set of random
variables), we know that the hypothesis $(A_{1},\ldots,A_{n})\looparrowleft(\alpha_{1},\ldots,\alpha_{n})$
is false. If the statement $\mathfrak{S}$ is always false when $(A_{1},\ldots,A_{n})\not\looparrowleft(\alpha_{1},\ldots,\alpha_{n})$,
the test becomes a \emph{criterion} for selective influences. A test
or criterion can be restricted to special classes of random variables
(e.g., random variables with finite numbers of values, or multivariate
normally distributed at every treatment) and/or factor sets (e.g.,
$2\times2$ experimental designs). 

The JDC provides a general logic for constructing such tests: we ask
whether the hypothetical JDC-vector $H=\left(H_{x_{1}^{\alpha_{1}}},\ldots,H_{x_{k_{1}}^{\alpha_{i}}},\ldots,H_{x_{1}^{\alpha_{n}}},\ldots,H_{x_{k_{n}}^{\alpha_{n}}}\right)$,
containing one variable for each factor point of each factor, can
be assigned a joint distribution such that its marginals corresponding
to the subsets of factor points that form treatments $\phi\in T$
are distributed as $(A_{1},\ldots,A_{n})(\phi)$. Put more succinctly:
is there a joint distribution of $\left(H_{x_{1}^{\alpha_{1}}},\ldots,H_{x_{k_{1}}^{\alpha_{i}}},\ldots,H_{x_{1}^{\alpha_{n}}},\ldots,H_{x_{k_{n}}^{\alpha_{n}}}\right)$
with given marginal distributions of the vectors 
\[
H_{\phi}=\left(H_{\phi_{\left\{ \alpha_{1}\right\} }},\ldots,H_{\phi_{\left\{ \alpha_{n}\right\} }}\right)
\]
for all $\phi\in T$?%
\footnote{\label{fn:Surprisingly,-at-least}Surprisingly, at least for the authors,
a slightly less general version of the same problem (the existence
of a joint distributions compatible with observable marginals) plays
a prominent role in quantum mechanics, in dealing with the quantum
entanglement problem (Fine, 1982a-b). We are grateful to Jerome Busemeyer
for bringing this fact to our attention. The parallels with quantum
mechanisms will be discussed in a separate publication. %
} 

Thus, in a study of random variables $(A,B)$ in a $2\times2$ factorial
design, with $\alpha=\left\{ 1^{\alpha},2^{\alpha}\right\} $, $\beta=\left\{ 1^{\beta},2^{\beta}\right\} $,
and $T$ containing all four logically possible treatments, we consider
a hypothetical JDC-vector $\left(H_{1^{\alpha}},H_{2^{\alpha}},H_{1^{\beta}},H_{2^{\beta}}\right)$
of which we know the four 2-marginal distributions corresponding to
treatments:
\[
\begin{array}{c}
H_{1^{\alpha}1^{\beta}}=(H_{1^{\alpha}},H_{1^{\beta}})\sim(A,B)\left(1^{\alpha},1^{\beta}\right),\\
\\
H_{1^{\alpha}2^{\beta}}=(H_{1^{\alpha}},H_{2^{\beta}})\sim(A,B)\left(1^{\alpha},2^{\beta}\right),\\
\\
\text{\emph{etc.}}
\end{array}
\]
Of course, we also know the lower-level marginals, in this case the
marginal distributions of $H_{1^{\alpha}}$, $H_{2^{\alpha}}$, $H_{1^{\beta}}$,
and $H_{2^{\beta}}$, but they need not be considered separately as
they are determined by the higher-order marginals. The question one
poses within the logic of JDC is: can one assign probability densities
to different values of $H=\left(H_{1^{\alpha}},H_{2^{\alpha}},H_{1^{\beta}},H_{2^{\beta}}\right)$
so that the computed marginal distributions of $(H_{1^{\alpha}},H_{1^{\beta}}$),
$(H_{1^{\alpha}},H_{2^{\beta}})$, etc., coincide with the known ones?

If the vector $A=(A_{1},\ldots,A_{n})$ has a finite number of possible
values (we may state this without mentioning $\phi$ because, by our
convention, the set of values does not depend on $\phi$), then so
does the vector $H=\left(H_{x_{1}^{\alpha_{1}}},\ldots,H_{x_{k_{1}}^{\alpha_{i}}},\ldots,H_{x_{1}^{\alpha_{n}}},\ldots,H_{x_{k_{n}}^{\alpha_{n}}}\right)$,
and the logic of JDC is directly implemented in the \emph{Linear Feasibility
Test} introduced in the next section. When the set of values for $A$
is infinite or too large to be handled by the Linear Feasibility Test,
one may have to use an indirect approach: computing from the distribution
of each $H_{\phi}$ certain functionals%
\footnote{A functional $g\left(X\right)$ is a function mapping each random
variable $X$ from some set of random variables into, typically, a
real or complex number (more generally, an element of a certain {}``standard''
set). A typical example of a functional is the expected value $\mathrm{E}\left[X\right]$.%
} $g_{1}\left(H_{\phi}\right),\ldots,g_{m}\left(H_{\phi}\right)$ and
constructing a statement 
\[
\mathfrak{S}\left(g_{1}(H_{\phi}),\ldots,g_{m}(H_{\phi}):\phi\in T\right)
\]
relating to each other these functionals for all $\phi\in T$. The
statement should be chosen so that it holds true if $H$ possesses
a joint distribution, but may be (or, better still, always is) false
otherwise.

We illustrate this logic on a simple distance test of the variety
introduced in Kujala and Dzhafarov (2008). Assuming that all random
variables in $(A_{1},\ldots,A_{n})$ take their values in the set
of reals, for each pair of factor points $\left\{ x^{\alpha},y^{\beta}\right\} $
define
\[
Mx^{\alpha}y^{\beta}=\mathrm{E}\left[\left|H_{x^{\alpha}}-H_{y^{\beta}}\right|\right],
\]
where, for convenience, we write $Mx^{\alpha}y^{\beta}$ in place
of $M\left(x^{\alpha},y^{\beta}\right)$. It can be easily shown that
$M$ is a metric on the set $H$ if $H$ possesses a joint distribution
for its components. For each treatment $\phi$, define the functional
\[
g_{\alpha,\beta}\left(H_{\phi}\right)=M\phi_{\left\{ \alpha\right\} }\phi_{\left\{ \beta\right\} },
\]
whose value can be computed from the known distributions:
\begin{equation}
M\phi_{\left\{ \alpha\right\} }\phi_{\left\{ \beta\right\} }=\mathrm{E}\left[\left|A_{\left\{ \alpha\right\} }\left(\phi\right)-A_{\left\{ \beta\right\} }\left(\phi\right)\right|\right],\label{eq:M-metric}
\end{equation}
where $A_{\left\{ \alpha\right\} }(\phi)$ and $A_{\left\{ \beta\right\} }(\phi)$
are the random variables in $(A_{1},\ldots,A_{n})(\phi)$ which are
supposed to be selectively influenced by $\alpha$ and $\beta$, respectively.
Due to the marginal selectivity (which we assume to hold because otherwise
selective influences have already been ruled out), this quantity is
the same for all treatments $\phi$ which contain the same factor
points $x^{\alpha},y^{\beta}$ of factors $\alpha,\beta$. The statement
$\mathfrak{S}$ is then as follows: for any (not necessarily pairwise
distinct) treatments $\phi^{1},\ldots,\phi^{l}$$\in T$ and any factors
$\alpha^{1},\ldots,\alpha^{l}\in\Phi$ ($l\geq$3) such that
\begin{equation}
\alpha_{1}\not=\alpha_{2}\not=\ldots\not=\alpha_{l-1}\not=\alpha_{l}\not=\alpha_{1},\label{eq:pairwise}
\end{equation}
and

\begin{equation}
\phi_{\left\{ \alpha^{1}\right\} }^{1}=\phi_{\left\{ \alpha^{1}\right\} }^{2},\ldots,\quad\phi_{\left\{ \alpha^{l-1}\right\} }^{l-1}=\phi_{\left\{ \alpha^{l-1}\right\} }^{l},\phi{}_{\left\{ \alpha^{l}\right\} }^{l}=\phi_{\left\{ \alpha^{l}\right\} }^{1},\label{eq:dist_cond_intro}
\end{equation}
we should have 
\begin{equation}
\begin{array}{r}
g_{\alpha^{1},\alpha^{l}}\left(H_{\phi^{1}}\right)\leq g_{\alpha^{1},\alpha^{2}}\left(H_{\phi^{2}}\right)+\ldots+g_{\alpha^{l-1},\alpha^{l}}\left(H_{\phi^{l}}\right).\end{array}\label{eq:dist test intro}
\end{equation}
The truth of $\mathfrak{S}$ for $H$ with jointly distributed components
follows from the triangle inequality for $M$. The inequality may
very well be violated when the components of $H$ do not possess a
joint distribution (i.e., when the hypothesis of selective influences
is false).
\begin{example}
To apply this test to Example \ref{exa:Identity}, we make use of
the property that if $(A,B)\looparrowleft(\alpha,\beta)$ then $(A^{*},B^{*})\looparrowleft(\alpha,\beta)$
for any factor-point-specific transformations $(A^{*},B^{*})$ of
$(A,B)$. Let us put $B^{*}=B$ and
\[
A^{*}=\begin{cases}
A & \textnormal{if }\phi_{\left\{ \alpha\right\} }=1^{\alpha},\\
1-A & \textnormal{if }\phi_{\left\{ \alpha\right\} }=2^{\alpha}.
\end{cases}
\]
This yields the distributions

\

\begin{center}%
\begin{tabular}{cc|cc|c|}
\hline 
\multicolumn{1}{|c}{$\alpha$} & $\beta$ & $A^{*}$ & $B^{*}$ & $\Pr$\tabularnewline
\hline 
\hline 
\multicolumn{1}{|c}{1} & 1 & 0 & 0 & $.1$\tabularnewline
\cline{1-2} 
 &  & 0 & 1 & $0$\tabularnewline
 &  & 1 & 0 & $0$\tabularnewline
 &  & 1 & 1 & $.9$\tabularnewline
\cline{3-5} 
\end{tabular}$\quad$%
\begin{tabular}{cc|cc|c|}
\hline 
\multicolumn{1}{|c}{$\alpha$} & $\beta$ & $A^{*}$ & $B^{*}$ & $\Pr$\tabularnewline
\hline 
\hline 
\multicolumn{1}{|c}{1} & 2 & 0 & 0 & $.09$\tabularnewline
\cline{1-2} 
 &  & 0 & 1 & $.01$\tabularnewline
 &  & 1 & 0 & $.81$\tabularnewline
 &  & 1 & 1 & $.09$\tabularnewline
\cline{3-5} 
\end{tabular}

\

\begin{tabular}{cc|cc|c|}
\hline 
\multicolumn{1}{|c}{$\alpha$} & $\beta$ & $A^{*}$ & $B^{*}$ & $\Pr$\tabularnewline
\hline 
\hline 
\multicolumn{1}{|c}{2} & 1 & 1 & 0 & $0$\tabularnewline
\cline{1-2} 
 &  & 1 & 1 & $.9$\tabularnewline
 &  & 0 & 0 & $.1$\tabularnewline
 &  & 0 & 1 & $0$\tabularnewline
\cline{3-5} 
\end{tabular}$\quad$%
\begin{tabular}{cc|cc|c|}
\hline 
\multicolumn{1}{|c}{$\alpha$} & $\beta$ & $A^{*}$ & $B^{*}$ & $\Pr$\tabularnewline
\hline 
\hline 
\multicolumn{1}{|c}{2} & 2 & 1 & 0 & $0$\tabularnewline
\cline{1-2} 
 &  & 1 & 1 & $.9$\tabularnewline
 &  & 0 & 0 & $.1$\tabularnewline
 &  & 0 & 1 & $0$\tabularnewline
\cline{3-5} 
\end{tabular}\end{center}

\

\protect{\noindent}It is easy to check that
\[
\begin{array}{l}
M1^{\alpha}1^{\beta}=\mathrm{E}\left[\left|A\left(1^{\alpha},1^{\beta}\right)-B\left(1^{\alpha},1^{\beta}\right)\right|\right]=0,\\
\\
M1^{\alpha}2^{\beta}=\mathrm{E}\left[\left|A\left(1^{\alpha},2^{\beta}\right)-B\left(1^{\alpha},2^{\beta}\right)\right|\right]=0.82,\\
\\
M2^{\alpha}1^{\beta}=\mathrm{E}\left[\left|A\left(2^{\alpha},1^{\beta}\right)-B\left(2^{\alpha},1^{\beta}\right)\right|\right]=0,\\
\\
M2^{\alpha}2^{\beta}=\mathrm{E}\left[\left|A\left(2^{\alpha},2^{\beta}\right)-B\left(2^{\alpha},2^{\beta}\right)\right|\right]=0.
\end{array}
\]
Since
\[
0.82=M1^{\alpha}2^{\beta}>M1^{\alpha}1^{\beta}+M2^{\alpha}1^{\beta}+M2^{\alpha}2^{\beta}=0,
\]
the triangle inequality is violated, rejecting thereby the hypothesis
$(A^{*},B^{*})\looparrowleft(\alpha,\beta)$, hence also the hypothesis
$(A,B)\looparrowleft(\alpha,\beta)$.\qed 
\end{example}

\subsubsection{Sample-level tests\label{sub:Sample-level-tests}}

\label{sub:stats}Although this paper is not concerned with statistical
questions, it may be useful to outline the general logic of constructing
a sample-level test corresponding to a population-level one. Analytic
procedures and asymptotic approximations have to be different for
different tests, but if the population-level test can be computed
efficiently, the following Monte-Carlo procedure is always applicable. 
\begin{enumerate}
\item For each of the random variables $A_{1},\ldots,A_{n}$, if it has
more than a finite number of values (or has too many values, even
if finite), we discretize it in the conventional way, by forming successive
adjacent intervals and replacing each of them with its midpoint. Continue
to denote the discretized random variables $A_{1},\ldots,A_{n}$.
\item We now have sample proportions $\hat{\Pr}\left[\left(A_{1}=a_{1},\ldots,A_{n}=a_{n}\right)\left(x_{1}^{\alpha_{1}},\ldots,x_{n}^{\alpha_{n}}\right)\right]$,
where $a_{1},\ldots,a_{n}$ are possible values of the corresponding
random variables $A_{1},\ldots,A_{n}$.
\item For each treatment, we form a confidence region of possible probabilities
$\Pr\left[\left(A_{1}=a_{1},\ldots,A_{n}=a_{n}\right)\left(x_{1}^{\alpha_{1}}\ldots x_{n}^{\alpha_{n}}\right)\right]$
for a given set of estimates, at a given level of a familywise confidence
level for the Cartesian product of these confidence regions, with
an appropriately adopted convention on how this familywise confidence
is computed (glossing over a controversial issue). 
\item The hypothesis of selective influences is retained or rejected according
as the combined confidence region contains or does not contain a point
(a set of joint probabilities) which passes the population test in
question. (Gradualized versions of this procedure are possible, when
each point in the space of population-level probabilities is taken
with the weight proportional to its likelihood.)
\end{enumerate}
Instead of a confidence region of multivariate distributions based
on a discretization, one can also generate confidence regions of distributions
belonging to a specified class, say, multivariate normal ones. 

Resampling techniques is another obvious approach, although the results
will generally depend on one's often arbitrary choice of the resampling
procedure. One simple choice is the permutation test in which the
joint sample proportions $\hat{\Pr}\left[A_{1}=a_{1},\ldots,A_{n}=a_{n}\right]$
obtained at different treatments (and treated as probabilities) are
randomly assigned to the treatments $\phi$. If the initial, observed
assignment passes a test, while the proportion of the permuted assignments
which pass the test is sufficiently small, the hypothesis of selective
influences is considered supported.

\section{\label{sec:Linear-Feasibility-Test}Linear Feasibility Test}

In this section we assume that each random variable $A_{i}(\phi)$
in $(A_{1},\ldots,A_{n})(\phi)$ has a finite number $m_{i}$ of possible
values $a_{i1},\ldots,a_{im_{i}}$. It is arguably the most important
special case both because it is ubiquitous in psychological theories
and because in all other cases random variables can be discretized
into finite number of categories. We are interested in establishing
the truth or falsity of the diagram of selective influences (\ref{diag:bijective DSI}),
where each factor $\alpha_{i}$ in $(\alpha_{1},\ldots,\alpha_{n})$
contains $k_{i}$ factor points. The \emph{Linear Feasibility Test}
to be described is a direct application of JDC to this situation,%
\footnote{\label{fn:In-reference-to}In reference to footnote \ref{fn:Surprisingly,-at-least},
this test has been proposed in the context of dealing with multiple-particle
multiple-measurement quantum entanglement situations by Werner \&
Wolf (2001a, b) and Basoalto \& Percival (2003). %
} furnishing both a necessary and sufficient condition for the diagram
of selective influences $(A_{1},\ldots,A_{n})\looparrowleft(\alpha_{1},\ldots,\alpha_{n})$. 

In the hypothetical JDC-vector
\[
H=\left(H_{x_{1}^{\alpha_{1}}},\ldots,H_{x_{k_{1}}^{\alpha_{1}}},\ldots,H_{x_{1}^{\alpha_{n}}},\ldots,H_{x_{k_{n}}^{\alpha_{n}}}\right),
\]
since we assume that 
\[
H_{x_{j}^{\alpha_{i}}}\sim A_{i}\left(\phi\right)
\]
for any $x_{j}^{\alpha_{i}}$ and any treatment $\phi$ containing
$x_{j}^{\alpha_{i}}$, we know that the set of possible values for
the random variable $H_{x_{j}^{\alpha_{i}}}$ is $\left\{ a_{i1},\ldots,a_{im_{i}}\right\} $,
irrespective of $x_{j}$. Denote 
\begin{equation}
\begin{array}{r}
\begin{array}{l}
\Pr\left[\left(A_{1}=a_{1l_{1}},\ldots,A_{n}=a_{nl_{n}}\right)\left(x_{\lambda_{1}}^{\alpha_{1}},\ldots,x_{\lambda_{n}}^{\alpha_{n}}\right)\right]\\
\\
=P\left(\stackrel{\textnormal{for r.v.s}}{\overbrace{l_{1},\ldots,l_{n}}}\,;\,\stackrel{\textnormal{for factor points}}{\overbrace{\lambda_{1},\ldots,\lambda_{n}}}\right),
\end{array}\end{array}\label{eq:p's}
\end{equation}
where $l_{i}\in\left\{ 1,\ldots,m_{i}\right\} $ and $\lambda_{i}\in\left\{ 1,\ldots,k_{i}\right\} $
for $i=1,\ldots,n$ ({}``r.v.s'' abbreviates {}``random variables'').
Denote 
\begin{equation}
\begin{array}{l}
\Pr\left[\begin{array}{c}
H_{x_{1}^{\alpha_{1}}}=a_{1l_{11}},\ldots,H_{x_{k_{1}}^{\alpha_{1}}}=a_{1l_{1k_{1}}},\\
\ldots,\\
H_{x_{1}^{\alpha_{n}}}=a_{nl_{n1}},\ldots,H_{x_{k_{n}}^{\alpha_{n}}}=a_{nl_{nk_{n}}}
\end{array}\right]\\
\\
=Q\left(\stackrel{\textnormal{for }A_{1}}{\overbrace{l_{11},\ldots,l_{1k_{1}}}},\ldots,\stackrel{\textnormal{for }A_{n}}{\overbrace{l_{n1},\ldots,l_{nk_{n}}}}\right),
\end{array}\label{eq:q's}
\end{equation}
where $l_{ij}\in\left\{ 1,\ldots,m_{i}\right\} $ for $i=1,\ldots,n$.
This gives us $m_{1}^{k_{1}}\times\ldots\times m_{n}^{k_{n}}$ $Q$-probabilities.
A required joint distribution for the JDC-vector $H$ exists if and
only if these probabilities can be found subject to $m_{1}^{k_{1}}\times\ldots\times m_{n}^{k_{n}}$
nonnegativity constraints

\begin{equation}
Q\left(l_{11},\ldots,l_{1k_{1}},\ldots,l_{n1},\ldots,l_{nk_{n}}\right)\geq0,\label{eq:nonegativity}
\end{equation}
and (denoting by $n_{T}$ the number of treatments in $T$) $n_{T}\times m_{1}\times\ldots\times m_{n}$
linear equations
\begin{equation}
\begin{array}{r}
\sum Q\left(l_{11},\ldots,l_{1k_{1}},\ldots,l_{n1},\ldots,l_{nk_{n}}\right)\\
\\
=P\left(l_{1},\ldots,l_{n};\lambda_{1},\ldots,\lambda_{n}\right),
\end{array}\label{eq:linear equations}
\end{equation}
where the summation is across all possible values of the set 
\[
\left\{ l_{11},\ldots,l_{1k_{1}},\ldots,l_{n1},\ldots,l_{nk_{n}}\right\} -\left\{ l_{1\lambda_{1}},\ldots,l_{n\lambda_{n}}\right\} ,
\]
while
\[
l_{1\lambda_{1}}=l_{1},\ldots,l_{n\lambda_{n}}=l_{n}.
\]
Selective influences hold if and only if the system of these linear
equalities with the nonnegativity constraints is \emph{feasible} (i.e.,
has a solution). This is a typical \emph{linear programming} problem
(see, e.g., Webster, 1994, Ch. 4).%
\footnote{More precisely, this is a linear programming task in the standard
form and with a dummy objective function (e.g., a linear combination
with zero coefficients).%
} Many standard statistical and mathematical packages can handle this
problem. 

Note that the maximal value for $n_{T}$ is $n_{T}=k_{1}\times\ldots\times k_{n}$,
whence the maximal number of linear equations is $\left(m_{1}k_{1}\right)\times\ldots\times\left(m_{n}k_{n}\right)$.
Since $m_{i}k_{i}\leq m_{i}^{k_{i}}$ (assuming $m_{i},k_{i}\ge2)$,
with the equality only achieved at $k_{i}=m_{i}=2$, the system of
linear equations is always underdetermined. In fact, the system of
equations is underdetermined even if $k_{i}=m_{i}=2$ for all $i=1,\ldots,n$,
because of the obvious linear dependences among the equations. 
\begin{example}
Let $\alpha=\left\{ 1^{\alpha},2^{\alpha}\right\} $, $\beta=\left\{ 1^{\beta},2^{\beta}\right\} $,
and the set of allowable treatments $T$ consist of all four possible
combinations of the factor points. Let $A$ and $B$ be Bernoulli
variables distributed as shown:

\

\begin{center}%
\begin{tabular}{cc|cc|c|}
\hline 
\multicolumn{1}{|c}{$\alpha$} & $\beta$ & $A$ & $B$ & $\Pr$\tabularnewline
\hline 
\hline 
\multicolumn{1}{|c}{1} & 1 & 0 & 0 & $.140$\tabularnewline
\cline{1-2} 
 &  & 0 & 1 & $.360$\tabularnewline
 &  & 1 & 0 & $.360$\tabularnewline
 &  & 1 & 1 & $.140$\tabularnewline
\cline{3-5} 
\end{tabular}$\quad$%
\begin{tabular}{cc|cc|c|}
\hline 
\multicolumn{1}{|c}{$\alpha$} & $\beta$ & $A$ & $B$ & $\Pr$\tabularnewline
\hline 
\hline 
\multicolumn{1}{|c}{1} & 2 & 0 & 0 & $.198$\tabularnewline
\cline{1-2} 
 &  & 0 & 1 & $.302$\tabularnewline
 &  & 1 & 0 & $.302$\tabularnewline
 &  & 1 & 1 & $.198$\tabularnewline
\cline{3-5} 
\end{tabular}

\

\begin{tabular}{cc|cc|c|}
\hline 
\multicolumn{1}{|c}{$\alpha$} & $\beta$ & $A$ & $B$ & $\Pr$\tabularnewline
\hline 
\hline 
\multicolumn{1}{|c}{2} & 1 & 0 & 0 & $.189$\tabularnewline
\cline{1-2} 
 &  & 0 & 1 & $.311$\tabularnewline
 &  & 1 & 0 & $.311$\tabularnewline
 &  & 1 & 1 & $.189$\tabularnewline
\cline{3-5} 
\end{tabular}$\quad$%
\begin{tabular}{cc|cc|c|}
\hline 
\multicolumn{1}{|c}{$\alpha$} & $\beta$ & $A$ & $B$ & $\Pr$\tabularnewline
\hline 
\hline 
\multicolumn{1}{|c}{2} & 2 & 0 & 0 & $.460$\tabularnewline
\cline{1-2} 
 &  & 0 & 1 & $.040$\tabularnewline
 &  & 1 & 0 & $.040$\tabularnewline
 &  & 1 & 1 & $.460$\tabularnewline
\cline{3-5} 
\end{tabular}\end{center}

\

\protect{\noindent}Marginal selectivity here is satisfied trivially:
all marginal probabilities are equal 0.5, for all treatments. The
linear programing routine of Mathematica\texttrademark (using the
interior point algorithm) shows that the linear equations (\ref{eq:linear equations})
have nonnegative solutions corresponding to the JDC-vector

\

\protect{\noindent}%
\begin{tabular}{|cccc|c|}
\hline 
$H_{1^{\alpha}}$ & $H_{2^{\alpha}}$ & $H_{1^{\beta}}$ & $H_{2^{\beta}}$ & $\Pr$\tabularnewline
\hline 
\hline 
0 & 0 & 0 & 0 & $.02708610$\tabularnewline
0 & 0 & 0 & 1 & $.00239295$\tabularnewline
0 & 0 & 1 & 0 & $.16689300$\tabularnewline
0 & 0 & 1 & 1 & $.03358610$\tabularnewline
0 & 1 & 0 & 0 & $.00197965$\tabularnewline
0 & 1 & 0 & 1 & $.10854100$\tabularnewline
0 & 1 & 1 & 0 & $.00204128$\tabularnewline
0 & 1 & 1 & 1 & $.15748000$\tabularnewline
\hline 
\end{tabular}$\quad$%
\begin{tabular}{|cccc|c|}
\hline 
$H_{1^{\alpha}}$ & $H_{2^{\alpha}}$ & $H_{1^{\beta}}$ & $H_{2^{\beta}}$ & $\Pr$\tabularnewline
\hline 
\hline 
1 & 0 & 0 & 0 & $.15748000$\tabularnewline
1 & 0 & 0 & 1 & $.00204128$\tabularnewline
1 & 0 & 1 & 0 & $.10854100$\tabularnewline
1 & 0 & 1 & 1 & $.00197965$\tabularnewline
1 & 1 & 0 & 0 & $.03358610$\tabularnewline
1 & 1 & 0 & 1 & $.16689300$\tabularnewline
1 & 1 & 1 & 0 & $.00239295$\tabularnewline
1 & 1 & 1 & 1 & $.02708610$\tabularnewline
\hline 
\end{tabular}

\

\protect{\noindent}This proves that in this case we do have $(A,B)\looparrowleft(\alpha,\beta)$.\qed
\end{example}

\begin{example}
In the previous example, let us change the distributions of $(A,B)$
to the following:

\

\begin{center}%
\begin{tabular}{cc|cc|c|}
\hline 
\multicolumn{1}{|c}{$\alpha$} & $\beta$ & $A$ & $B$ & $\Pr$\tabularnewline
\hline 
\hline 
\multicolumn{1}{|c}{1} & 1 & 0 & 0 & $.450$\tabularnewline
\cline{1-2} 
 &  & 0 & 1 & $.050$\tabularnewline
 &  & 1 & 0 & $.050$\tabularnewline
 &  & 1 & 1 & $.450$\tabularnewline
\cline{3-5} 
\end{tabular}$\quad$%
\begin{tabular}{cc|cc|c|}
\hline 
\multicolumn{1}{|c}{$\alpha$} & $\beta$ & $A$ & $B$ & $\Pr$\tabularnewline
\hline 
\hline 
\multicolumn{1}{|c}{1} & 2 & 0 & 0 & $.105$\tabularnewline
\cline{1-2} 
 &  & 0 & 1 & $.395$\tabularnewline
 &  & 1 & 0 & $.395$\tabularnewline
 &  & 1 & 1 & $.105$\tabularnewline
\cline{3-5} 
\end{tabular}

\

\begin{tabular}{cc|cc|c|}
\hline 
\multicolumn{1}{|c}{$\alpha$} & $\beta$ & $A$ & $B$ & $\Pr$\tabularnewline
\hline 
\hline 
\multicolumn{1}{|c}{2} & 1 & 0 & 0 & $.170$\tabularnewline
\cline{1-2} 
 &  & 0 & 1 & $.330$\tabularnewline
 &  & 1 & 0 & $.330$\tabularnewline
 &  & 1 & 1 & $.170$\tabularnewline
\cline{3-5} 
\end{tabular}$\quad$%
\begin{tabular}{cc|cc|c|}
\hline 
\multicolumn{1}{|c}{$\alpha$} & $\beta$ & $A$ & $B$ & $\Pr$\tabularnewline
\hline 
\hline 
\multicolumn{1}{|c}{2} & 2 & 0 & 0 & $.110$\tabularnewline
\cline{1-2} 
 &  & 0 & 1 & $.390$\tabularnewline
 &  & 1 & 0 & $.390$\tabularnewline
 &  & 1 & 1 & $.110$\tabularnewline
\cline{3-5} 
\end{tabular}\end{center}

\

\protect{\noindent}Once again, marginal selectivity is satisfied
trivially, as all marginal probabilities are 0.5, for all treatments.
The linear programing routine of Mathematica\texttrademark, however,
shows that the linear equations (\ref{eq:linear equations}) have
no nonnegative solutions. This excludes the existence of a JDC-vector
for this situations, ruling out thereby the possibility of $(A,B)\looparrowleft(\alpha,\beta)$.\qed
\end{example}
Since the Linear Feasibility Test is both a necessary and sufficient
condition for selective influences, if it is passed for $(A_{1},\ldots,A_{n})(\phi)$,
it is guaranteed to be passed following any factor-point-specific
transformations of these random outputs. All such transformations
in the case of discrete random variables can be described as combinations
of renamings (factor-point specific ones) and augmentations (grouping
of some values together). In fact, a result of the Linear Feasibility
Test simply does not depend on the values of the random variables
involved, only their probabilities matter. Therefore a renaming, such
as in Example \ref{exa:renaming}, will not change anything in the
system of linear equations and inequalities (\ref{eq:nonegativity})-(\ref{eq:linear equations}).
An example of augmentation (or {}``coarsening'') will be redefining
$A$ and $B$, each having possible values $1,2,3,4$, into binary
variables
\[
A^{*}\left(\phi\right)=\begin{cases}
0 & \textnormal{if }A\left(\phi\right)=1,2,\\
1 & \textnormal{if }A\left(\phi\right)=3,4,
\end{cases}\quad B^{*}\left(\phi\right)=\begin{cases}
0 & \textnormal{if }B\left(\phi\right)=1,2,3,\\
1 & \textnormal{if }B\left(\phi\right)=4.
\end{cases}
\]
It is clear that any such an augmentation amounts to replacing some
of the equations in (\ref{eq:linear equations}) with their sums.
Therefore, if the original system has a solution, so will also the
system after such replacements. 

The same reasoning applies to one's redefining the factors by grouping
together some of the factor points: e.g., redefining $\alpha=\left\{ 1^{\alpha},2^{\alpha},3^{\alpha}\right\} $
into 
\[
\alpha^{*}=\left\{ \left\{ 1^{\alpha},2^{\alpha}\right\} ^{\alpha^{*}},\left\{ 3^{\alpha}\right\} ^{\alpha^{*}}\right\} =\left\{ 1^{\alpha^{*}},2^{\alpha^{*}}\right\} .
\]
This change will amount to replacing by their sum any two equations
whose right hand sides correspond to identical vectors $\left(l_{1},\ldots,l_{n};\lambda_{1},\ldots,\lambda_{n}\right)$
except for the factor point for $\alpha$ being 1 in one of them and
2 in another. 

Summarizing, the Linear Feasibility Test cannot reject selective influences
on a coarser level of representation (for random variables and/or
factors) and uphold it on a finer level (although the reverse, obviously,
can happen).

If the random variables involved have more than finite number of values
and/or the factors consist of more than finite number of factor points,
or if these numbers, though finite, are too large to handle the ensuing
linear programming problem, then the Linear Feasibility Test can still
be used after the values of the random variables and/or factors have
been appropriately grouped. The Linear Feasibility Test then becomes
only a necessary condition for selective influences, and its results
will generally be different for different (non-nested) groupings. 
\begin{example}
Consider the hypothesis $(A,B)\looparrowleft(\alpha,\beta)$ with
the factors having a finite number of factor points each, and $A$
and $B$ being response times. To use the Linear Feasibility Test,
one can transform the random variable $A$ as, say,
\[
A^{*}\left(\phi\right)=\begin{cases}
1 & \textnormal{if }A\left(\phi\right)\leq a_{1/4}\left(\phi\right),\\
2 & \textnormal{if }a_{1/4}\left(\phi\right)<A\left(\phi\right)\leq a_{1/2}\left(\phi\right),\\
3 & \textnormal{if }a_{1/2}\left(\phi\right)<A\left(\phi\right)\leq a_{3/4}\left(\phi\right),\\
4 & \textnormal{if }A\left(\phi\right)>a_{3/4}\left(\phi\right),
\end{cases}
\]
and transform $B$ as
\[
B^{*}\left(\phi\right)=\begin{cases}
1 & \textnormal{if }B\left(\phi\right)\leq b_{1/2}\left(\phi\right),\\
2 & \textnormal{if }B\left(\phi\right)>b_{1/2}\left(\phi\right),
\end{cases}
\]
where $a_{p}\left(\phi\right)$ and $b_{p}\left(\phi\right)$ designate
the $p$th quantiles of, respectively $A\left(\phi\right)$ and $B\left(\phi\right)$.
The initial hypothesis now is reformulated as $(A^{*},B^{*})\looparrowleft(\alpha,\beta)$,
with the understanding that if it is rejected then the initial hypothesis
will be rejected too (a necessary condition only). The Linear Feasibility
test will now be applied to distributions of the form

\

\begin{center}%
\begin{tabular}{cc|cc|c|}
\hline 
\multicolumn{1}{|c}{$\alpha$} & $\beta$ & $A$ & $B$ & $\Pr$\tabularnewline
\hline 
\hline 
\multicolumn{1}{|c}{$x$} & y & 1 & 1 & $p_{11}$\tabularnewline
\cline{1-2} 
 &  & 1 & 2 & $p_{12}$\tabularnewline
 &  & $\vdots$ & $\vdots$ & $\vdots$\tabularnewline
 &  & 4 & 1 & $p_{41}$\tabularnewline
 &  & 4 & 2 & $p_{42}$\tabularnewline
\cline{3-5} 
\end{tabular}\end{center}

\

\protect{\noindent}where the marginals for $A$ are constrained to
0.25 and the marginals for $B$ to 0.5, for all treatments $\left\{ x^{\alpha},y^{\beta}\right\} $,
yielding a trivial compliance with marginal selectivity. Note that
the test may very well uphold $(A^{*},B^{*})\looparrowleft(\alpha,\beta)$
even if marginal selectivity is violated for $(A,B)(\phi)$ (e.g.,
if the quantiles $a_{p}\left(x^{\alpha},y^{\beta}\right)$ change
as a function of $y^{\beta}$). \qed

Sample level problems do not seem to present a serious difficulty.
The general approach mentioned in Section \ref{sub:stats} is facilitated
by the following consideration. If a system of linear equations and
inequalities has an {}``interior'' solution (one for which all inequalities
are satisfied in the strict form, which in our case means that the
solution contains no zeros), then the solution is stable with respect
to sufficiently small perturbations of its coefficients. In our case,
this means that if an interior solution exists for population-level
values of $P\left(l_{1},\ldots,l_{n};\lambda_{1},\ldots,\lambda_{n}\right)$,
and if the sample estimates of the latter are sufficiently close to
the population values, then the system will also have a solution for
sample estimates. By the same token, if no solution exists for the
population-level values of $P\left(l_{1},\ldots,l_{n};\lambda_{1},\ldots,\lambda_{n}\right)$,
then no solution will be found for sufficiently close to them sample
estimates. The only unstable situation exists if solutions exists
on the hypothetical population level (i.e., the selectiveness of influences
is satisfied), but they are all non-interior (contain zeros).\end{example}
\begin{rem}
The question arises: how restrictive is the condition of selective
influences within the class of distributions satisfying marginal selectivity?
We do not know anything close to a complete answer to this question,
but simulations show that selectivity of influence is not overly restrictive
with respect to marginal selectivity. Thus, if $k_{i}=m_{i}=2$ for
$i=1,2$, and if we constrain all marginal probabilities to 0.5 and
pick $P\left(1,1;1,1\right),P\left(1,1;1,2\right),P\left(1,1;2,1\right),P\left(1,1;2,2\right)$
from four independent uniform distributions between 0 and 0.5, the
probability of {}``randomly'' obtaining selective influences is
about 0.67. If $k_{i}=m_{i}=2$ for $i=1,2,3$, and we constrain all
2-marginal probabilities to 0.25, the analogous probability is about
0.10. 
\end{rem}

\section{\label{sec:Distance-type-tests}Distance-type tests}

\subsection{General theory\label{sub:General-theory}}

First, we establish the general terminology related to distance-type
functions. Given a set $\mathcal{R}$, a function $d:\mathcal{R}\times\mathcal{R}\rightarrow\left[0,\infty\right]$
is a \emph{premetric} if $d\left(x,x\right)=0$. The inclusion of
the possibility $d\left(x,y\right)=\infty$ usually adds the qualifier
{}``extended'' (in this case, extended premetric), but we will omit
it for brevity. A premetric that satisfies the triangle inequality,
\[
d\left(x,z\right)\le d\left(x,y\right)+d\left(y,z\right),
\]
for any $x,y,z\in\mathcal{R}$, is a \emph{pseudo-quasi-metric} (\emph{p.q.-metric},
for short). A p.q.-metric which is symmetric,
\[
d\left(x,y\right)=d\left(y,x\right),
\]
for all $x,y\in\mathcal{R}$, is a \emph{pseudometric}. A p.q.-metric
such that
\[
x\neq y\Longrightarrow d\left(x,y\right)>0
\]
(equivalently, $d\left(x,y\right)=0$ if and only if $x=y$) is a
\emph{quasimetric}. A p.q.-metric which is simultaneously a quasimetric
and a pseudometric is a \emph{conventional} (symmetric) \emph{metric}.
The words {}``metric'' and {}``distance'' can be used interchangeably:
so one can speak of conventional (symmetric) distances, pseudodistances,
quasidistances, and p.q.-distances. %
\footnote{The terminology adopted in this paper is conventional but not universal.
In particular, the term {}``metric'' or {}``distance'' is sometimes
used to mean pseudometric. In the context of Finsler geometry and
the dissimilarity cumulation theory (Dzhafarov, 2010) the term {}``metric''
is used to designate quasimetric with an additional property of being
{}``symmetric in the small.'' %
} 

We are interested in the situation when $\mathcal{R}$ is a set of
jointly distributed random variables (discreet, continuous, or mixed),
with the intent to apply a distance-type function definable on such
an $\mathcal{R}$ to the JDC-vector $H$ of random variables for the
diagram of selective influences (\ref{diag:bijective DSI}). The random
variables $A(\phi)=(A_{1},\ldots,A_{n})(\phi)$, the factors $\Phi=\left\{ \alpha_{1},\ldots,\alpha_{n}\right\} $,
and the set of treatments $T$ are defined as above. The main property
we are concerned with is the triangle inequality, that is, it is typically
sufficient to know that the distance-type function we are dealing
with is a p.q.-metric.

The function (\ref{eq:M-metric}) considered in Section \ref{sub:General logic}
serves as an introductory example of a metric on which one can base
a test for selective influences. As a simple example of using a p.q.-metric
which is not a conventional metric (in fact, not even a pseudometric
or quasimetric), consider the following. Let the elements of $\mathcal{R}$
be binary random variables, with values $\left\{ 1,2\right\} $. Define,
for any $A_{1},\ldots,A_{p},B_{1},\ldots,B_{q}\in\mathcal{R}$,
\[
P^{\left(2\right)}\left[\left(A_{1},\ldots,A_{p}\right)\left(B_{1},\ldots,B_{q}\right)\right]=\Pr\left[\begin{array}{c}
A_{i}=1\textnormal{ for }i=1,\ldots,p,\\
B_{j}=2\textnormal{ for }j=1,\ldots,q
\end{array}\right].
\]
The parentheses may be dropped around singletons, in particular,
\[
\Pr\left[A=1,B=2\right]=P^{\left(2\right)}\left[\left(A\right)\left(B\right)\right]=P^{\left(2\right)}\left[AB\right].
\]
The latter is clearly a premetric: $P^{\left(2\right)}$ is nonnegative,
and $P^{\left(2\right)}\left[RR\right]=0$, for any $R\in\mathcal{R}$.
To prove the triangle inequality,
\[
P^{\left(2\right)}\left[R_{1}R_{2}\right]\leq P^{\left(2\right)}\left[RR_{2}\right]+P^{\left(2\right)}\left[R_{1}R\right],
\]
for any $R_{1},R_{2},R\in\mathcal{R}$, observe that
\[
P^{\left(2\right)}\left[R_{1}R_{2}\right]=P^{\left(2\right)}\left[\left(R_{1},R\right)R_{2}\right]+P^{\left(2\right)}\left[R_{1}\left(R_{2},R\right)\right],
\]
\[
P^{\left(2\right)}\left[RR_{2}\right]=P^{\left(2\right)}\left[\left(R_{1},R\right)R_{2}\right]+P^{\left(2\right)}\left[R\left(R_{1},R_{2}\right)\right],
\]
\[
P^{\left(2\right)}\left[R_{1}R\right]=P^{\left(2\right)}\left[\left(R_{1},R_{2}\right)R\right]+P^{\left(2\right)}\left[R_{1}\left(R_{2},R\right)\right],
\]
whence
\[
\begin{array}{c}
P^{\left(2\right)}\left[RR_{2}\right]+P^{\left(2\right)}\left[R_{1}R\right]-P^{\left(2\right)}\left[R_{1}R_{2}\right]\\
\\
=P^{\left(2\right)}\left[R\left(R_{1},R_{2}\right)\right]+P^{\left(2\right)}\left[\left(R_{1},R_{2}\right)R\right]\geq0.
\end{array}
\]
Note that $P^{\left(2\right)}$ is not a pseudometric because generally
\[
\begin{array}{l}
P^{\left(2\right)}\left[R_{1}R_{2}\right]=\Pr\left[R_{1}=1,R_{2}=2\right]\\
\not=\Pr\left[R_{2}=1,R_{1}=2\right]=P^{\left(2\right)}\left[R_{2}R_{1}\right].
\end{array}
\]
Nor is $P^{\left(2\right)}$ a quasimetric because it may very well
happen that $R_{1}\not=R_{2}$ but
\[
P^{\left(2\right)}\left[R_{1}R_{2}\right]=\Pr\left[R_{1}=1,R_{2}=2\right]=0.
\]

To use this p.q.-metric for our purposes: each random variable $H_{x^{\alpha}}$
in the hypothetical JDC-vector $H$ has a set of possible values $\mathcal{A}_{\alpha}$,
in which we choose and fix a measurable subset $\mathcal{A}_{x^{\alpha}}^{+}$
and its complement $\mathcal{A}_{x^{\alpha}}^{-}$. Note that $\mathcal{A}_{\alpha}$
is the same for all factor points of the factor $\alpha$ (and coincides
with the spectrum of the random variable in the diagram (\ref{eq:DSI})
which is supposed to be selectively influenced by $\alpha$). Transform
each $H_{x^{\alpha}}$ as
\begin{equation}
R_{x^{\alpha}}=\left\{ \begin{array}{ccc}
1 & \textnormal{if} & H_{x^{\alpha}}\in\mathcal{A}_{x^{\alpha}}^{-},\\
2 & \textnormal{\textnormal{i}f} & H_{x^{\alpha}}\in\mathcal{A}_{x^{\alpha}}^{+},
\end{array}\right.\label{eq:plus/minus trans}
\end{equation}
and define, for each pair of factor points $x^{\alpha},y^{\beta}$,
\begin{equation}
Dx^{\alpha}y^{\beta}=P^{\left(2\right)}\left[R_{x^{\alpha}}R_{y^{\beta}}\right].\label{eq:plus/minus}
\end{equation}
Here, once again (see Section \ref{sub:Population-level-tests}),
we write $x^{\alpha}y^{\beta}$ in place of $\left(x^{\alpha},y^{\beta}\right)$.
This time we are going to formalize this notation as part the following
general convention: any \emph{chain} (a finite sequence) of factor
points will be written as a \emph{string of symbols}, without commas
and parentheses, such as $x_{1}^{\alpha_{1}}\ldots x_{l}^{\alpha_{l}}$,
$x^{\alpha}y^{\beta}z^{\gamma}$, etc.

The value of $Dx^{\alpha}y^{\beta}$ is computable for any $x^{\alpha}y^{\beta}$
which is part of a treatment $\phi\in T$. The test therefore consists
in checking whether
\begin{equation}
Dx_{1}^{\alpha_{1}}x_{l}^{\alpha_{l}}\leq Dx_{1}^{\alpha_{1}}x_{2}^{\alpha_{2}}+Dx_{2}^{\alpha_{2}}x_{3}^{\alpha_{3}}+\ldots+Dx_{l-1}^{\alpha_{l-1}}x_{l}^{\alpha_{l}}\label{eq:distance test}
\end{equation}
for any chain of factor points $x_{1}^{\alpha_{1}}\ldots x_{l}^{\alpha_{l}}$
($l\geq$3) satisfying (\ref{eq:pairwise}) and such that for some
treatments $\phi^{\left(1\right)},\ldots,\phi^{\left(l\right)}\in T$
(not necessarily pairwise distinct),
\begin{equation}
\left\{ x_{1}^{\alpha_{1}},x_{l}^{\alpha_{l}}\right\} \subset\phi^{\left(1\right)},\left\{ x_{1}^{\alpha_{1}},x_{2}^{\alpha_{2}}\right\} \subset\phi^{\left(2\right)},\ldots,\left\{ x_{l-1}^{\alpha_{l-1}},x_{l}^{\alpha_{l}}\right\} \subset\phi^{\left(l\right)}.\label{eq:dist cond}
\end{equation}
Note that this is just another way of writing (\ref{eq:dist_cond_intro})-(\ref{eq:dist test intro}).
If the test is failed (i.e., the inequality is violated) for at least
one such sequence of factor points, then the hypothesis $(A_{1},\ldots,A_{n})\looparrowleft(\alpha_{1},\ldots,\alpha_{n})$
is rejected. In the following we will refer to any sequence of factor
points $x_{1}^{\alpha_{1}}\ldots x_{l}^{\alpha_{l}}$ ($l\geq3$)
subject to (\ref{eq:pairwise}) and (\ref{eq:dist cond}) as a \emph{treatment-realizable}
chain.
\begin{example}
\label{exa:classification test}Let $\alpha=\left\{ 1^{\alpha},2^{\alpha}\right\} $,
$\beta=\left\{ 1^{\beta},2^{\beta}\right\} $, and the set of allowable
treatments $T$ consist of all four possible combinations of the factor
points. Let $(A,B)$ be bivariate normally distributed at every treatment
$\phi$, with standard normal marginals and with correlations
\[
\rho\left(x^{\alpha},y^{\beta}\right)=\begin{cases}
-.9 & \textnormal{at }\left\{ x^{\alpha},y^{\beta}\right\} =\left\{ 1^{\alpha},1^{\beta}\right\} ,\\
+.9 & \textnormal{at }\left\{ x^{\alpha},y^{\beta}\right\} =\left\{ 1^{\alpha},2^{\beta}\right\} ,\\
+.9 & \textnormal{at }\left\{ x^{\alpha},y^{\beta}\right\} =\left\{ 2^{\alpha},1^{\beta}\right\} ,\\
-.1 & \textnormal{at }\left\{ x^{\alpha},y^{\beta}\right\} =\left\{ 2^{\alpha},2^{\beta}\right\} .
\end{cases}
\]
We form variables
\[
A^{*}\left(\phi\right)=\begin{cases}
1 & \textnormal{if }A\left(\phi\right)\leq0,\\
2 & \textnormal{if }A\left(\phi\right)>0,
\end{cases}\quad B^{*}\left(\phi\right)=\begin{cases}
1 & \textnormal{if }B\left(\phi\right)\leq0,\\
2 & \textnormal{if }B\left(\phi\right)>0,
\end{cases}
\]
with all marginals obviously constrained to 0.5, for all treatments.
The joint distributions are computed to be

\

\begin{center}%
\begin{tabular}{cc|cc|c|}
\hline 
\multicolumn{1}{|c}{$\alpha$} & $\beta$ & $A^{*}$ & $B^{*}$ & $\Pr$\tabularnewline
\hline 
\hline 
\multicolumn{1}{|c}{1} & 1 & 1 & 1 & $\cdots$\tabularnewline
\cline{1-2} 
 &  & 1 & 2 & $.428217$\tabularnewline
 &  & 2 & 1 & $\cdots$\tabularnewline
 &  & 2 & 2 & $\cdots$\tabularnewline
\cline{3-5} 
\end{tabular}$\quad$%
\begin{tabular}{cc|cc|c|}
\hline 
\multicolumn{1}{|c}{$\alpha$} & $\beta$ & $A^{*}$ & $B^{*}$ & $\Pr$\tabularnewline
\hline 
\hline 
\multicolumn{1}{|c}{1} & 2 & 1 & 1 & $\cdots$\tabularnewline
\cline{1-2} 
 &  & 1 & 2 & $.0717831$\tabularnewline
 &  & 2 & 1 & $\cdots$\tabularnewline
 &  & 2 & 2 & $\cdots$\tabularnewline
\cline{3-5} 
\end{tabular}

\

\begin{tabular}{cc|cc|c|}
\hline 
\multicolumn{1}{|c}{$\alpha$} & $\beta$ & $A^{*}$ & $B^{*}$ & $\Pr$\tabularnewline
\hline 
\hline 
\multicolumn{1}{|c}{2} & 1 & 1 & 1 & $\cdots$\tabularnewline
\cline{1-2} 
 &  & 1 & 2 & $.0717831$\tabularnewline
 &  & 2 & 1 & $\cdots$\tabularnewline
 &  & 2 & 2 & $\cdots$\tabularnewline
\cline{3-5} 
\end{tabular}$\quad$%
\begin{tabular}{cc|cc|c|}
\hline 
\multicolumn{1}{|c}{$\alpha$} & $\beta$ & $A^{*}$ & $B^{*}$ & $\Pr$\tabularnewline
\hline 
\hline 
\multicolumn{1}{|c}{2} & 2 & 1 & 1 & $\cdots$\tabularnewline
\cline{1-2} 
 &  & 1 & 2 & $.265942$\tabularnewline
 &  & 2 & 1 & $\cdots$\tabularnewline
 &  & 2 & 2 & $\cdots$\tabularnewline
\cline{3-5} 
\end{tabular}\end{center}

\

\protect{\noindent}where for each treatment $\phi$ we only show
the probabilities $\Pr\left[A^{*}=1,B^{*}=2\right]=P^{\left(2\right)}\left[A^{*}B^{*}\right]$,
other probabilities being irrelevant for our computations. Since $\left\{ 1^{\alpha},1^{\beta}\right\} $,
$\left\{ 1^{\alpha},2^{\beta}\right\} $,$\left\{ 2^{\alpha},2^{\beta}\right\} $,
and $\left\{ 2^{\alpha},1^{\beta}\right\} $ are all allowable treatment,
1$^{\alpha}2^{\beta}2^{\alpha}1^{\beta}$ is a treatment-realizable
chain. We can put therefore 
\[
Dx^{\alpha}y^{\beta}=P^{\left(2\right)}\left[A^{*}\left(x^{\alpha},y^{\beta}\right)B^{*}\left(x^{\alpha},y^{\beta}\right)\right]
\]
and observe that
\[
.428217=D1^{\alpha}1^{\beta}>D1^{\alpha}2^{\beta}+D2^{\alpha}2^{\beta}+D2^{\alpha}1^{\beta}=0.409508.
\]
This violation of the chain inequality rules out $\left(A,B\right)\looparrowleft\left(\alpha,\beta\right)$.\qed
\end{example}
The formulation of the test (\ref{eq:distance test}), subject to
(\ref{eq:pairwise}) and (\ref{eq:dist cond}), is valid for any p.q.-metric
$D$ imposed on the hypothetical JDC-vector $H$ for the diagram (\ref{diag:bijective DSI}).
It turns out, however, that using all possible treatment-realizable
chains $x_{1}^{\alpha_{1}}\ldots x_{l}^{\alpha_{l}}$ of factor points
would be redundant, in view of the lemma below. For its formulation
we need an additional concept. A treatment-realizable chain $x_{1}^{\alpha_{1}}\ldots x_{l}^{\alpha_{l}}$
($l\geq$3) is called \emph{irreducible} if 
\begin{enumerate}
\item the only nonempty subsets thereof that are subsets of treatments are
the pairs listed in (\ref{eq:dist cond}), and
\item no factor point in it occurs more than once.
\end{enumerate}
Thus, a triadic treatment-realizable chain $x^{\alpha}y^{\beta}z^{\gamma}$
is irreducible if and only if there is no treatment $\phi$ that includes
$\left\{ x^{\alpha},y^{\beta},z^{\gamma}\right\} $. Tetradic treatment-realizable
chains of the form $x^{\alpha}y^{\beta}u^{\alpha}v^{\beta}$ are irreducible
if and only if $x^{\alpha}\neq u^{\alpha}$ and $y^{\beta}\not=v^{\beta}$.
\begin{thm}[Distance-type Tests]
\label{thm:irreducible}Given a p.q.-metric $D$ on the hypothetical
JDC-vector $H$ for the diagram (\ref{diag:bijective DSI}), the inequality
(\ref{eq:distance test}) is satisfied for all treatment-realizable
chains if and only if this inequality holds for all irreducible chains.
\end{thm}
This theorem is an immediate consequence of Lemma \ref{lem:dist2}
in the appendix, where it is proved for a general set-up involving
arbitrary sets of random entities and factors. 

Note that if $T$ includes all possible combinations of factor points,
$T=\alpha_{1}\times\ldots\times\alpha_{m}$ ({}``completely crossed
design''), then the condition of treatment-realizability is equivalent
to (\ref{eq:pairwise}). In this situation any set of factor points
belonging to pairwise different factors (e.g., $\left\{ x^{\alpha},y^{\beta}\right\} $,
or $\left\{ x^{\alpha},y^{\beta},z^{\gamma}\right\} $ with $\alpha\not=\beta\not=\gamma\not=\alpha$)
belongs to some treatment, whence an irreducible chain cannot contain
factor points of more than two distinct factors: they must all be
of the form $x_{1}^{\alpha}x_{2}^{\beta}x_{3}^{\alpha}x_{4}^{\beta}...x_{2k-1}^{\alpha}x_{2k}^{\beta}$
($\alpha\not=\beta$). It is easy to see, however, that if $k>2$,
each of the subsets $\left\{ x_{1}^{\alpha},x_{4}^{\beta}\right\} $
and $\left\{ x_{2}^{\beta},x_{5}^{\alpha}\right\} $ belongs to a
treatment. It follows that that all irreducible chains in a completely
crossed design are of the form $x^{\alpha}y^{\beta}u^{\alpha}v^{\beta}$,
with $\alpha\not=\beta$, $x^{\alpha}\neq u^{\alpha}$ and $y^{\beta}\not=v^{\beta}$. 
\begin{thm}[Distance-type Tests for Completely Crossed Designs]
If the set of treatments $T$ consists of all possible combinations
of factor points, then the inequality (\ref{eq:distance test}) is
satisfied for all treatment-realizable sequences of factor points
if and only if this inequality holds for all tetradic sequences of
the form $x^{\alpha}y^{\beta}u^{\alpha}v^{\beta}$, with $\alpha\not=\beta$,
$x^{\alpha}\neq u^{\alpha}$ and $y^{\beta}\not=v^{\beta}$.
\end{thm}
This formulation is given in Dzhafarov and Kujala (2010), although
there it is unnecessarily confined to metrics of a special kind, denoted
$M^{\left(p\right)}$ below.

\subsection{Classes of p.q.-metrics}

Let us consider some classes of p.q.-metrics that can be used for
distance-type tests. We do not attempt a systematization or maximal
generality, our goals being to show the reader how broad the spectrum
of the usable p.q.-metrics is, and how easy it is to generate new
ones.

\subsubsection{Minkowki-type metrics\label{sub:Minkowki-type-metrics}}

These are (conventional, symmetric) metrics of the type

\begin{equation}
M^{\left(p\right)}\left(A,B\right)=\begin{cases}
\sqrt[p]{\mathrm{E}\left[\left|A-B\right|^{p}\right]} & \textnormal{for }1\leq p<\infty,\\
\mathrm{ess}\sup\left|A-B\right| & \textnormal{for }p=\infty,
\end{cases}\label{eq:M^p}
\end{equation}
where 
\[
\mathrm{ess}\sup\left|A-B\right|=\inf\left\{ v:\Pr\left[\left|A-B\right|\leq v\right]=1\right\} .
\]
In the context of selective influences these metrics have been introduced
in Kujala and Dzhafarov (2008) and further analyzed in Dzhafarov and
Kujala (2010). The metric $M$ discussed in Section \ref{sub:Population-level-tests}
is a special case ($p=1$). An important property of $M^{\left(p\right)}$
is that the result of an $M^{\left(p\right)}$-based distance-type
test is not invariant with respect to factor-point-specific transformations
of the random variables. This allows one to conduct an infinity of
different tests on one and the same $A(\phi)=(A_{1},\ldots,A_{n})(\phi)$.
For numerous examples of how the test works see Kujala and Dzhafarov
(2008) and Dzhafarov and Kujala (2010).

\subsubsection{Classification p.q.-metrics\label{sub:Classification-p.q.-metrics}}

Classification p.q.-metrics are the p.q.-metrics defined through the
p.q.-metric $P^{\left(2\right)}$ by (\ref{eq:plus/minus}), following
a transformation (\ref{eq:plus/minus trans}). The general definition
is that for each random variable $X$ in a set of jointly distributed
random variables $\mathcal{R}$ we designate two complementary events
$E_{X}^{-}$ and $E_{X}^{+}$, and put 
\[
D_{C}\left(A,B\right)=\Pr\left[E_{A}^{-}\&E_{B}^{+}\right].
\]
The results of a $D_{C}$-based distance-type test for selective influences
depend on the choice of the events $E_{X}^{+}$, so different choices
would lead to different tests for one and the same $A\left(\phi\right)=(A_{1},\ldots,A_{n})(\phi)$.
See Example \ref{exa:classification test} for an illustration. 

To the best of our knowledge this interesting p.q.-metric was not
previously considered in mathematics. One standard way to generalize
it (see the principles of constructing derivative metrics in Section
\ref{sub:Generating-p.q.-metrics} below) is to make the set of events
$\left\{ E_{X}^{+}:X\in\mathcal{R}\right\} $ a random entity. In
the special case when all random variables in $\mathcal{R}$ take
their values in the set of real numbers, and $E_{X}^{+}$ for each
$X\in\mathcal{R}$ is defined by $X\geq v$, the {}``randomization''
of $\left\{ E_{X}^{+}:X\in\mathcal{R}\right\} $ reduces to that of
$v$. The p.q.-metric then becomes
\[
D_{S}\left(A,B\right)=\Pr\left[A\leq V<B\right]
\]
where $V$ is a random variable. An additively symmetrized (i.e.,
pseudometric) version of this p.q.-metric, $D_{S}\left(A,B\right)+D_{S}\left(B,A\right)$,
was introduced in Taylor (1984, 1985) under the name {}``separation
(pseudo)metric,'' and shown to be a conventional metric if $V$ is
chosen stochastically independent of all random variables in $\mathcal{R}$.

\subsubsection{Information-based p.q.-metric}

Let the jointly distributed random variables constituting the set
$\mathcal{R}$ be all discrete. Perhaps the simplest information-based
p.q.-metric is
\[
h\left(A|B\right)=-\sum_{a,b}p_{AB}\left(a,b\right)\log\frac{p_{AB}\left(a,b\right)}{p_{B}\left(b\right)},
\]
with the conventions $0\log\frac{0}{0}=0\log0=0$. is This function
is called \emph{conditional entropy}. The identity $h\left(A|A\right)=0$
is obvious, and the triangle inequality,
\[
h\left(A|B\right)\leq h\left(A|C\right)+h\left(C|B\right),
\]
follows from the standard information theory (in)equalities,
\[
h\left(A|B\right)\leq h\left(A,C|B\right),
\]
\[
h\left(A,C|B\right)=h\left(A|C,B\right)+h\left(C|B\right),
\]
and
\[
h\left(A|C,B\right)\leq h\left(A|C\right).
\]
Note that the test of selectiveness based on $h\left(A,B\right)$
(and any other information-based measure) is invariant with respect
to all bijective transformations of the variables.

The additively symmetrized (i.e., pseudometric) version of this p.q.-metric,
$h\left(A|B\right)+h\left(B|A\right)$ is well-known (Cover \& Thomas,
1990). Normalized versions of $h\left(A|B\right)$ are also of interest,
for instance,
\[
h_{N}\left(A|B\right)=\frac{2h\left(A|B\right)}{h\left(A,B\right)},
\]
where 
\[
h\left(A,B\right)=-\sum_{a,b}p_{AB}\left(a,b\right)\log p_{AB}\left(a,b\right),
\]
the \emph{joint entropy} of $A$ and $B$; $h_{N}\left(A|B\right)$
is bound between $0$ (attained when $A$ is a bijective transformation
of $B$) and $1$ (when $A$ and $B$ are independent). A proof of
the triangle inequality for $h_{N}$ can be found in Kraskov et al.
(2003), as part of their proof that $\frac{1}{2}\left[h_{N}\left(A|B\right)+h_{N}\left(B|A\right)\right]$
is a pseudometric.

\subsubsection{\label{sub:Generating-p.q.-metrics}Constructing p.q.-metrics from
other p.q.-metrics}

There are numerous ways of creating new p.q.-metrics from the ones
mentioned above, or from ones taken from outside probabilistic context.
Thus, if $d$ is a p.q.-metric on a set $S$, then, for any space
$\mathcal{R}$ of jointly distributed random variables taking their
values in $S$,
\[
D\left(A,B\right)=\mathrm{E}\left[d\left(A,B\right)\right],\quad A,B\in\mathcal{R},
\]
is a p.q.-metric on $\mathcal{R}$. This follows from the fact that
expectation $\mathrm{E}$ preserves inequalities and equalities identically
satisfied for all possible realizations of the arguments. Thus, the
distance $M\left(A,B\right)=\mathrm{E}\left[\left|A-B\right|\right]$
of Section \ref{sub:General logic} trivially obtains from the metric
$d\left(a,b\right)=\left|a-b\right|$ on reals. In the same way one
obtains the well-known Fréchet distance
\[
F\left(A,B\right)=\mathrm{E}\left[\frac{\left|A-B\right|}{1+\left|A-B\right|}\right].
\]

Below we present an incomplete list of transformations which, given
a p.q.-metric (quasimetric, pseudometric, conventional metric) $D$
on a space $\mathcal{R}$ of jointly distributed random variables
produces a new p.q.-metric (respectively, quasimetric, pseudometric,
or conventional metric) on the same space. The proofs are trivial
or well-known, so we omit them. The arrows $\Longrightarrow$ should
be read {}``can be transformed into.''
\begin{enumerate}
\item $D\Longrightarrow D^{q}$ ($q<1$). In this way, for example, we can
obtain metrics 
\[
M^{\left(p,q\right)}\left(A,B\right)=\begin{cases}
\left(\mathrm{E}\left[\left|A-B\right|^{p}\right]\right)^{q/p} & \textnormal{for }1\leq p<\infty,q\leq1\\
\left(\mathrm{ess}\sup\left|A-B\right|\right)^{q} & \textnormal{for }p=\infty,q\leq1
\end{cases}
\]
from the metrics $M^{\left(p\right)}$ in (\ref{eq:M^p}).
\item $D\Longrightarrow D/\left(1+D\right)$. This is a standard way of
creating a bounded p.q.-metric. 
\item $D_{1},D_{2}\Longrightarrow\max\left\{ D_{1},D_{2}\right\} $ or $D_{1},D_{2}\Longrightarrow D_{1}+D_{2}$.
This transformations can be used to symmetrize p.q.-metrics: $D\left(A,B\right)+D\left(B,A\right)$
or $\max\left\{ D\left(A,B\right),D\left(B,A\right)\right\} $.
\item A generalization of the previous: $\left\{ D_{\upsilon}:\upsilon\in\Upsilon\right\} \Longrightarrow\sup\left\{ D_{\upsilon}\right\} $
and $\left\{ D_{\upsilon}:\upsilon\in\Upsilon\right\} \Longrightarrow\mathrm{E}\left[D_{V}\right]$,
where $\left\{ D_{\upsilon}:\upsilon\in\Upsilon\right\} $ is a family
of p.q.-metrics, and $V$ designates a random entity distributed as
$\left(\Upsilon,\Sigma_{\Upsilon},m\right)$, so that
\[
D\left(A,B\right)=\int_{\upsilon\in\Upsilon}D_{\upsilon}\left(A,B\right)\mathrm{dm\left(\upsilon\right)}.
\]
We have discussed in Section \ref{sub:Classification-p.q.-metrics}
how such a procedure leads from our {}``classification'' p.q.-metrics
$D_{C}$ to {}``separation'' p.q.-metrics $D_{S}$. 
\end{enumerate}

\section{\label{sec:Non-distance-tests}Non-distance tests}

The general principle of constructing tests for selective influences
presented in Section \ref{sub:General logic} does not only lead to
distance-type tests. In this section we will consider two examples,
one proposed previously and one new, of tests in which the functionals
$g\left(H_{\phi}\right)$ mentioned in Section \ref{sub:General logic}
are, respectively, two-argument but not distance-type, and multiple-argument
ones. Recall that the tests in question are only necessary conditions
for selective influences (in the form of the diagram \ref{diag:bijective DSI}).

\subsection{Cosphericity test\label{sub:Cosphericity-test}}

Given a hypothetical JDC-vector 
\[
H=\left(H_{x_{1}^{\alpha_{1}}},\ldots,H_{x_{k_{1}}^{\alpha_{i}}},\ldots,H_{x_{1}^{\alpha_{n}}},\ldots,H_{x_{k_{n}}^{\alpha_{n}}}\right)
\]
with real-valued random variables, the following statement $\mathfrak{S}$
should be satisfied: for any quadruple of factor points $\left\{ x^{\alpha},y^{\beta},u^{\alpha},v^{\beta}\right\} $
with $\alpha\not=\beta$ such that for some treatments $\phi_{1},\phi_{2},\phi_{3},\phi_{4}\in T$,
\[
\left\{ x^{\alpha},y^{\beta}\right\} \subset\phi_{1},\left\{ x^{\alpha},v^{\beta}\right\} \subset\phi_{2},\left\{ u^{\alpha},y^{\beta}\right\} \subset\phi_{3},\left\{ u^{\alpha},v^{\beta}\right\} \subset\phi_{4},
\]
we have
\[
\begin{array}{l}
\left\vert \rho_{x^{\alpha}y^{\beta}}\rho_{x^{\alpha}v^{\beta}}-\rho_{u^{\alpha}y^{\beta}}\rho_{u^{\alpha}v^{\beta}}\right\vert \\
\\
\leq\sqrt{1-\rho_{x^{\alpha}y^{\beta}}^{2}}\sqrt{1-\rho_{x^{\alpha}v^{\beta}}^{2}}+\sqrt{1-\rho_{u^{\alpha}y^{\beta}}^{2}}\sqrt{1-\rho_{u^{\alpha}v^{\beta}}^{2}},
\end{array}
\]
where $\rho_{x^{\alpha}y^{\beta}}$ denotes the correlation between
$H_{x^{\alpha}}$ and $H_{y^{\beta}}$, $\rho_{x^{\alpha}u^{\beta}}$
denotes the correlation between $H_{x^{\alpha}}$ and $H_{u^{\beta}}$,
etc. Ergo, if the inequality is violated for at least one such a quadruple
of factor points, the JDC-vector cannot exist, and the diagram of
selective influences \ref{diag:bijective DSI} should be rejected.
For numerous illustrations see Kujala and Dzhafarov (2008), where
this test has been proposed, and where it is also shown that for two
bivariate normally distributed variables in a $2\times2$ factorial
design this test is both a necessary and sufficient condition for
selective influences.

\subsection{Diversity Test\label{sub:Diversity-Test}}

The p.q.-metric $P^{\left(2\right)}$ introduced in Section \ref{sec:Distance-type-tests}
lends itself to an interesting generalization. Let $\mathcal{R}$
be a set of jointly distributed random variables, each having $\left\{ 1,2,\ldots,s\right\} $
as its set of possible values. Define
\[
\begin{array}{l}
P^{\left(s\right)}\left[\left(R_{1}^{1},\ldots,R_{1}^{k_{1}}\right)\ldots\left(R_{i}^{1},\ldots,R_{i}^{k_{i}}\right)\ldots\left(R_{s}^{1},\ldots,R_{s}^{k_{s}}\right)\right]\\
=\Pr\left[\begin{array}{c}
R_{i}^{j}=i,\textnormal{ for }j=1,\ldots,k_{i}\textnormal{ and }i=1,\ldots,s\end{array}\right].
\end{array}
\]
In particular,
\[
\Pr\left[\begin{array}{c}
R_{1}=1,\ldots,R_{s}=s\end{array}\right]=P^{\left(s\right)}\left[\left(R_{1}\right)\ldots\left(R_{s}\right)\right].
\]
It is easy to show that the latter is a generalized p.q.-distance,
in the sense of satisfying the following two properties: for any $R_{1},\ldots,R_{s},R\in\mathcal{R}$,
\begin{enumerate}
\item (generalized premetric) $P^{\left(s\right)}\left[\left(R_{1}\right)\ldots\left(R_{s}\right)\right]$
is nonnegative, and it is zero if any two of $R_{1},\ldots,R_{s}$
are identical.
\item (simplicial inequality): 
\[
\begin{array}{l}
P^{\left(s\right)}\left[\left(R_{1}\right)\ldots\left(R_{s}\right)\right]\leq\sum_{i=1}^{s}P^{\left(s\right)}\left[\left(R_{1}\right)\ldots\left(R\right)\ldots\left(R_{s}\right)\right],\end{array}
\]
where in the $i$th summand on the right, $R_{i}$ in the sequence
$\left(R_{1}\right)\ldots\left(R_{i}\right)\ldots\left(R_{s}\right)$
is replaced with $R$ ($i=1,\ldots,s$), the rest of the sequence
remaining intact.%
\footnote{With the addition of permutation-invariance, functions $\mathcal{R}$$^{s}\rightarrow\mathbb{R}$
(with $\mathcal{R}$ an arbitrary set) satisfying these properties
are sometimes called $\left(s-1\right)$-semimetrics (Deza \& Rosenberg,
2000); with the addition of the property that $P^{\left(s\right)}>0$
if no two arguments thereof are equal, they become $\left(s-1\right)$-metrics.%
}
\end{enumerate}
The generalized premetric property is obvious. To avoid cumbersome
notation, let us prove the simplicial inequality for $s=3$, the generalization
to arbitrary $s$ being straightforward. We drop in $P^{\left(3\right)}$
the parentheses around singletons: $P^{\left(3\right)}\left[R_{1}R_{2}R_{3}\right]$,
$P^{\left(3\right)}\left[R_{1}\left(R_{2},R\right)R_{3}\right]$,
etc. The simplicial inequality in question is

\[
P^{\left(3\right)}\left[R_{1}R_{2}R_{3}\right]\leq P^{\left(3\right)}\left[RR_{2}R_{3}\right]+P^{\left(3\right)}\left[R_{1}RR_{3}\right]+P^{\left(3\right)}\left[R_{1}R_{2}R\right].
\]
We have
\[
\begin{array}{l}
P^{\left(3\right)}\left[R_{1}R_{2}R_{3}\right]\\
\\
=P^{\left(3\right)}\left[\left(R_{1},R\right)R_{2}R_{3}\right]+P^{\left(3\right)}\left[R_{1}\left(R_{2},R\right)R_{3}\right]+P^{\left(3\right)}\left[R_{1}R_{2}\left(R_{3},R\right)\right],
\end{array}
\]
\[
\begin{array}{l}
P^{\left(3\right)}\left[RR_{2}R_{3}\right]\\
\\
=P^{\left(3\right)}\left[\left(R_{1},R\right)R_{2}R_{3}\right]+P^{\left(3\right)}\left[R\left(R_{1},R_{2}\right)R_{3}\right]+P^{\left(3\right)}\left[RR_{2}\left(R_{1},R_{3}\right)\right],
\end{array}
\]
and analogously for $P^{\left(3\right)}\left[R_{1}RR_{3}\right]$
and $P^{\left(3\right)}\left[R_{1}R_{2}R\right]$. Then
\[
\begin{array}{r}
P^{\left(3\right)}\left[RR_{2}R_{3}\right]+P^{\left(3\right)}\left[R_{1}RR_{3}\right]+P^{\left(3\right)}\left[R_{1}R_{2}R\right]-P^{\left(3\right)}\left[R_{1}R_{2}R_{3}\right]\\
\\
=P^{\left(3\right)}\left[R\left(R_{1},R_{2}\right)R_{3}\right]+P^{\left(3\right)}\left[RR_{2}\left(R_{1},R_{3}\right)\right]\\
\\
+P^{\left(3\right)}\left[\left(R_{1},R_{2}\right)RR_{3}\right]+P^{\left(3\right)}\left[R_{1}R\left(R_{2},R_{3}\right)\right]\\
\\
+P^{\left(3\right)}\left[\left(R_{1},R_{3}\right)R_{2}R\right]+P^{\left(3\right)}\left[R_{1}\left(R_{2},R_{3}\right)R\right]\geq0.
\end{array}
\]

We call $P^{\left(s\right)}$ a \emph{diversity }function. To use
this function for a test of selective influences, for each random
variable $H_{x^{\alpha}}$ in the hypothetical JDC-vector $H$ we
partition the set of its possible values $\mathcal{A}_{x^{\alpha}}$
into $s$ pairwise disjoint subsets $\mathcal{A}_{x^{\alpha}}^{1},\ldots,\mathcal{A}_{x^{\alpha}}^{s}$,
and we transform $H_{x^{\alpha}}$ as
\[
R_{x^{\alpha}}=\left\{ \begin{array}{ccc}
1 & \textnormal{if} & H_{x^{\alpha}}\in\mathcal{A}_{x^{\alpha}}^{1},\\
\vdots & \vdots & \vdots\\
s & \textnormal{if} & H_{x^{\alpha}}\in\mathcal{A}_{x^{\alpha}}^{s}.
\end{array}\right.
\]
Define
\[
Dx_{1}^{\mu_{1}}\ldots x_{s}^{\mu_{s}}=P^{\left(s\right)}\left[R_{x_{1}^{\mu_{1}}}\ldots R_{x_{s}^{\mu_{s}}}\right].
\]
Let us restrict the consideration to $s=3$ again. Assuming all factor
points mentioned below belong to $\bigcup\Phi$, and given a triadic
chain of factor points $t=x^{\alpha}y^{\beta}z^{\gamma}$ (with the
elements pairwise distinct), we define a certain set of triadic chains
referred to as a \emph{polyhedral set} over $t$. 
\begin{enumerate}
\item For any triadic chain $t=x^{\alpha}y^{\beta}z^{\gamma}$ ($x^{\alpha}\not=y^{\beta}\not=z^{\gamma}\not=x^{\alpha}$)
and any $u^{\mu}\notin\left\{ x^{\alpha},y^{\beta},z^{\gamma}\right\} $,
the set $\left\{ u^{\mu}y^{\beta}z^{\gamma},x^{\alpha}u^{\mu}z^{\gamma},x^{\alpha}y^{\beta}u^{\mu}\right\} $
is a polyhedral set over $t$;
\item For any triadic chains $t$ and $t'$, if $\mathcal{\mathfrak{P}}$
is a polyhedral set over $t$, and $\mathfrak{\mathcal{\mathfrak{P}}}'$
is a polyhedral set over any $t'\in\mathcal{\mathfrak{P}}$, then
the set $\left(\mathcal{\mathfrak{P}}-\left\{ t'\right\} \right)\cup\mathfrak{P'}$
is a polyhedral set over $t$.
\item Any polyhedral set over any triadic chain $t$ is obtained by a finite
number of applications of 1 and 2 above.
\end{enumerate}
We call such a set polyhedral because if one interprets each element
of it as a list of vertices forming a (triangular) face, then the
whole set, combined with the root face $t$, forms a complete polyhedron.

A polyhedral set $\mathcal{\mathfrak{P}}$ over $t=x^{\alpha}y^{\beta}z^{\gamma}$
is called \emph{treatment-realizable} if each element (triadic chain)
that belongs to $\mathcal{\mathfrak{P}}$$\cup\left\{ t\right\} $
consists of elements of some treatment $\phi\in T$ (which implies,
in particular, $\alpha\not=\beta\not=\gamma\not=\alpha$). The diversity
test for selective influences consists in checking the compliance
of the hypothetical JDC-vector with the following statement: for any
treatment-realizable polyhedral set $\mathcal{\mathfrak{P}}$ over
$x_{1}^{\mu_{1}}x_{2}^{\mu_{2}}x_{3}^{\mu_{3}}$,

\begin{equation}
Dx_{1}^{\mu_{1}}x_{2}^{\mu_{2}}x_{3}^{\mu_{3}}\leq\sum_{x_{i}^{\mu_{i}}x_{j}^{\mu_{j}}x_{k}^{\mu_{k}}\in\mathcal{\mathfrak{P}}}Dx_{i}^{\mu_{i}}x_{j}^{\mu_{j}}x_{k}^{\mu_{k}}.\label{eq:simplicial chain}
\end{equation}
The inequality trivially follows from the simplicial inequality and
the definition of $\mathfrak{P}$.

The classification p.q.-metric tests considered earlier form a special
case of the diversity tests. For complete analogy one should replace
chains in the formulation of the $P^{\left(2\right)}$-based tests
with a \emph{polygonal set} $\mathcal{\mathfrak{P}}$ of pairs of
factor points (\emph{dipoles}) over a given dipole $d=x^{\alpha}y^{\beta}$
($x^{\alpha}\not=y^{\beta}$). This set is defined as a set obtainable
by repeated applications of the following two rules: 
\begin{enumerate}
\item for any $d=x^{\alpha}y^{\beta}$ ($x^{\alpha}\not=y^{\beta}$) and
any $u^{\mu}\notin\left\{ x^{\alpha},y^{\beta}\right\} $, the set
$\left\{ u^{\mu}y^{\beta},x^{\alpha}u^{\mu}\right\} $ is a polygonal
set over $d$; 
\item if $\mathcal{\mathfrak{P}}$ is a polygonal set over $d$, and $\mathfrak{P'}$
is a polygonal set over any $d'\in\mathcal{\mathfrak{P}}$, then the
set $\left(\mathcal{\mathfrak{P}}-\left\{ d'\right\} \right)\cup\mathfrak{P'}$
is a polygonal set over $d$. 
\end{enumerate}
The generalization to $s>3$ involves \emph{polytopal sets} of $s$-element
chains and is conceptually straightforward. The notion of an irreducible
chain is also generalizable to polytopal sets, but we are not going
to discuss this and related issues here: the diversity function and
diversity-based tests form a rich topic that deserves a special investigation.
\begin{example}
Let $\alpha,\beta,\gamma,\delta$ be binary (1/2) factors, and let
the set of allowable treatments $T$ consist of all combinations of
the factor points subject to the following constraint: $\left\{ 1^{\alpha},1^{\beta},2^{\gamma},1^{\delta}\right\} $
is the only treatment in $T$ of the forms $\left\{ 1^{\alpha},1^{\beta},2^{\gamma},v^{\delta}\right\} $,
$\left\{ 1^{\alpha},1^{\beta},v^{\gamma},1^{\delta}\right\} $, $\left\{ 1^{\alpha},v^{\beta},2^{\gamma},1^{\delta}\right\} $,
and $\left\{ v^{\alpha},1^{\beta},2^{\gamma},1^{\delta}\right\} $.
Let the random variables $A,B,C,D$ in the hypothetical diagram $\left(A,B,C,D\right)\looparrowleft\left(\alpha,\beta,\gamma,\delta\right)$
each have three values, denoted 1,2,3, and let the distributions of
$\left(A,B,C,D\right)$ be as shown in the tables, with all omitted
joint probabilities being zero:

\

\begin{center}%
\begin{tabular}{cccc|cccc|c|}
\hline 
\multicolumn{1}{|c}{$\alpha$} & $\beta$ & $\gamma$ & $\delta$ & $A$ & $B$ & $C$ & $D$ & $\Pr$\tabularnewline
\hline 
\hline 
\multicolumn{1}{|c}{$x$} & $y$ & $z$ & $u$ & $\vdots$ & $\vdots$ & $\vdots$ & $\vdots$ & $\vdots$\tabularnewline
\cline{1-4} 
 &  &  &  & 1 & 2 & 3 & 1 & 1/3\tabularnewline
 &  &  &  & 1 & 2 & 3 & 2 & 1/3\tabularnewline
 &  &  &  & 1 & 2 & 3 & 3 & 1/3\tabularnewline
 &  &  &  & $\vdots$ & $\vdots$ & $\vdots$ & $\vdots$ & $\vdots$\tabularnewline
\cline{5-9} 
\end{tabular}$\quad$%
\begin{tabular}{cccc|cccc|c|}
\hline 
\multicolumn{1}{|c}{$\alpha$} & $\beta$ & $\gamma$ & $\delta$ & $A$ & $B$ & $C$ & $D$ & $\Pr$\tabularnewline
\hline 
\hline 
\multicolumn{1}{|c}{$1$} & $1$ & $2$ & $1$ & $\vdots$ & $\vdots$ & $\vdots$ & $\vdots$ & $\vdots$\tabularnewline
\cline{1-4} 
 &  &  &  & 1 & 2 & 3 & 1 & 1/2\tabularnewline
 &  &  &  & 1 & 2 & 3 & 2 & 1/2\tabularnewline
 &  &  &  & 1 & 2 & 3 & 3 & 0\tabularnewline
 &  &  &  & $\vdots$ & $\vdots$ & $\vdots$ & $\vdots$ & $\vdots$\tabularnewline
\cline{5-9} 
\end{tabular}\end{center}

\

\protect{\noindent}where $\left\{ x^{\alpha},y^{\beta},z^{\delta},u^{\gamma}\right\} $
is any treatment in $T$ other than $\left\{ 1^{\alpha},1^{\beta},2^{\delta},1^{\gamma}\right\} $.
It is easy to check that the 3-marginals (hence also all lower-order
marginals) of the distributions satisfy marginal selectivity. One
can also check that $\left\{ 1^{\alpha}1^{\beta}1^{\delta},1^{\alpha}1^{\gamma}1^{\delta},1^{\beta}1^{\gamma}1^{\delta}\right\} $
is a polyhedral set (in fact, the simplest one, forming a tetrahedron
with vertices $1^{\alpha},1^{\beta},1^{\gamma},1^{\delta}$). This
polyhedral set is treatment-realizable, because
\[
\begin{array}{c}
\left\{ 1^{\alpha},1^{\beta},1^{\gamma}\right\} \subset\left\{ 1^{\alpha},1^{\beta},1^{\gamma},2^{\delta}\right\} ,\quad\left\{ 1^{\alpha},1^{\beta},1^{\delta}\right\} \subset\left\{ 1^{\alpha},1^{\beta},2^{\gamma},1^{\delta}\right\} ,\\
\\
\left\{ 1^{\alpha},1^{\gamma},1^{\delta}\right\} \subset\left\{ 1^{\alpha},2^{\beta},1^{\gamma},1^{\delta}\right\} ,\quad\left\{ 1^{\beta},1^{\gamma},1^{\delta}\right\} \subset\left\{ 2^{\alpha},1^{\beta},1^{\gamma},1^{\delta}\right\} .
\end{array}
\]
Putting 
\[
\begin{array}{l}
D1^{\alpha}1^{\beta}1^{\gamma}=P^{\left(3\right)}\left[H_{1^{\alpha}}H_{1^{\beta}}H_{1^{\gamma}}\right]\\
=\Pr\left[\left\{ A=1,B=2,C=3\right\} \left(1^{\alpha},1^{\beta},1^{\gamma},2^{\delta}\right)\right]=1,\\
\\
D1^{\alpha}1^{\beta}1^{\delta}=P^{\left(3\right)}\left[H_{1^{\alpha}}H_{1^{\beta}}H_{1^{\delta}}\right]\\
=\Pr\left[\left\{ A=1,B=2,D=3\right\} \left(1^{\alpha},1^{\beta},2^{\gamma},1^{\delta}\right)\right]=0,\\
\\
D1^{\alpha}1^{\delta}1^{\gamma}=P^{\left(3\right)}\left[H_{1^{\alpha}}H_{1^{\delta}}H_{1^{\gamma}}\right]\\
=\Pr\left[\left\{ A=1,D=2,C=3\right\} \left(1^{\alpha},2^{\beta},1^{\gamma},1^{\delta}\right)\right]=\frac{1}{3},\\
\\
D1^{\delta}1^{\beta}1^{\gamma}=P^{\left(3\right)}\left[H_{1^{\delta}}H_{1^{\beta}}H_{1^{\gamma}}\right]\\
=\Pr\left[\left\{ D=1,B=2,C=3\right\} \left(2^{\alpha},1^{\beta},1^{\gamma},1^{\delta}\right)\right]=\frac{1}{3},
\end{array}
\]
where $H_{x^{\mu}}$ are elements of the hypothetical JDC-vector,
we see that the simplicial inequality is violated:
\[
1=D1^{\alpha}1^{\beta}1^{\gamma}>D1^{\alpha}1^{\beta}1^{\delta}+D1^{\alpha}1^{\delta}1^{\gamma}+D1^{\delta}1^{\beta}1^{\gamma}=\frac{2}{3}.
\]
This rules out the possibility of $\left(A,B,C,D\right)\looparrowleft\left(\alpha,\beta,\gamma,\delta\right)$.
\qed
\end{example}

\section{Conclusion}

Selectiveness in the influences exerted by a set of inputs upon a
set of random and stochastically interdependent outputs is a critical
feature of many psychological models, often built into the very language
of these models. We speak of an internal representation of a given
stimulus, as separate from an internal representation of another stimulus,
even if these representations are considered random entities and they
are not independent. We speak of decompositions of response time into
signal-dependent and signal-independent components, or into a perceptual
stage (influenced by stimuli) and a memory-search stage (influenced
by the number of memorized items), without necessarily assuming that
the two components or stages are stochastically independent. Moreover,
the same as with theory of measurement and model selection studies,
the issue of selective probabilistic influences, while born within
psychology and motivated by psychological theorizing, pertains in
fact to any area of empirical science dealing with inputs and random
outputs. 

In this paper, we have described the fundamental Joint Distribution
Criterion for selective influences, and proposed a direct application
of this criterion to random variables with finite numbers of values,
the Linear Feasibility Test for selective influences. This test can
be performed by means of standard linear programming. Due to the fact
that any random output can be discretized, the Linear Feasibility
Test is universally applicable, although one should keep in mind that
if a diagram of selective influences is upheld by the test at some
discretization, it may be rejected at a finer or non-nested discretization
(but not at a coarser one). 

Based on the Joint Distribution Criterion we have also formulated
a general scheme for constructing various necessary conditions (tests)
for selective influences. Among the tests thus generated is a wide
spectrum of distance-type tests and some other tests described in
the paper. The results of some of these tests (e.g., all those involving
expected values) are not invariant with respect to factor-point-specific
transformations of the random outputs, which allows one to expand
each of such tests into an infinity of different tests for different
transformations. 

The abundance of different tests which we now have at our disposal
poses new problems. The Linear Feasibility Test is superior to other
tests as it allows one to prove (rather than only disprove) the adherence
of a system of inputs and outputs to a given diagram of selective
influences (for a given discretization, if one is involved). It is
possible, however, that discretization is not desirable, or the size
of the problem is too large to be handled by available computational
methods. In these cases one faces the problem of devising an optimal,
or at least systematic way of applying a sequence of different necessary
conditions, such as distance-type tests. Let us call a test $T_{1}$
stronger than test $T_{2}$ with respect to a given diagram of selective
influences if the latter cannot be upheld by $T_{1}$ and rejected
by $T_{2}$, while the reverse is possible. Thus, in Kujala and Dzhafarov
(2008) it is shown that the cosphericity test (Section \ref{sub:Cosphericity-test})
is stronger than the Minkowski distance test with $p=2$ (Section
\ref{sub:Minkowki-type-metrics}). We know very little, however, about
the comparative strengths of different tests on a broader scale. 

The problem of devising optimal strategies of sequential testing arises
also within the confines a particular class of tests. Thus, the classification
test (Sections \ref{sub:General-theory} and \ref{sub:Classification-p.q.-metrics})
and the diversity test (Section \ref{sub:Diversity-Test}) can be
used repeatedly, each time with a different choice of the partitions
of the random outputs' domains. We do not know at present how to organize
the sequences of these choices optimally. In the case of the Minkowski
distance test we do not know in which order one should use different
values of $p$ and different factor-point-specific transformations
of the random variables. The latter also applies to the nonlinear
transformations in the cosphericity test.

Finally, adaptation of the population-level tests to data analysis
is another problem to be addressed by future research. Although sample-level
procedures corresponding to our tests seem conceptually straightforward
(Section \ref{sub:Sample-level-tests}), the issues of statistical
power and statistical interdependence compound the problems of comparative
strength of the tests and optimal strategy of sequential testing. 

\appendix

\onecolumngrid

{\centering\section{GENERALIZATIONS TO ARBITRARY SETS}

}

\twocolumngrid\setlength{\parindent}{10pt}

\subsubsection*{Random Entities and Variables}

For the purposes of this paper it is convenient to view a \emph{random
entity} $A$ as a quadruple $\left(`A\textnormal{'},\mathcal{A},\Sigma,\mu\right)$,
where $`A\textnormal{'}$ is a unique \emph{name}, $\mathcal{A}$
is a nonempty set (of values of $A$), $\Sigma$ is a sigma algebra
of subsets of $\mathcal{A}$ (called \emph{measurable} subsets), and
$\mu$ is a probability measure on $\Sigma$ with the interpretation
that $\mu\left(\mathfrak{a}\right)$ for any $\mathcal{\mathfrak{a}}\in\Sigma$
is the probability with which $A$ falls within $\mathcal{\mathfrak{a}\subset\mathcal{A}}$.
$\left(\mathcal{A},\Sigma\right)$ is referred to as the \emph{observation
space} for $A$\emph{. }We call the probability space $\left(\mathcal{A},\Sigma,\mu\right)$
the \emph{distribution} for $A$ and say that $A$ is \emph{distributed
as} $\left(\mathcal{A},\Sigma,\mu\right)$. The inclusion of the label
$`A\textnormal{'}$ is needed to ensure an unlimited collection of
distinct random entities with the same distribution. If two random
entities $A$ and $A'$ have the same distribution, we write $A\sim A'$.
If $A$ and $B$ are distributed as, respectively, $\left(\mathcal{A},\Sigma_{\mathcal{A}},\mu\right)$
and $\left(\mathcal{B},\Sigma_{\mathcal{B}},\nu\right)$, then we
say $B\sim f(A)$ if $f:\mathcal{A}\rightarrow\mathcal{B}$ is such
that $\mathcal{\mathfrak{b}}\in\Sigma_{\mathcal{B}}$ implies $f^{-1}\left(\mathfrak{b}\right)\in\Sigma_{\mathcal{A}}$
and $\nu\left(\mathfrak{b}\right)=\mu\left(f^{-1}\left(\mathfrak{b}\right)\right)$,
$\nu$ being referred to as the \emph{induced measure} (with respect
to $\mu,f$), and the function $f$ being said to be $\left(\mathcal{A},\Sigma_{\mathcal{A}},\mu\right)-\left(\mathcal{B},\Sigma_{\mathcal{B}},\nu\right)$-\emph{measurable}.

With any indexed set of random entities $\left\{ A_{\omega}\right\} _{\omega\in\Omega}$
each of which is distributed as $\left(\mathcal{A}_{\omega},\Sigma_{\omega},\mu_{\omega}\right)$,
$\omega\in\Omega$, we associate its {}``natural'' observation space
$\left(\mathcal{A},\Sigma\right)$, with $\mathcal{\mathcal{A}=\prod_{\omega\in\Omega}\mathcal{A}_{\omega}}$
(Cartesian product) and $\Sigma=\bigotimes_{\omega\in\Omega}\Sigma_{\omega}$
being the smallest sigma algebra containing all sets of the form $\mathfrak{a}_{\omega}\times\prod_{\iota\in\Omega-\left\{ \omega\right\} }\mathcal{A}_{\iota}$,
$\mathfrak{a}_{\omega}\in\Sigma_{\omega}$. We say that the random
entities in $\left\{ A_{\omega}\right\} _{\omega\in\Omega}$ \emph{possess
a joint distribution} if $\left\{ A_{\omega}\right\} _{\omega\in\Omega}$
is a random entity distributed as $\left(\mathcal{A},\Sigma,\mu\right)$
with $\mu\left(\mathfrak{a}_{\omega}\times\prod_{\iota\in\Omega-\left\{ \omega\right\} }\mathcal{A}_{\iota}\right)=\mu_{\omega}\left(\mathfrak{a}_{\omega}\right).$
Every subset $\Omega'\subset\Omega$ possesses a\emph{ marginal distribution}
$\left(\prod_{\omega\in\Omega'}\mathcal{A}_{\omega},\bigotimes_{\omega\in\Omega'}\Sigma_{\omega},\mu'\right)$,
where $\mu'\left(\mathfrak{a}\right)=\mu\left(\mathfrak{a}\times\prod_{\iota\in\Omega-\Omega'}\mathcal{A}_{\iota}\right)$,
for all $\mathfrak{a}\in\bigotimes_{\omega\in\Omega'}\Sigma_{\omega}$.%
\footnote{The standard definition of a random entity (also called {}``random
element'' or simply {}``random variable'') is a measurable function
from a sample space to an observation space. The present terminology
can be reconciled with this view by considering $\left(\left\{ `A\textnormal{'}\right\} \times\mathcal{A},\left\{ \left\{ `A\textnormal{'}\right\} \times\mathfrak{a}:\mathfrak{a}\in\Sigma\right\} ,\nu\right)$
a sample space, $\left(\mathcal{A},\Sigma\right)$ an observation
space, and $A$ the projection function $\left\{ `A\textnormal{'}\right\} \times\mathcal{A}\rightarrow\mathcal{A}$.
In the case of jointly distributed random entities, $A=\left\{ A_{\omega}\right\} _{\omega\in\Omega}$,
each of them, with an observation space $\left(\mathcal{A}_{\omega},\Sigma_{\omega}\right)$,
can be defined as the projection function $\left\{ `A\textnormal{'}\right\} \times\mathcal{A}\rightarrow\mathcal{A}_{\omega}$.
We do not, however, assume a common sample space for all random entities
being considered. The notion of a sample space is a source of conceptual
confusions, the chief one being the notion that there is only one
sample space {}``in this universe,'' so that any set of random entities
possesses a joint distribution.%
} 
\begin{rem}
Note that the elements of the Cartesian product $\prod_{\omega\in\Omega}\mathcal{A}_{\omega}$
are \emph{choice functions} $\Omega\rightarrow\bigcup_{\omega\in\Omega}\mathcal{A}_{\omega}$,
that is, they are sets of pairs of the form $\left(\omega,a\right)$,
$\omega\in\Omega$, $a\in\mathcal{A}_{\omega}$. This means that the
indexation of $\left\{ A_{\omega}\right\} _{\omega\in\Omega}$ is
part of the identity of $\mathcal{\mathcal{A}=\prod_{\omega\in\Omega}\mathcal{A}_{\omega}}$,
hence also of the distribution of $A=\left\{ A_{\omega}\right\} _{\omega\in\Omega}$.
Ideally, only the {}``ordinal structure'' of the indexing set $\Omega$
should matter, and this can be ensured by agreeing that $\Omega$
is always an initial segment of the class of ordinal numbers. With
these conventions in mind, $\left\{ A_{\omega}\right\} _{\omega\in\Omega}$
can be viewed as generalizing the notion of a finite vector (although
it is convenient not to complicate notation to reflect this fact).
For sets of jointly distributed and identically indexed random entities,
the relation $\left\{ A_{\omega}\right\} _{\omega\in\Omega}\sim\left\{ B_{\omega}\right\} _{\omega\in\Omega}$
should always be understood in the sense of {}``corresponding indices,''
implying, in particular, $\left\{ A_{\omega}\right\} _{\omega\in\Omega'}\sim\left\{ B_{\omega}\right\} _{\omega\in\Omega'}$
for any subset $\Omega'$ of $\Omega$.
\end{rem}
The equality $A_{1}=A_{2}$ in the present context means that the
two random entities have a common observation space $\left(\mathcal{A},\Sigma\right)$,
and that $\left\{ A_{1},A_{2}\right\} $ is a jointly distributed
random entity with measure $\mu$ such that $\mu\left(\left\{ \left(a_{1},a_{2}\right)\in\mathcal{A}\times\mathcal{A}:a_{1}=a_{2}\right\} \right)=1$
(this corresponds to the equality {}``almost surely'' in the traditional
terminology). We also follow the common practice of using equality
to replace {}``is'' or {}``denotes'' in definitions and abbreviations,
such as $A=\left\{ A_{\omega}\right\} _{\omega\in\Omega}$. The two
meanings of equality are easily distinguished by context.

A \emph{random variable} is a special case of random entity. Its definition
can be given as follows: (i) if $\mathcal{A}$ is countable, $\Sigma$
is the power set of $\mathcal{A}$, then a random entity distributed
as $\left(\mathcal{A},\Sigma,\mu\right)$ is a random variable; (ii)
if $\mathcal{A}$ is an interval of reals, $\Sigma$ is the Lebesgue
sigma-algebra on $\mathcal{A}$, then a random entity distributed
as $\left(\mathcal{A},\Sigma,\mu\right)$ is a random variable; (iii)
any jointly distributed vector $\left(A_{1},\ldots,A_{n}\right)$
with all components random variables is a random variable. The notion
thus defined is more general than in the main text, but the theory
presented there applies with no modifications. 
\begin{lem}
\label{lem:basic}A set $\left\{ A_{\omega}\right\} _{\omega\in\Omega}$
of random entities possesses a joint distribution if and only if there
is a random entity $R$ distributed as a probability space $\left(\mathcal{R},\Sigma_{\mathcal{R}},\nu\right)$
and some functions $\left\{ f_{\omega}:\mathcal{R}\rightarrow\mathcal{A}_{\omega}\right\} _{\omega\in\Omega}$,
such that $\left\{ A_{\omega}\right\} _{\omega\in\Omega}=\left\{ f_{\omega}\left(R\right)\right\} _{\omega\in\Omega}$. \end{lem}
\begin{proof}
(Note that the formulation implies that all the functions involved
are appropriately measurable.) To show sufficiency, observe that the
induced measure $\mu$ of any set of the form $\prod_{\omega\in N}\mathfrak{a}_{\omega}\times\prod_{\iota\in\Omega-N}\mathcal{A}_{\iota}$,
where $N$ is a finite subset of $\Omega$ and $\mathfrak{a}$$_{\omega}\in\Sigma_{\omega}$
for $\omega\in N$, is $\nu\left(\bigcap_{\omega\in N}f_{\omega}^{-1}\left(\mathfrak{a}_{\omega}\right)\right)$,
and this measure is uniquely extended to $\bigotimes_{\omega\in\Omega}\Sigma_{\omega}$.
To show necessity, put $R=\left\{ A_{\omega}:\omega\in\Omega\right\} $
and, for every $\omega\in\Omega$, define $f_{\omega}:\mathcal{R}\rightarrow\mathcal{A}_{\omega}$
to be the (obviously measurable) projection $f_{\omega}:\prod_{\iota\in\Omega}\mathcal{A}_{\iota}\rightarrow\mathcal{A}_{\omega}$. \end{proof}
\begin{cor}
\label{cor:rvs}If $\Omega$ is finite and $\left\{ A_{\omega}\right\} _{\omega\in\Omega}$
is a set of random variables, then $R$ in Lemma \ref{lem:basic}
can be chosen to be a random variable. Moreover, $R$ can be chosen
arbitrarily, as any continuously (atomlessly) distributed random variable
(e.g., uniformly distributed between 0 and 1).\end{cor}
\begin{proof}
The first statement follows from the fact that $R=\left\{ A_{\omega}\right\} _{\omega\in\Omega}$
in the necessity part of Lemma \ref{lem:basic} is then a random variable.
The second statement follows from Theorem 1 in Dzhafarov \& Gluhovsky,
2006, based on a general result for standard Borel spaces (e.g., in
Kechris, 1995, p. 116).
\end{proof}

\subsubsection*{Selective influences and JDC}

A \emph{factor} is defined as a nonempty set of \emph{factor points}
with a unique name: the notation used is $x^{\alpha}=\left\{ x,`\alpha\textnormal{'}\right\} $.
Let $\Phi$ be a nonempty set of factors, and let $T\subset\prod\Phi$
be a nonempty set of \emph{treatments}. Note that any treatment $\phi\in T$
is a function $\phi:\Phi\rightarrow\bigcup\Phi$, so $\phi\left(\alpha\right)$
denotes the factor point $x^{\alpha}$ of the factor $\alpha$ which
belongs to the treatment $\phi$. (The notation for $\phi\left(\alpha\right)$
used in the main text is $\phi_{\left\{ \alpha\right\} }$.)

Let $\Omega$ be an indexing set for a set of random entities $\left\{ R_{\omega}\right\} _{\omega\in\Omega}$.
A diagram of selective influences is a mapping $M:\Omega\rightarrow2^{\Phi}$.
For any such a diagram one can redefine the set of factors and the
set of treatments in the following way. For every $\omega\in\Omega$,
put 
\[
\omega^{*}=\left\{ s^{\omega^{*}}:s\in\prod M\left(\omega\right)\right\} ,
\]
if $M\left(\omega\right)$ is nonempty; if it is empty, put $\omega^{*}=\left\{ \textrm{Ø}^{\omega^{*}}\right\} $.
This establishes the bijective mapping $M^{*}:\Omega\rightarrow2^{\Phi^{*}}$,
where $\Phi^{*}=\left\{ \omega^{*}\right\} _{\omega\in\Omega}$. For
each treatment $\phi\in T$ we define the corresponding treatment
$\phi^{*}$ as $\left\{ s^{\omega^{*}}:s\subset\phi\wedge s\in\prod M{}^{*}\left(\omega\right),\omega\in\Omega\right\} $.
The set of all such treatments $\phi^{*}$ is denoted $T^{*}$. (In
the main text the procedure just described is called \emph{canonical
rearrangement}.) In the following we omit asterisks and simply put
$\Phi=\Omega$, replacing $M:\Omega\rightarrow2^{\Phi}$ with the
identity map $M:\Omega\rightarrow\Phi$.

Among several equivalent definitions of selective influences we choose
here the one most immediately prompting the Joint Distribution Criterion
(JDC). 
\begin{defn}
\label{def:SI general}Let $\mathbb{A}=\left\{ A_{\phi}\right\} _{\phi\in T}$,
and $A_{\phi}=\left\{ A_{\phi,\alpha}\right\} _{\alpha\in\Phi}$ for
every $\phi\in T$. Let $T$ be a set of treatments associated with
a set of factors $\Phi$. Let $A_{\phi,\alpha}$ for each $\alpha,\phi$
be distributed as $\left(\mathcal{A}_{\phi\left(\alpha\right)},\Sigma_{\phi\left(\alpha\right)},\mu_{\phi,\alpha}\right)$.
We say that each $A_{\phi,\alpha}$ is \emph{selectively influenced}
\emph{by} $\alpha$ ($\alpha\in\Phi,\phi\in T$), and write schematically
$\mathbb{A}\looparrowleft\Phi$, if there is a random entity $R$
distributed as $\left(\mathcal{R},\Sigma_{\mathcal{R}},\nu\right)$
and some functions $\left\{ f_{x^{\alpha}}:\mathcal{R}\rightarrow\mathcal{A}_{x^{\alpha}}\right\} _{x^{\alpha}\in\bigcup\Phi}$
such that $A_{\phi}=\left\{ A_{\phi,\alpha}\right\} _{\alpha\in\Phi}\sim\left\{ f_{\phi\left(\alpha\right)}\left(R\right)\right\} _{\alpha\in\Phi}$,
for all $\phi\in T$. \end{defn}
\begin{rem}
Note that the formulation implies that all the functions involved
are appropriately measurable. Also, in $\left\{ f_{x^{\alpha}}:\mathcal{R}\rightarrow\mathcal{A}_{x^{\alpha}}\right\} _{x^{\alpha}\in\bigcup\Phi}$
the set $\bigcup\Phi$ can be replaced with $\bigcup_{\phi\in T,\alpha\in\Phi}\phi\left(\alpha\right)$
if the latter is a proper subset of $\bigcup\Phi$ (and the same applies
to the definition of $H$ in the theorem below). We assume, however,
that factor points never used in treatments can simply be deleted
from the factors. 
\end{rem}

\begin{rem}
\label{rem:Invariance}In the main text we assume that $\left(\mathcal{A}_{\phi\left(\alpha\right)},\Sigma_{\phi\left(\alpha\right)}\right)=\left(\mathcal{A}_{\alpha},\Sigma_{\alpha}\right)$,
that is, the observation space $\left(\mathcal{A}_{\alpha},\Sigma_{\alpha}\right)$
of the entity $A_{\phi,}{}_{\alpha}$ is the same across different
treatments $\phi\in T$. In footnote \ref{fn:Invariance} we mention
that this constraint is not essential, as the random entities $A_{\phi,\alpha}$
can always be redefined to force $\left(\mathcal{A}_{\phi\left(\alpha\right)},\Sigma_{\phi\left(\alpha\right)}\right)=\left(\mathcal{A}_{\alpha},\Sigma_{\alpha}\right)$
without affecting selective influence. This redefinition can be done
in a variety of ways, the simplest one being to put 
\[
\mathcal{A}_{\alpha}=\bigcup_{\phi\in T}\left\{ \phi(\alpha)\right\} \times\mathcal{A}_{\phi(\alpha)},
\]
and let $\Sigma_{\alpha}$ be the smallest sigma-algebra containing
$\left\{ \left\{ \phi(\alpha)\right\} \times\mathfrak{a}:\mathfrak{a}\in\Sigma_{\phi\left(\alpha\right)},\ \phi\in T\right\} $.
Define $g_{\phi\left(\alpha\right)}:\mathcal{A}_{\phi\left(\alpha\right)}\rightarrow\mathcal{A}_{\alpha}$
by $g_{\phi\left(\alpha\right)}\left(a\right)=\left(\phi(\alpha),a\right)$,
for $a\in\mathcal{A}_{\phi\left(\alpha\right)},\phi\in T,\alpha\in\Phi$.
Then $A_{\phi,}^{*}{}_{\alpha}=g_{\phi\left(\alpha\right)}\left(A_{\phi,\alpha}\right)$
and $A_{\phi}^{*}=\left\{ A_{\phi,\alpha}^{*}\right\} _{\alpha\in\Phi}$
are the redefined random entities sought. Note that if $\mathbb{A}\looparrowleft\Phi$,
then $\mathbb{A}^{*}=\left\{ A_{\phi}^{*}\right\} _{\phi\in T}\looparrowleft\Phi$,
because Definition \ref{def:SI general} applies to $\mathbb{A}^{*}$
with the same $R$ and with the composite functions $g_{x^{\alpha}}\circ f_{x^{\alpha}}$
replacing $f_{x^{\alpha}}$, for all $x^{\alpha}\in\bigcup\Phi$.
(In the terminology of the main text, $g_{x^{\alpha}}$ are factor-point-specific
transformations.)\end{rem}
\begin{thm}[JDC]
A necessary and sufficient condition for $\mathbb{A}\looparrowleft\Phi$
in Definition \ref{def:SI general} is the existence of a set of jointly
distributed random entities
\[
H=\left\{ H_{x^{\alpha}}\right\} _{x^{\alpha}\in\bigcup\Phi}
\]
(one random entity for each factor point of each factor), such that
\[
\left\{ H_{x^{\alpha}}\right\} _{x^{\alpha}\in\phi}\sim A_{\phi}
\]
for every treatment $\phi\in T$\textup{\emph{.}}\end{thm}
\begin{proof}
Immediately follows from the definition and Lemma \ref{lem:basic}.\end{proof}
\begin{thm}
If $\:\bigcup\Phi$ in Definition \ref{def:SI general} is a finite
set and $A_{\phi\left(\alpha\right)}$ is a random variable for every
\textup{$\alpha,\phi$, then }$R$ can always be chosen to be a random
variable. Moreover, $R$ can be chosen arbitrarily, as any continuously
(atomlessly) distributed random variable. \end{thm}
\begin{proof}
Immediately follows from JDC and Corollary \ref{cor:rvs}.\end{proof}
\begin{rem}
In Dzhafarov and Gluhovsky (2006) this inference was not made because
JDC at that time was not explicitly formulated (outside quantum mechanics,
see footnotes \ref{fn:Surprisingly,-at-least} and \ref{fn:In-reference-to}).
\end{rem}
The three basic properties of selective influences listed in Section
\ref{sub:Three-basic-properties} trivially generalize to arbitrary
sets of factors and random entities.

\subsubsection*{Distance-type tests}

The principles of test construction (Section \ref{sub:General-principles})
and the logic of the distance-type tests in particular, apply without
changes to arbitrary sets of factors. As to the random entities, some
of the test measures are confined to discrete and/or real-valued variables
(e.g., information-based and Minkowski-type ones), others (such as
classification measures) are completely general. 

We will use the notation and terminology adopted in Dzhafarov and
Kujala (2010). Chains of factor points can be denoted by capital Roman
letters, $X=x_{1}^{\alpha_{1}}\ldots x_{l}^{\alpha_{l}}$ . A subsequence
of points belonging to a chain forms its \emph{subchain}. A \emph{concatenation}
of two chains $X$ and $Y$ is written as $XY$. So, we can have chains
$x^{\alpha}Xy^{\beta}$, $x^{\alpha}XYy^{\beta}$, etc. The number
of points in a chain $X$ is its \emph{cardinality}, $\left|X\right|$.
For any treatment-realizable chain $X=x_{1}^{\alpha_{1}}\ldots x_{l}^{\alpha_{l}}$,
we write 
\[
DX=\sum_{i=1}^{l-1}Dx^{\alpha_{i}}x^{\alpha_{i+1}}
\]
(with the understanding that the sum is zero if $l$ is 0 or 1).

A treatment-realizable chain $u^{\mu}Xv^{\nu}$ is called \emph{compliant}
(with the chain inequality) if $Du^{\mu}v^{\nu}\leq Du^{\mu}Xv^{\nu}=Dx^{\mu}x_{1}^{\alpha_{1}}+DX+Dx^{\alpha_{n}}x^{\nu}$;
it is called \emph{contravening} (the chain inequality) if $Du^{\mu}v^{\nu}>Du^{\mu}Xv^{\nu}$.
The proofs of the two lemmas below are very similar, but it is convenient
to keep them separate.
\begin{lem}
\label{lem:dist1}If a treatment-realizable chain $X_{0}=x_{1}^{\alpha_{1}}\ldots x_{l}^{\alpha_{l}}$
($l\geq$3) is contravening, then it contains a contravening subchain
in which no factor point occurs more than once.\end{lem}
\begin{proof}
If $l=3$ then the chain contains no factor point more than once,
because otherwise it is not treatment-realizable. If $l>3$, and $X_{0}$
contains factor points $x_{i}^{\alpha_{i}}=x_{j}^{\alpha_{j}}$, then
it can be presented as $X_{0}=x_{1}^{\alpha_{1}}\ldots x_{i}^{\alpha_{i}}Ux_{j}^{\alpha_{j}}\ldots x_{l}^{\alpha_{l}}$,
where $U$ is some nonempty subchain ($i$ may coincide with $1$
or $j$ coincide with $l$, but not both). But then $X_{1}=x_{1}^{\alpha_{1}}\ldots x_{i}^{\alpha_{i}}\ldots x_{l}^{\alpha_{l}}$
is also treatment-realizable and contravening, because
\[
\begin{array}{r}
Dx_{1}^{\alpha_{1}}x_{l}^{\alpha_{l}}>DX_{0}=Dx_{1}^{\alpha_{1}}\ldots x_{i}^{\alpha_{i}}Ux_{j}^{\alpha_{j}}\ldots x_{l}^{\alpha_{l}}\\
>Dx_{1}^{\alpha_{1}}\ldots x_{i}^{\alpha_{i}}\ldots x_{l}^{\alpha_{l}}=DX_{1}.
\end{array}
\]
If $X_{1}$ contains two equal factor points, then $3\leq\left|X_{1}\right|<\left|X_{0}\right|$,
and we can repeat the same procedure to obtain $X_{2}$, etc. As the
procedure has to stop at some $X_{t},$ this subchain will contain
no factor point twice.\end{proof}
\begin{lem}
\label{lem:dist2}If a treatment-realizable chain $X_{0}=x_{1}^{\alpha_{1}}\ldots x_{l}^{\alpha_{l}}$
($l\geq$3) is contravening, then it contains a contravening irreducible
subchain.\end{lem}
\begin{proof}
By the previous lemma, we can assume that every factor point in $X_{0}$
occurs no more than once. If $l=3$, the chain $X_{0}$ itself is
irreducible, because otherwise there would exist a treatment $\phi\in T$
that includes the elements of the chain, and this would make the chain
compliant. If $l>3$, and the chain $X_{0}$ is not irreducible, then
it must contain a subchain $x_{i}^{\alpha_{i}}x_{j}^{\alpha_{j}}$
such that $j>i+1$ and $\left\{ x_{i}^{\alpha_{i}},x_{j}^{\alpha_{j}}\right\} $
is part of some treatment $\phi\in T$. The chain then can be presented
as $X_{0}=x_{1}^{\alpha_{1}}\ldots x_{i}^{\alpha_{i}}Ux_{j}^{\alpha_{j}}\ldots x_{l}^{\alpha_{l}}$,
where $U$ is some nonempty subchain ($i$ may coincide with $1$
or $j$ with $l$, but not both). The subchain $x_{i}^{\alpha_{i}}Ux_{j}^{\alpha_{j}}$
is clearly treatment-realizable. If it is contravening, then we replace
$X_{0}$ with $X_{1}=x_{i}^{\alpha_{i}}Ux_{j}^{\alpha_{j}}$; if it
is compliant, then we replace $X_{0}$ with $X_{1}=x_{1}^{\alpha_{1}}\ldots x_{i}^{\alpha_{i}}x_{j}^{\alpha_{j}}\ldots x_{l}^{\alpha_{l}}$.
In both cases we obtain a treatment-realizable subchain $X_{1}$ of
$X_{0}$ such that $3\leq\left|X_{1}\right|<\left|X_{0}\right|$,
and $X_{1}$ is contravening: in the former case $X_{1}=x_{i}^{\alpha_{i}}Ux_{j}^{\alpha_{j}}$
is contravening by construction, in the latter case $Dx_{i}^{\alpha_{i}}Ux_{j}^{\alpha_{j}}>Dx_{i}^{\alpha_{i}}x_{j}^{\alpha_{j}}$
whence 
\[
\begin{array}{r}
Dx_{1}^{\alpha_{1}}x_{l}^{\alpha_{l}}>DX_{0}=Dx_{1}^{\alpha_{1}}\ldots x_{i}^{\alpha_{i}}Ux_{j}^{\alpha_{j}}\ldots x_{l}^{\alpha_{l}}\\
>Dx_{1}^{\alpha_{1}}\ldots x_{i}^{\alpha_{i}}x_{j}^{\alpha_{j}}\ldots x_{l}^{\alpha_{l}}=DX_{1}.
\end{array}
\]
If $X_{1}$ is not irreducible, we can apply the same procedure to
$X_{1}$ to obtain a contravening subchain $X_{2}$ with $3\leq\left|X_{2}\right|<\left|X_{1}\right|$,
and continue in this manner. Eventually we have to reach a contravening
subchain $X_{t}$ of $X_{0}$ such that $\left|X_{t}\right|\geq3$
and the procedure cannot continue, indicating that $X_{t}$ is irreducible.
\end{proof}

\section*{REFERENCES}

\setlength{\parindent}{0cm}\everypar={\hangindent=15pt}Basoalto,
R.M., \& Percival, I.C. (2003). BellTest and CHSH experiments with
more than two settings. \emph{Journal of Physics A: Mathematical \&
General}, 36, 7411\textendash{}7423.

Bloxom, B. (1972). The simplex in pair comparisons. \emph{Psychometrika},
37, 119\textendash{}136.

Cover, T. M. and Thomas, J. A. (1991). \emph{Elements of Information
Theory}. New York: Wiley.

Deza, M.-M., \& Rosenberg, I. G. (2000). n-Semimetrics. \emph{European
Journal of Combinatorics}, 21, 797\textendash{}806.

Dzhafarov, E.N. (1992). The structure of simple reaction time to step-function
signals. \emph{Journal of Mathematical Psychology}, 36, 235-268.

Dzhafarov, E. N. (1997). Process representations and decompositions
of response times. In A. A. J. Marley (Ed.), \emph{Choice, decision
and measurement: Essays in honor of R. Duncan Luce}, pp. 255\textendash{}278.
New York: Erlbaum.

Dzhafarov, E.N. (1999). Conditionally selective dependence of random
variables on external factors. \emph{Journal of Mathematical Psychology},
43, 123\textendash{}157.

Dzhafarov, E.N. (2001). Unconditionally selective dependence of random
variables on external factors. \emph{Journal of Mathematical Psychology},
45, 421\textendash{}451.

Dzhafarov, E.N. (2002). Multidimensional Fechnerian scaling: Pairwise
comparisons, regular minimality, and nonconstant self-similarity.
\emph{Journal of Mathematical Psychology,} 46, 583-608. 

Dzhafarov, E.N. (2003a). Selective influence through conditional independence.
\emph{Psychometrika}, 68, 7\textendash{}26.

Dzhafarov, E.N. (2003b). Thurstonian-type representations for {}``same-different'''
discriminations: Deterministic decisions and independent images. \emph{Journal
of Mathematical Psychology}, 47, 208-228. 

Dzhafarov, E.N. (2003c). Thurstonian-type representations for {}``same-different'''
discriminations: Probabilistic decisions and interdependent images.
\emph{Journal of Mathematical Psychology}, 47, 229-243.

Dzhafarov, E.N. (2006) On the Law of Regular Minimality: Reply to
Ennis. \emph{Journal of Mathematical Psychology}, 50, 74-93.

Dzhafarov, E.N. (2010). Dissimilarity, quasidistance, distance. \emph{Journal
of Mathematical Psychology}, 54, 334-337.

Dzhafarov, E.N., \& Gluhovsky, I. (2006). Notes on selective influence,
probabilistic causality, and probabilistic dimensionality. \emph{Journal
of Mathematical Psychology}, 50, 390\textendash{}401.

Dzhafarov, E.N., \& Kujala, J.V. (2010). The Joint Distribution Criterion
and the Distance Tests for selective probabilistic causality. \emph{Frontiers
in Quantitative Psychology and Measurement}, 1:151 doi: 10.3389/fpsyg.2010.00151.

Dzhafarov, E.N., Schweickert, R., \& Sung, K. (2004). Mental architectures
with selectively influenced but stochastically interdependent components.
\emph{Journal of Mathematical Psychology}, 48, 51-64.

Fine, A. (1982a). Joint distributions, quantum correlations, and commuting
observables. \emph{Journal of Mathematical Physics}, 23, 1306-1310.

Fine, A. (1982b). Hidden variables, joint probability, and the Bell
inequalities.\emph{ Physical Review Letters}, 48, 291-295.

Kechris, A. S. (1995). \emph{Classical Descriptive Set Theory}. New
York: Springer.

Kraskov, A., Stögbauer, H., Andrzejak, R.G., \& Grassberger, P. (2003).
Hierarchical Clustering Based on Mutual Information. \emph{ArXiv q-bio/0311039}.

Kujala, J. V., \& Dzhafarov, E. N. (2008). Testing for selectivity
in the dependence of random variables on external factors. \emph{Journal
of Mathematical Psychology}, 52, 128\textendash{}144.

Kujala, J.V., \& Dzhafarov, E.N. (2009). Regular Minimality and Thurstonian-type
modeling. Journal of Mathematical Psychology, 53, 486-501.

Levine, M. V. (2003). Dimension in latent variable models. \emph{Journal
of Mathematical Psychology}, 47, 450\textendash{}466.

McDonald, R. P. (1967). Nonlinear factor analysis. \emph{Psychometrika
Monographs}, 15. 

McDonald, R. P. (1982). Linear versus nonlinear models in item response
theory. \emph{Applied Psychological Measurement}, 6, 379\textendash{}396.

Schweickert, R. (1982). The bias of an estimate of coupled slack in
stochastic PERT networks. \emph{Journal of Mathematical Psychology},
26, 1\textendash{}12.

Schweickert, R., Fisher, D.L., \& Goldstein, W.M. (2010). Additive
factors and stages of mental processes in task networks. \emph{Journal
of Mathematical Psychology,} 54, 405\textendash{}414

Schweickert, R., \& Townsend, J. T. (1989). A trichotomy: Interactions
of factors prolonging sequential and concurrent mental processes in
stochastic discrete mental (PERT) networks. \emph{Journal of Mathematical
Psychology}, 33, 328-347.

Sternberg, S. (1969). The discovery of processing stages: Extensions
of Donders\textquoteright{} method. In W.G. Koster (Ed.), \emph{Attention
and Performance II. Acta Psychologica}, 30, 276\textendash{}315.

Suppes, P., \& Zanotti, M. (1981). When are probabilistic explanations
possible? \emph{Synthese}, 48, 191\textendash{}199.

Taylor, M.D. (1984). Separation metrics for real-valued random variables\emph{.
International Journal of Mathematics and Mathematical Sciences}, 7,
407-408.

Taylor, M.D. (1985). New metrics for weak convergence of distribution
fiunctions. \emph{Stochastica}, 9, 5-17.

Thurstone, L. L. (1927b). A law of comparative judgments. \emph{Psychological
Review}, 34, 273\textendash{}286.

Townsend, J.T. \& Ashby, G. (1983). \emph{The Stochastic Modeling
of Elementary Psychological Processes}. Cambridge: Cambridge University
Press.

Townsend, J. T. (1984). Uncovering mental processes with factorial
experiments. \emph{Journal of Mathematical Psychology}, 28, 363\textendash{}400. 

Townsend, J.T., \& Schweickert, R. (1989). Toward the trichotomy method
of reaction times: Laying the foundation of stochastic mental networks.
\emph{Journal of Mathematical Psychology}, 33, 309\textendash{}327.

Townsend, J. T., \& Thomas, R. D. (1994). Stochastic dependencies
in parallel and serial models: Effects on systems factorial interactions.
\emph{Journal of Mathematical Psychology}, 38, 1\textendash{}34.

Webster, R. (1994). \emph{Convexity}. New York: Oxford University
Press.

Werner, R.F., \& Wolf, M.M. (2001). All multipartite Bell correlation
inequalities for two dichotomic observables per site. arXiv:quant-ph\slash{}0102024v1.

Werner, R.F., \& Wolf, M.M. (2001). Bell inequalities and entanglement.
arXiv:quant-ph\slash{}0107093v2.
\end{document}